\pdfoutput=1
\RequirePackage{fix-cm}
\documentclass[smallcondensed]{svjour3}     
\smartqed  
\usepackage{graphicx}
\usepackage{amsmath,amssymb}
\usepackage{soul,color}
\usepackage{enumerate}
\usepackage{algorithm}
\usepackage{algorithmic}
\usepackage{graphicx}
\usepackage[compatibility=false]{caption}
\usepackage{subcaption}
\usepackage{xcolor,colortbl}
\usepackage{xspace}
\usepackage{tabularx}
\usepackage{multirow}
\usepackage{booktabs}
\usepackage{float}
\usepackage{mathtools}
\usepackage{natbib}
\usepackage[export]{adjustbox}
\usepackage{romannum}
\usepackage{capt-of}

\usepackage[T1]{fontenc}

\newcommand{\vbf}[1]{\ensuremath{\boldsymbol{\mathrm{#1}}}}
\newcommand{\norm}[1]{\ensuremath{\lVert{#1}\rVert}_{F}^{2}}

\newcommand{\eg}[0]{\emph{e.g.},\xspace}
\newcommand{\ie}[0]{\emph{i.e.},\xspace}

\newcommand{\TopicResponse}{\textsc{TopicResponse}\xspace}
\newcommand{\GGNMF}{\textsc{GG-NMF}\xspace}
\newcommand{\reals}{\ensuremath{\mathbb{R}}\xspace}
\newcommand{\Var}{\ensuremath{\mathrm{Var}}\xspace}
\newcommand{\Exp}{\ensuremath{\mathbb{E}}\xspace}

\usepackage{stackengine}
\stackMath
\newlength\matfield
\newlength\tmplength
\def\matscale{1.}
\newcommand\dimbox[3]{%
    \setlength\matfield{\matscale\baselineskip}%
    \setbox0=\hbox{\vphantom{X}\smash{#3}}%
    \setlength{\tmplength}{#1\matfield-\ht0-\dp0}%
    \fboxrule=1pt\fboxsep=-\fboxrule\relax%
    \fbox{\makebox[#2\matfield]{\addstackgap[.5\tmplength]{\box0}}}%
}

\newcommand\matbox[5]{
    \stackunder{\dimbox{#1}{#2}{$\mathbf{#5}$}}{\scriptstyle(#3\times #4)}
}    
    
\DeclareMathOperator*{\argmin}{argmin}
\DeclarePairedDelimiter{\floor}{\lfloor}{\rfloor}

\newtheorem{thm}{Theorem} 
\newtheorem{defn}[thm]{Definition} 

\newcolumntype{P}[1]{>{\centering\arraybackslash}p{#1}}

\begin{document}

\title{TopicResponse: A Marriage of Topic Modelling and Rasch Modelling for Automatic Measurement in MOOCs
}

\titlerunning{TopicResponse: Automatic Measurement in MOOCs}        

\author{Jiazhen He$^{1,3}$        \and
	Benjamin~I.~P.~Rubinstein$^{1}$ \and
	James Bailey$^{1,3}$ \and
        Rui Zhang$^{1}$ \and \\
        Sandra Milligan$^{2}$
}

\authorrunning{Jiazhen He et al.} 

\institute{Jiazhen He \at
              \email{jiazhenh@student.unimelb.edu.au}           
           \and
            Benjamin I. P. Rubinstein, James Bailey, Rui Zhang and Sandra Milligan\at
            \email{\{brubinstein, baileyj, rui.zhang, s.milligan\}@unimelb.edu.au}
                  \\  
                  \\        
		  ${}^1$School of Computing \& Information Systems, The University of Melbourne, Australia              \\
                            ${}^2$Melbourne Graduate School of Education, The University of Melbourne, Australia \\
                            ${}^3$Data61/CSIRO, Australia
}

\date{Received: date / Accepted: date}

\maketitle

\begin{abstract}
This paper explores the suitability of using automatically discovered topics
from MOOC discussion forums for modelling students' academic
abilities. The Rasch model from psychometrics is a popular generative probabilistic
model that relates latent student skill, latent item difficulty, and observed student-item
responses within a principled, unified framework.
According to scholarly educational theory, discovered topics can be 
regarded as appropriate measurement items if (1) students' participation
across the discovered topics is well fit by the Rasch model, and if (2) 
the topics are interpretable to subject-matter experts as being educationally
meaningful. Such Rasch-scaled topics, with associated difficulty levels, could be of
potential benefit to curriculum refinement, student assessment and personalised feedback.
The technical challenge that remains, is to discover meaningful topics that simultaneously achieve good
statistical fit with the
Rasch model. To address this challenge, we combine the Rasch model with non-negative
matrix factorisation based topic modelling, jointly fitting both models. We demonstrate
the suitability of our approach with quantitative experiments on data from three Coursera
MOOCs, and with qualitative survey results on topic interpretability on a Discrete Optimisation MOOC.

\keywords{MOOCs \and Topic Modelling \and Matrix Factorisation \and Psychometrics \and Item Response Theory \and Rasch Model}
\end{abstract}

\section{Introduction}
\label{intro}
Massive Open Online Courses (MOOCs) have attracted wide attention due to the promise of delivering education at scale. 
This new learning environment produces a variety of data (\eg demographic data, student engagement, and forum activities), which offer new opportunities to understanding student learning. 
While quizzes and assignments have dominated summative assessment, the many sources of rich student engagement data generated in MOOC platforms present new views on student learning and avenues for formative feedback. 
This paper explores whether students' participation across automatically discovered MOOC forum topics is suitable for modelling academic ability. 

%

Our work is inspired by the importance of forum discussions as an active learning activity, and recent research on quantitative measurement of student learning in the education community. 
In particular, \romannum{1}) 
MOOC discussion forums, as the main platform for student-instructor and student-student interactions, 
is of importance in gaining insights into student learning.
\romannum{2}) Recent research in education~\citep{milligan2015crowd} suggests that a distinctive and complex learning skill is required to promote learning in MOOCs. Educators are interested in whether and how the possession of this complex learning skill may be evidenced by latent complex patterns of engagement, instead of traditional assessment tools such as quizzes and assignments.
\romannum{3}) 
In order to validate such a hypothesis, measurement theory can be used \citep{rasch1993probabilistic,wright1982rating}. A set of items is handcrafted from forum activities (\eg ``\textit{contributed a post attracting votes from others}'' and ``\textit{made repeated thread visits in more than half the weeks}''), and calibrated (\eg deleted or changed) to fit a measurement model as evidence as to whether the set of items is appropriate for measuring the complex learning skill~\citep{milligan2015crowd}.
This process is human-intensive and time-consuming as reflected by Figure~\ref{fig:flow}.

\begin{figure}[!htb]
    \centering
    \includegraphics[scale=0.65]{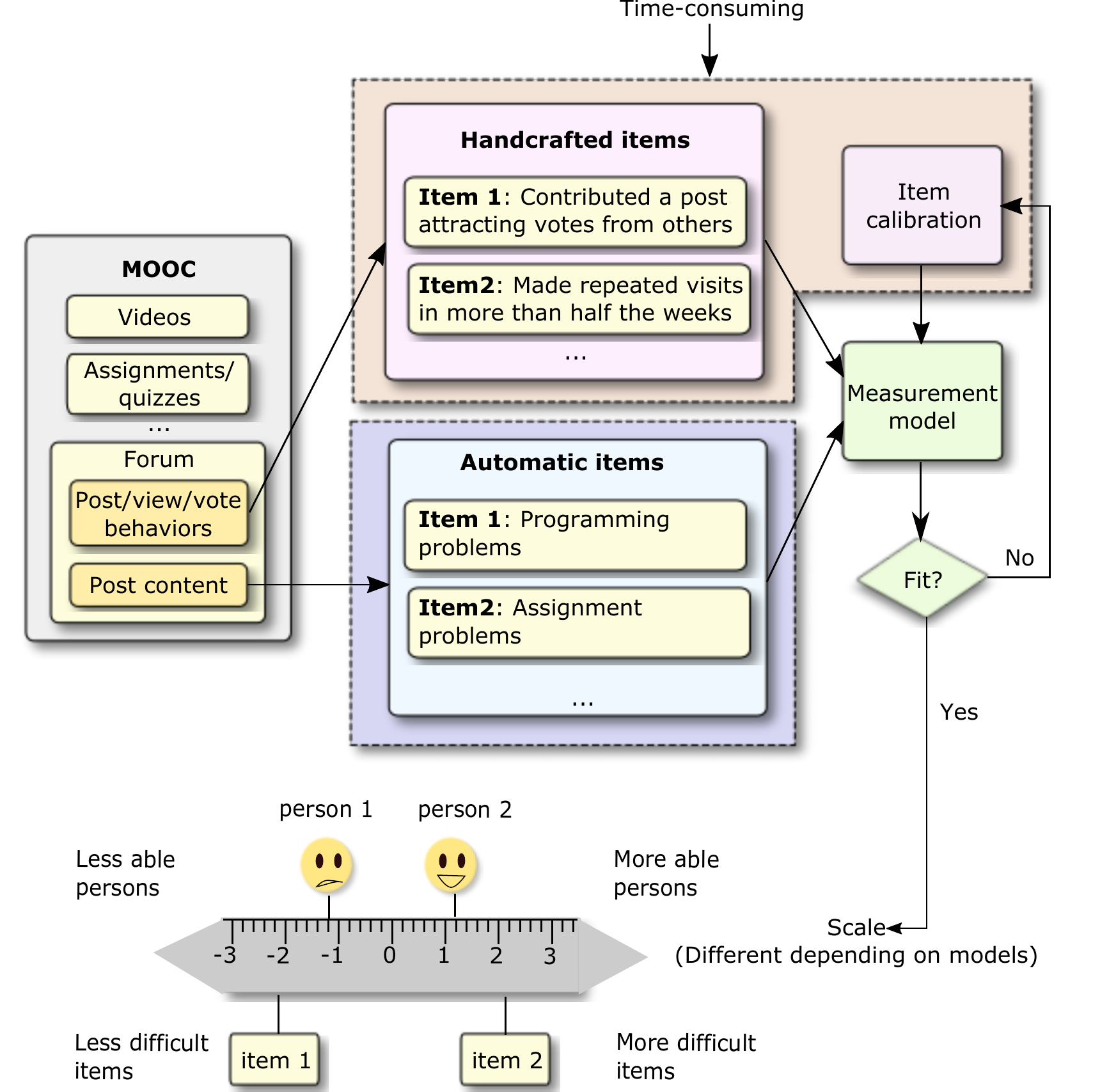}
    \caption{Workflow for devising items manually versus automatically discovering topics as items for measurement. Traditionally, a set of items are handcrafted from MOOC forum behaviours, and then the students' dichotomous responses on the items are examined using the Rasch model. If the model fits well, then the students and items can be compared on an inferred scale (the ruler). Otherwise the items are refined (changed, added or deleted) manually until model fit. The process of handcrafting and calibration is time-consuming. Instead, we aim to automatically generate topics from discussion posts as items that fit the Rasch model by design.}
    \label{fig:flow}
\end{figure}

Driven by these observations, we investigate whether students' participation in automatically discovered forum topics can be used as an instrument to model students' ability.
If students' participation across the discovered topics fit a measurement model (in this paper, we use the Rasch model) in terms of statistical effectiveness, and the topics are interpretable to subject-matter experts by way of qualitative effectiveness, then the discovered topics can be regarded as useful items for measurement. 
The resulting scaled topics, endowed with estimated difficulty levels, can assist in subsequent curriculum refinement, student assessment, and personalised feedback. 

The technical challenge, then, is to automatically discover topics such that students' participation across them fit the Rasch model.
 \citet{he2016moocs} have adapted topic modelling of students' online forum postings, such that students' participation across these topics conforms to the Guttman scale. 
However, the Guttman scale is widely regarded 
as overly-idealised and impractical in the real world. 
In contrast the Rasch model, one of the simplest item response theory (IRT) models and the basis for many extensions, has been widely used in education and psychology.
It is a generative probabilistic model that represents student responses as noisy observations of latent student abilities related to item difficulties. It can be viewed as a stochastic counterpart to the Guttman scale, permitting measurement error.
If a person's ability level is higher than an item's difficulty, the person will answer the item correctly in the Guttman scale, while in the Rasch model there is a certain probability of incorrect response.
While the Guttman scale only permits ordering of persons and items, Rasch models the locations on the scale and hence also meaningful differences~\citep{scholten2011admissible}. 
The algorithm proposed for the Guttman scale~\citep{he2016moocs} does not adapt readily for Rasch modelling. 
Instead we propose the \TopicResponse algorithm, which simultaneously performs non-negative matrix factorisation and Rasch model fitting. 
The main contributions of this paper include:




\begin{itemize}
    \item The first study that combines topic modelling with Rasch modelling in psychometric testing: generating topics that measure students' academic abilities based on online forum postings;
    \item An algorithm \TopicResponse fitting NMF and Rasch models simultaneously, for which we provide a proof of convergence; and 
    \item Quantitative experiments on three Coursera MOOCs covering a broad swath of disciplines, establishing statistical effectiveness of our algorithm, and qualitative results on a Discrete Optimisation MOOC, supporting interpretability.
\end{itemize}



We review related work in Section~\ref{sec:relwork}. In Section~\ref{sec:preps}, we present preliminaries and formalise our problem. Our algorithm is introduced in Section~\ref{sec:nmfraschg}, and evaluated in Section~\ref{sec:exp}. Section~\ref{sec:con} concludes the paper.

\section{Related Work}
\label{sec:relwork}
Many studies have focused on item response theory (IRT) or MOOC data analysis, but research on automatic discovery of items for measurement in MOOCs has received little attention. The main relevant work to this paper is \citep{he2016moocs}, where NMF-based topic modelling is adapted and used for Guttman scaling~\citep{guttman1950basis} in order to measure students' latent abilities based on their MOOC forum posts. A major drawback of that work is that the Guttman scale is regarded to be the most restrictive IRT model and is overly idealised: it neither serves as the basis of more sophisticated (probabilistic) models, nor is it practical in the real world as a deterministic model.
While the Guttman scale only models ordering of persons and items, the (probabilistic) Rasch model permits the interpretation of the differences between items and people~\citep{scholten2011admissible}. 
The Rasch model is a generative model that models student responses as noisy observations of latent student abilities in relation to item difficulties.
The algorithm for Guttman scaling~\citep{he2016moocs} does not naturally extend to incorporating Rasch modelling.

\subsection{Item Response Theory (IRT)}
The field of IRT studies statistical models for measurement in education and psychology. Such models specify the probability of a person's response on an item as a mathematical function of the person's and item's latent attributes. 
A principal goal of IRT is to create a scale on which persons and items can be placed and compared meaningfully. 
IRT has been used for computerised adaptive testing (CAT), which aims to accurately and efficiently assess individuals' trait levels, and is used in the Scholastic Aptitude Test (SAT), Graduate Record Examination (GRE), while \cite{chen2005personalized} proposed a personalised e-learning system based on IRT considering course material difficulty and learner ability.

As a statistical model, IRT has attracted attention in machine learning recently.
\citet{bergner2012model} applied model-based collaborative filtering to estimate the parameters for IRT models, considering IRT as a type of collaborative filtering task, where the user-item interactions are factorised into user and item parameters.
\citet{bachrach2012grade} proposed a probabilistic graphical model that jointly models the difficulties of questions, the abilities of participants and the
correct answers to questions in aptitude testing and crowdsourcing settings.
While in MOOCs, \citet{champaign2014correlating} investigated the correlations between resource use and students' skill and relative skill improvement measured by IRT.
\cite{colvin2014comparing} analysed pre-post test questions using IRT, to compare the learning in MOOCs and a blended on-campus course.
Past work has tended to focus on using already-devised items to measure student ability under IRT models, while we are interested in automatically discovering content-based items that are characteristic of measurement in MOOCs~\citep{milligan2015crowd}.

\subsection{MOOC Forums}
MOOC forums have been of great interest recently, due to the availability of rich textual data and social behaviour. 
Various studies have been conducted such as sentiment analysis, community finding, question recommendation, answers \& intervention prediction. 
\citet{wen2014sentiment} use sentiment analysis to monitor students' trending opinions towards the course and to correlate sentiment with dropouts over time using survival analysis.
\citet{yang2015exploring} predict students' confusion during learning activities as expressed in discussion forums, using discussion behaviour and clickstream data; they further explore the impact of confusion on student dropout. 
\citet{ramesh2015weakly} predict sentiment in MOOC forums using hinge-loss Markov random fields. 
\citet{gillani2014communication} find communities using Bayesian Non-Negative Matrix Factorisation.
\citet{yang2014question} recommend questions of interest to students by designing a context-aware matrix factorisation model considering constraints on students and questions.
MOOC forum data has also been leveraged in the task of predicting accepted answers to forum questions~\citep{jendersanswer} and predicting instructor intervention~\citep{chaturvedi2014predicting}.
Despite the variety of studies, little machine learning research has explored forum discussions for the purpose of measurement in MOOCs.

\section{Preliminaries and Problem Formulation}
\label{sec:preps}
We choose NMF as the basic approach to discover forum topics due to the interpretability of the topics produced, and the extensibility of its optimisation formulation. 
For the IRT model for measurement, we focus on the Rasch model for dichotomous data due to its popularity, and due to being the basis for many extensions in education and psychology.
We next overview the Rasch model for dichotomous data and NMF, and then define our problem.

\subsection{Rasch Model}\label{sec:rasch}
The Rasch model~\citep{wright1982rating, bond2001applying} for dichotomous data (correct/incorrect, agree/disagree responses) specifies the probability of a person's positive response (correct, agree) on an item as a logistic function of the difference between the person's ability and item difficulty, 
\begin{equation}\label{equ:nmfguttmanobj}
\begin{aligned}
p_{ij}=P(X_{ij}=1|\beta_{i},\theta_{j})
=\frac{1}{1+\exp\left(-\left(\theta_{j}-\beta_{i}\right)\right)}\enspace,
\end{aligned}\end{equation}
where latent $\theta_j\in\reals$ denotes person $j$'s ability, latent $\beta_i\in\reals$ denotes item $i$'s difficulty, $X_{ij}\in\{0,1\}$ denotes person $j$'s observed random response on item $i$, and $p_{ij}$ is the probability of this response being positive.
This probability is best illustrated with the Item Characteristic Curve (ICC) as depicted in Figure~\ref{fig:icc} and commonly used in the field of IRT. It can be seen that the higher a person's ability is, relative to the difficulty of an item, the higher the probability of a positive response on that item. When a person's ability is equal to an item's difficulty on the latent scale, positive responses are observed with 0.5 probability.

\begin{figure}[!htb]
    \centering
    \includegraphics[scale=0.50]{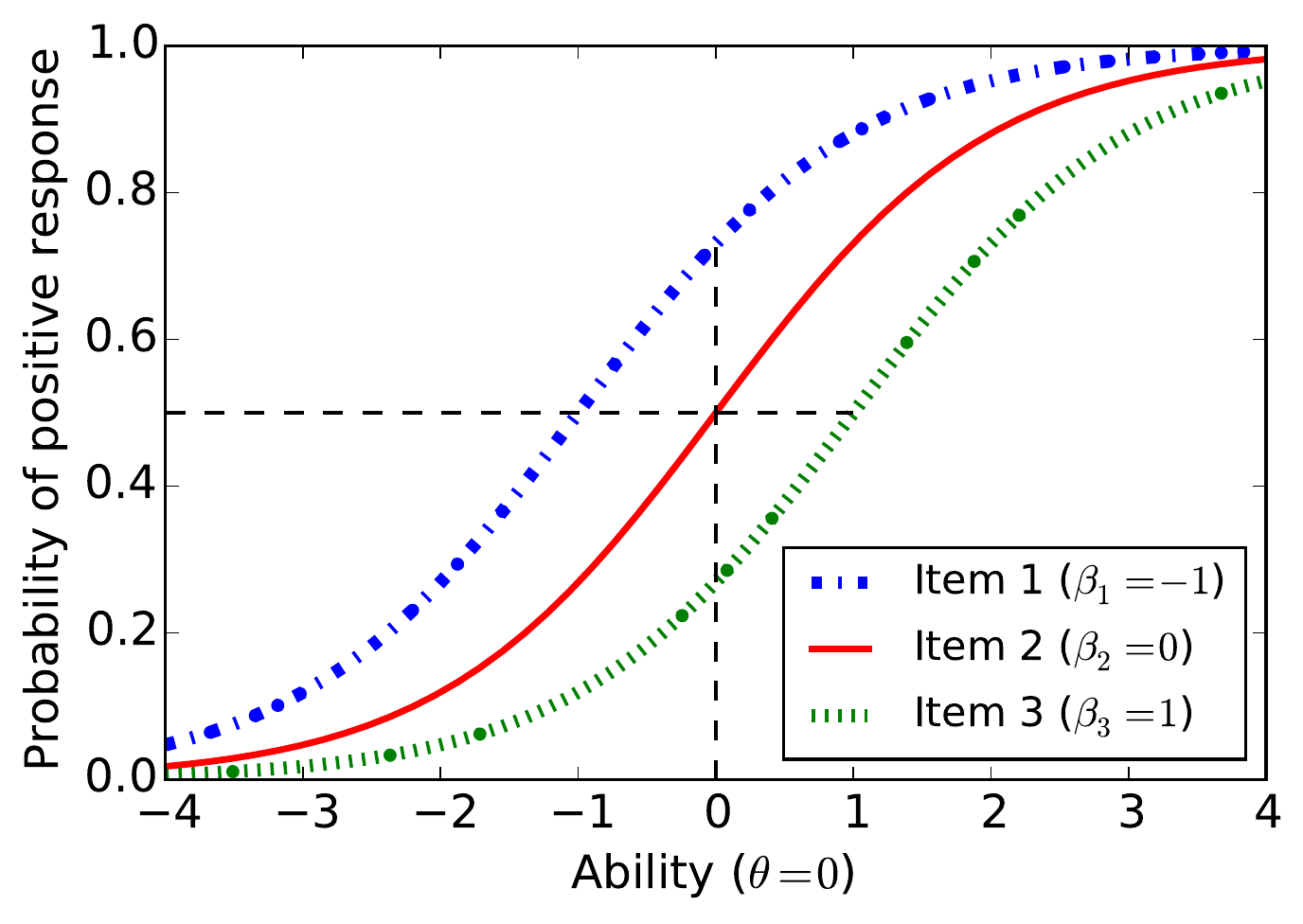}
    \caption{The Item Characteristic Curves for three items (item 1--the easiest, 3--the most difficult). A person with ability $\theta=0$ has 0.5 probability of responding positively on item 2 with difficulty $\beta=0$, and higher (and lower) probability on the easiest item 1 (most difficult item 3, respectively).}
    \label{fig:icc}
\end{figure}

The latent measurement scale is analogous to the ruler shown in Figure~\ref{fig:flow}, where persons and items are placed together and can be compared meaningfully.
The Rasch model provides a way to construct the ruler using persons' responses on items.
Persons and items are located along the scale according to their abilities $\theta_j$ and difficulties $\beta_i$ respectively.

The Rasch model can be viewed as a stochastic counterpart to the Guttman scale.
For example, in Figure~\ref{fig:flow}, person 1 and person 2 will have positive response on item 1 in a Guttman scale.  While in a Rasch scale, there are certain probabilities that person 1 and person 2 will enjoy positive responses on item 1, with person 1's probability being higher.
%
This error model leads to a higher level of measurement scale: the interval scale, where we can tell how much more able person 2 is compared to person 1. From the Guttman scale, by comparison, we can tell that person 2 is better than person 1 but not by how much.


\begin{table}[!htb]
    \centering
    \scriptsize
    \renewcommand{\arraystretch}{1.5}
    \caption{An example of items for measuring basic mathematical ability, students' responses, initial item difficulty estimates and student ability estimates.}
    \begin{tabular}{|P{1.7cm}|p{0.9cm}p{0.89cm}p{0.89cm}p{0.89cm}p{0.89cm}|P{1.30cm}|P{1.35cm}|}
        \hline
        & \textbf{Item 1}& \textbf{Item 2} & \textbf{Item 3} & \textbf{Item 4} & \textbf{Item 5} & \textbf{Proportion}  & \textbf{Ability $\theta_j^0$}\\
        & \textbf{(Count)}  & \textbf{($+$)} & \textbf{($-$)} & \textbf{($\times$)} & \textbf{($\div$)} & \textbf{correct $p_{\theta_j}$} & \textbf{$\log\left(\frac{p_{\theta_j}}{1-p_{\theta_j}}\right)$}\\ 
        \hline
        
        \textbf{Person 1} & 1  & 0 & 0 & 0 & 0 & 0.20 & -1.39\\
        \textbf{Person 2} & 1 & 1 & 0 & 0 & 0 & 0.60 & 0.41\\
        \textbf{Person 3} & 0 & 1 & 1 & 0 & 0 & 0.60 & 0.41\\
        \textbf{Person 4} & 1 & 0 & 1 & 1 & 0 & 0.67 & 0.71\\
        \textbf{Person 5} & 1 & 1 & 1 & 0 & 1 & 0.80 & 1.39\\
        \hline
        \textbf{Proportion correct $p_{\beta_i}$} & \multirow{2}*{0.80} & \multirow{2}*{0.33} & \multirow{2}*{0.33} & \multirow{2}*{0.20} & \multirow{2}*{0.20} & & \\
        \hline
        \textbf{Difficulty $\beta_i^0$ $\log\left(\frac{1-p_{\beta_i}}{p_{\beta_i}}\right)$} & \multirow{2}*{-1.39} & \multirow{2}*{0.71} & \multirow{2}*{0.71} & \multirow{2}*{1.39} & \multirow{2}*{1.39} & &\\
        \hline
    \end{tabular}%
    \label{tab:rasch}%
\end{table}%

Table~\ref{tab:rasch} further illustrates our setup, with an example of items for measuring basic mathematical ability, alongside hypothetical students' responses. 
The initial estimates (see Equations~\ref{equ:thetaini},\ref{equ:betaini} below) for item difficulties and person abilities are produced on a logit scale.
For example, if person 1 responds to the items positively 20\% of the time and negatively 80\% of the time, then the person's initial ability estimate is approximately $-1.39$ by taking the natural logarithm of the odds ratio for positive response $\frac{0.2}{0.8}$.

\subsubsection{Rasch Estimation}\label{sec:raschesti}
Given an observed response matrix \vbf{x}=$[x_{ij}]$ (\eg Table~\ref{tab:rasch}), a basic goal is to estimate the person and item parameters $\theta_j$ and $\beta_i$. The most common estimation methods are based on maximum-likelihood estimation, including: jointly maximum-likelihood (JML) estimation, conditional maximum-likelihood (CML) estimation and marginal maximum-likelihood (MML) estimation~\citep{baker2004item}. In this paper, we focus on JML.

Under the assumption that a sample of $n$ persons is drawn independently at random from a population of persons possessing a latent skill attribute, and the assumption of local independence that a person's responses to different items are statistically independent, the probability of an observed data matrix $\vbf{x}=[x_{ij}]$ with $k$ items and $n$ persons is the product of the probabilities of the individual responses, and can be given by the joint likelihood function 
\begin{equation}
\begin{aligned}
\mathcal{L}(\vbf{\beta},\vbf{\theta} | \vbf{x}) = &
\prod_{i=1}^{k} \prod_{j=1}^{n}
P(X_{ij}=1|\beta_i,\theta_j)^{x_{ij}}
\left(1-P\left(X_{ij}=1|\beta_i,\theta_j\right)\right)^{(1-x_{ij})}\\
=&\prod_{i=1}^{k} \prod_{j=1}^{n}
\frac{\exp\left(x_{ij}\left(\theta_j-\beta_i\right)\right)}
{1+\exp(\theta_j-\beta_i)}\enspace.
\end{aligned}
\end{equation}

The log-likelihood function is then
\begin{equation}
\begin{aligned}
\log\mathcal{L}(\vbf{\beta},\vbf{\theta} | \vbf{x}) &= 
\sum_{i=1}^{k}\sum_{j=1}^{n}x_{ij}(\theta_j-\beta_i)-\sum_{i=1}^{k}\sum_{j=1}^{n}\log(1+\exp\left(\theta_j-\beta_i)\right)\enspace.
\end{aligned}
\end{equation}

The parameters of the Rasch model can be estimated by joint maximum likelihood---maximisation of this expression---using Newton-Raphson~\citep{bertsekas1999nonlinear}, which yields the following iterative solution for $\beta{i}$ and $\theta{j}$,
\begin{equation}
\beta_{i}^{t+1}=\beta_{i}^{t}-\frac{\sum_{j=1}^{n}(p_{ij}-x_{ij})}
{-\sum_{j=1}^{n}p_{ij}(1-p_{ij})} \quad \text{for} \quad t\geq 0 \enspace,
\end{equation}
\begin{equation}
\theta_{j}^{t+1}=\theta_{j}^{t}-\frac{\sum_{i=1}^{k}(x_{ij}-p_{ij})}
{-\sum_{i=1}^{k}p_{ij}(1-p_{ij})} \quad \text{for} \quad t\geq 0\enspace.
\end{equation}

The convergence to a local optimum (with suitable step sizes) is guaranteed.
The initial estimates of $\theta_j$, $\theta_j^0$ can be obtained by firstly calculating the proportion of items that a person $j$ responded correctly $p_{\theta_j}$, and then taking the natural logarithm of the odds of person $j$'s correct response as shown in Table~\ref{tab:rasch}, which can be formalised as follows:
\begin{equation}\label{equ:thetaini}
    \theta_j^0=\log\left(\frac{p_{\theta_j}}{1-p_{\theta_j}}\right),  \enspace p_{\theta_j}=\frac{r_j}{k}, \enspace r_j=\sum_{i=1}^{k}{x_{ij}}\enspace,
\end{equation}
where $r_j$ denotes the number of items that person $j$ responded to positively.
Similarly, the initial estimates of $\beta_i$, $\beta_i^0$ can be obtained by
\begin{equation}\label{equ:betaini}
    \beta_i^0=\log\left(\frac{1-p_{\beta_i}}{p_{\beta_i}}\right),  \enspace p_{\beta_i}=\frac{s_i}{n}, \enspace s_i=\sum_{j=1}^{n}{x_{ij}}\enspace,
\end{equation}
where $s_i$ denotes the number of persons who responded correctly on item $i$, and $p_{\beta_i}$ denotes the proportion of persons who responded correctly on item $i$. 

For those items receiving no correct responses ($s_i=0$), or no incorrect responses ($s_i=n$), some implementations of the Rasch model will delete the item, while other models handle the situation as follows~\citep{baker2004item}, where $\epsilon$ is a small number (\eg 1.0 is used in our experiments), 
\begin{equation*}
s_i=\begin{cases}
\epsilon,  & \text{if}\,\, s_i=0\\
n-\epsilon,  & \text{if}\,\, s_i=n 
\end{cases}\enspace,\enspace\;\;\;\;
r_j=\begin{cases}
\epsilon,  & \text{if}\,\, s_i=0\\
k-\epsilon,  & \text{if}\,\, s_i=k 
\end{cases}\enspace.
\end{equation*} 
These pseudo counts are similar to frequentist Laplace corrections, or (weak) uniform Bayesian priors.

\subsubsection{Evaluating Model Fit}\label{sec:infit}
A set of items is said to measure a latent attribute on an interval scale when there is a close fit between data and model. 
The model-data fit is typically examined using infit and outfit statistics---two types of mean square error statistics----conveying information about the error in the estimates for each individual item and person.

Outfit and infit test statistics are defined for each item and person to test the fit of items and persons under the Rasch model, by carefully summarising the Rasch residuals.
The Rasch residuals are the differences between the observed responses and the expected responses according to the Rasch model. Formally, the expected response of person $j$ on item $i$ under the Rasch model $\Exp[x_{ij}]$ (abbreviated to $E_{ij}$) is $\Exp[X_{ij}]=p_{ij}$.
The residual between the observation $x_{ij}$ and the expected response $E_{ij}$ is then $R_{ij}=x_{ij}-E_{ij}$.
Standardised residuals are often used to assess the fit of a single person-item response
\begin{equation}
Z_{ij}=\frac{X_{ij}-E_{ij}}{\sqrt{\Var(X_{ij}-E_{ij})}}=\frac{R_{ij}}{\sqrt{\Var(X_{ij})}}\enspace,
\end{equation}
where $\Var(X_{ij})=p_{ij}(1-p_{ij})$
denotes the variance of $X_{ij}$ (abbreviated to $S_{ij}$).

The outfit of item $i$ summarises the squared standardised residuals, averaged over $n$ persons,
\begin{equation}
\text{Outfit}_{i}=\frac{1}{n}\sum_{1}^{n}Z_{ij}^2=
\frac{1}{n}\sum_{1}^{n}\frac{R_{ij}^2}{S_{ij}}\enspace.
\end{equation}
Typical treatments assume standardised residuals $Z_{ij}$ approximately following a unit normal distribution. Their sum of squares therefore approximately follows a $\chi^2$ distribution. Dividing this sum by its degrees of freedom yields a mean-square value, with an expectation of 1.0 and taking values in the range of 0 to infinity. 

Outfit is sensitive to unexpected responses to items, \eg
lucky guesses (\eg a person responds 111001) or careless sequences of mistakes (\eg a person responds 010100)~\citep{linacre2002infit}. Since outfit is sensitive to the very unexpected observations (outliers), infit was devised to be more sensitive to the overall pattern of responses~\citep{linacre2006misfit}. 
Infit is an information-weighted form of outfit: it weights the observations by their statistical information (model variance) which is larger for targeted observations, and smaller for extreme observations~\citep{bond2001applying}. In this paper, we focus on infit. Formally, the infit of item $i$ is given by
\begin{equation}
\text{Infit}_{i}=\frac{\sum_{j=1}^{n}S_{ij}Z_{i j}^2}{\sum_{j=1}^{n}S_{ij}}
=\frac{\sum_{j=1}^{n}R_{ij}^2}{\sum_{j=1}^{n}S_{ij}}\enspace.
\end{equation}

Both outfit and infit have the expected value of 1.0. Values larger than 1.0 indicate model underfitting, \ie data is less predictable than the model expects, while values less than 1.0 indicate overfitting, \ie observations are highly predictable~\citep{wright1994reasonable}. Conventionally, the acceptable range is usually taken to be [0.7,1.3] or [0.8,1.2] depending on application.

\subsection{Non-Negative Matrix Factorisation (NMF)}
Given a non-negative matrix $\vbf{V}\in \mathbb{R}^{m\times n}$  and a positive integer $k$, NMF factorises $\vbf{V}$ into the product of a non-negative matrix $\vbf{W}\in \mathbb{R}^{m\times k}$ and a non-negative matrix $\vbf{H}\in \mathbb{R}^{k\times n}$ such that
\begin{eqnarray*}
    \vbf{V}\approx \vbf{WH}
\end{eqnarray*}

        \begin{center}
            \centering
            $
            \matbox{7}{7}mn{V} \approx 
            \matbox{7}{4}mk{W} \times
            \matbox{4}{7}kn{H}$
            \enspace.
        \end{center}

A commonly-used measure for quantifying the quality of this approximation is the Frobenius norm between $\vbf{V}$ and $\vbf{WH}$. Thus, NMF involves solving
\begin{eqnarray}
\argmin_{\vbf{W},\vbf{H}} 
\norm{\vbf{V}-\vbf{WH}}  
\quad \ensuremath{\mbox{s.t.\xspace}} \quad \vbf{W}\geq \vbf{0},\ \ \vbf{H}\geq \vbf{0} \enspace. \label{equ:nmf}
\end{eqnarray}
This objective function is convex in $\vbf{W}$ and $\vbf{H}$ separately, but not together. Therefore standard convex solvers are not expected to find a global optimum in general. The multiplicative update algorithm~\citep{{lee2001algorithms}} is commonly used to find a local optimum,
where $\vbf{W}$ and $\vbf{H}$ are updated by a multiplicative factor that depends on the quality of the approximation.


        \begin{figure}[t!]
		\begin{minipage}[t]{1.0\textwidth}
               \centering
                   \[
                   \renewcommand{\arraystretch}{0.8}
                   \underset{\vbf{V}}{\bordermatrix{
                      \hspace{0.2cm} ~ & \text{stud}1 & \text{stud}2 & \cdots &  \text{stud}n      \cr
                       \text{solver} & 0.26 & 0.11 & \cdots & 0.52      \cr
                       \text{optim} & 0.32 & 0.18 & \cdots & 0.06      \cr
                       \text{code} & 0.68 & 0.01 & \cdots & 0.83      \cr
                       \text{algorithm} & 0.89 & 0.61 & \cdots & 0.44      \cr
                       \vdots & \vdots & \vdots &  \ddots &  \vdots     \cr
                       \text{word}m & 0.22 & 0.54 & \cdots & 0.98    
                    }\vspace{0.5cm}}
                    \]
	    \end{minipage} \\
	    \begin{minipage}[t]{0.48\textwidth}
		    \centering
                   \[
                      \underset{\vbf{W}}{\bordermatrix{
                      \hspace{0.2cm} ~ & \text{topic}1 & \text{topic}2 & \cdots &  \text{topic}k      \cr
                       \text{solver} & 0.22 & 0.01 & \cdots & 0.12      \cr
                       \text{optim} & 0.38 & 0.15 & \cdots & 0.06      \cr
                       \text{code} & 0.18 & 0.05 & \cdots & 0.03      \cr
                       \text{algorithm} & 0.09 & 0.21 & \cdots & 0.01      \cr
                       \vdots & \vdots & \vdots &  \ddots &  \vdots     \cr
                       \text{word}m & 0.02 & 0.04 & \cdots & 0.12      
                    }\vspace{0.5cm}}
                    \]
	    \end{minipage}\hfill
	    \begin{minipage}[t]{0.48\textwidth}
		    \vspace{1em}
                    \centering
                    \[
                     \underset{\vbf{H}}{\bordermatrix{
                      \hspace{0.2cm} ~ & \text{stud}1 & \text{stud}2 & \cdots &  \text{stud}k      \cr
                       \text{topic}1 & 0.83 & 0.17 & \cdots & 0.04      \cr
                       \text{topic}2 & 0.21 & 0.75 & \cdots & 0.16      \cr
                       \vdots & \vdots & \vdots & \ddots & \vdots      \cr
                       \text{topic}k & 0.09 & 0.64 &  \cdots & 0.62      
                    }\vspace{0.5cm}}
                    \]
	    \end{minipage}
                   \caption{Example matrices: word-student $\vbf{V}$, word-topic $\vbf{W}$, topic-student $\vbf{H}$.}
                   \label{fig:examplev}
            \end{figure}

In the present MOOC setting, 
we focus on the students who contributed posts or comments
in forums. For each student, we aggregate all posts or comments that
they contributed. Each student is represented by a bag of words as shown in the example word-student matrix $\vbf{V}$ in Figure~\ref{fig:examplev}, where $m$ represents
the number of words, and $n$ represents the number of students. 
Using NMF, a word-student matrix $\vbf{V}$ can be factorised
into two non-negative matrices: word-topic matrix $\vbf{W}$ and topic-student matrix $\vbf{H}$.
For each student, the column vector of $\vbf{V}$ is approximated by a linear combination of the columns of $\vbf{W}$, weighted by the components of $\vbf{H}$. Therefore, each column vector of $\vbf{W}$ can be regarded as a topic, and the memberships of students in these topics are encoded by $\vbf{H}$ as shown in Figure~\ref{fig:examplev}.

        

\subsection{Problem Statement}\label{sec:ps}

We seek to explore the feasibility of automatic discovery of forum discussion
topics for measuring students' academic abilities in MOOCs, as quantified by the Rasch
model. Our central tenet is that topics can be regarded as useful items for measuring a latent skill, if student responses to these topics are well fit by the Rasch model, and if the topics are interpretable to domain experts for educational relevance. 
Therefore, we need to discover topics from students' posts and comments in MOOC forums, in
such a way that students' participation across these topics fits the Rasch model.
Student item response records whether a student posts on the corresponding
topic or not. After discovery, topics must then be further assessed for interpretability
to domain experts. Our goal is decision support.

In particular, under the NMF framework, a word-student matrix $\vbf{V}$ can be factorised into two non-negative matrices: word-topic matrix $\vbf{W}$ and topic-student matrix $\vbf{H}$.
Our application requires that the topic-student matrix $\vbf{H}$ be \textbf{a)} binary ensuring the response of a student to a topic is dichotomous; \textbf{b)} useful for measuring students' academic abilities; 
and \textbf{c)} well-fit by the Rasch model.
NMF provides an elegant framework for incorporating these constraints via adding novel regularisation, as detailed in the next section.
A glossary of the symbols most used in this paper is given in Table~\ref{tab:symbols}.
\begin{table}[t!]
    \small
    \centering
    \caption{Glossary of symbols}
    \label{tab:symbols}
    \begin{tabular}{lll}
        \toprule
        Symbol		& 	Description \\
        \toprule
        $m$							    & the number of words \\ 
        $n$							    & the number of students \\
        $k$							    & the number of topics \\
        $\vbf{V}=(v_{ij})_{m\times n}$	&	word-student matrix \\
        $\vbf{W}=(w_{ij})_{m\times k}$	&	word-topic matrix \\
        $\vbf{H}=(h_{ij})_{k\times n}$	&	topic-student matrix \\
        $\vbf{H}_{ideal}=\left(\left(h_{ideal}\right)_{ij}\right)_{1\times n}$	&	matrix for students with ideal number of distinct topics posted\\
        $\vbf{1}_r$ & all-ones matrix with size $1\times n$\\        
        $g_j$ & student $j$'s grade\\
        $\vbf{\beta}=(\beta_{i})_{k}$	& item difficulty vector\\
        $\vbf{\theta}=(\theta_{j})_{n}$	& student ability vector\\
        $X_{ij}$	&	binary response (0 or 1) of person $j$ to item $i$\\
       	$x_{ij}$	&	observed response of person $j$ to item $i$\\
        $p_{ij}$    &   the probability of positive response of person $j$ to item $i$\\
        $S_{ij}$   &   variance of $X_{ij}$\\
        $Z_{ij}$ & standardised residual\\
        $\lambda_0, \lambda_1, \lambda_2, \lambda_3$							    & regularisation coefficients \\
        \bottomrule	
    \end{tabular}
\end{table}

\section{The \TopicResponse Algorithm: Joint NMF-Rasch Estimation}\label{sec:nmfraschg}

To favour topics that fit the Rasch model, we jointly optimise wwwboth NMF and Rasch models, which yields the objective function

    \begin{equation*}
    \begin{aligned}
    g(\vbf{W},\vbf{H}, \vbf{\theta}, \vbf{\beta}) =& \norm{\vbf{V}-\vbf{WH}} 
     - 
    \lambda_0f_R(\vbf{H, \theta, \beta})   
    \enspace,
    \end{aligned}
    \end{equation*}
\begin{equation*}\begin{aligned}
f_R(\vbf{H, \theta, \beta})= \sum_{i=1}^{k}\sum_{j=1}^{n}h_{ij}(\theta_j-\beta_i)-\sum_{i=1}^{k}\sum_{j=1}^{n}\log\left(1+\exp\left(\theta_j-\beta_i\right)\right)\enspace,
\end{aligned}\end{equation*}    
where $f_R(\vbf{H, \theta, \beta})$ is the log-likelihood function maximised in Rasch estimation, and $\lambda_0>0$ is a user-specified parameter controlling the trade-off between the quality of factorisation and Rasch estimation.

\paragraph{Weak supervision of item responses.} The fit between student topic responses $\vbf{H}$ and the Rasch model will provide statistical evidence of measuring skill attainment. However, it is difficult to conclude what the topics are measuring without domain knowledge. To favour the topics that can be used to measure students' academic abilities, we impose a constraint on $\vbf{H}$ based on some student grade, which provides an indicator of student's abilities (we discuss sources of auxiliary grade information below). In particular, we assume that there is the following relationship between the ideal number of distinct topics that each student $j$ contributes and their grade $g_j\in[0,100]$,
\begin{eqnarray*}
    ({h}_{ideal})_{1j}=\min\left\{\floor*{\frac{g_j+width}{width}}, k-1\right\}\enspace,\enspace
    width=\frac{100}{k-1}\enspace,
\end{eqnarray*}
where $\vbf{H}_{ideal}$ is a $1\times n$ matrix, denoting the ideal number of distinct topics posted by students. For example under $k=10$ items, student $j$ scoring $g_j=45$ should post on a number of topics $(h_{ideal})_{1j}=5$. 
The minimum and maximum number of different topics that a student $j$ posted is 1 and $k-1$ respectively. This is motivated by the initialisation of $\vbf{\theta}$ and $\vbf{\beta}$ as illustrated in Section~\ref{sec:raschesti}, where positive responses on 0 or $k$ topics is undesirable.

This supervision constraint is markedly weaker than a similar constraint found in \citep{he2016moocs}, as demonstrated in
Figure~\ref{fig:hidealcomp}. 
\cite{he2016moocs} leverage the student grade to exactly determine the item responses for the Guttman scale. The Guttman scale, as a deterministic model, requires that if a student can get a difficult item correct, they can also achieve correct responses on all easier items. This assumption is very restrictive, and rarely makes sense in practice. The Rasch model allows errors in the responses; and only constrains the number of distinct topics posted by a student, rather than the exact response pattern.

Most (MOOC) courses conduct multiple forms of assessment throughout the duration of teaching. For example, weekly quizzes, take-home assignments, mid-term tests, projects, presentations, etc. In the large-scale MOOC context, such evaluations may be peer-assessed. Students often enter courses with some cumulative grade-point average that may be (loosely) predictive of future performance. Any of these readily-available sources of student information could be reasonably used to seed $\vbf{H}_{ideal}$. Even final course grades could be used, particularly when the ultimate application of \TopicResponse is not measuring students, but refining curriculum.

\begin{figure}[b!]
               \begin{center}
                   \[
                   \vbf{H}_{ideal}~\text{for the Guttman scale}
                   \renewcommand{\arraystretch}{0.8}
                   \bordermatrix{
                      \hspace{0.2cm} \text{grade} & 8 & 25 & 46 & 67 & 89 &   98     &    78    &    35    &  55      \cr
                       ~ & 1 & 1 & 1 & 1 & 1 & 1 & 1 & 1 & 1      \cr
                       ~ & 0 & 1 & 1 & 1 & 1 & 1 & 1 & 1 & 1      \cr
                       ~ & 0 & 1 & 1 & 1 & 1 & 1 & 1 & 1 & 1      \cr
                       ~ & 0 & 0 & 1 & 1 & 1 & 1 & 1 & 1 & 1      \cr
                       ~ & 0 & 0 & 1 & 1 & 1 & 1 & 1 & 0 & 1      \cr
                       ~ & 0 & 0 & 0 & 1 & 1 & 1 & 1 & 0 & 1      \cr
                       ~ & 0 & 0 & 0 & 1 & 1 & 1 & 1 & 0 & 0      \cr
                       ~ & 0 & 0 & 0 & 0 & 1 & 1 & 1 & 0 & 0      \cr
                       ~ & 0 & 0 & 0 & 0 & 1 & 1 & 0 & 0 & 0      \cr
                       ~ & 0 & 0 & 0 & 0 & 0 & 1 & 0 & 0 & 0 
                    }
                    \]
                \end{center}
                
                \begin{center}
                    \[
                    \vbf{H}_{ideal}~\text{for the Rasch model}
                    \renewcommand{\arraystretch}{0.8}
                    \bordermatrix{
                        \hspace{0.5cm} \text{grade}  & 8 & 25 & 46 & 67 & 89 &   98     &    78    &    35    &  55      \cr
                        ~ & 1 & 3 & 5 & 7 & 9 & 9 & 8 & 4 & 5      
                    }
                    \]
                \end{center}
                    \caption{An example of $\vbf{H}_{ideal}$ in the Guttman scale and the Rasch model.}
                    \label{fig:hidealcomp}
\end{figure}

In order to encourage satisfaction of the $\vbf{H}_{ideal}$ soft constraint on topic responses, we introduce a regularisation term on $\vbf{H}$, namely $\norm{\vbf{1}_r\vbf{H}-\vbf{H}_{ideal}}$.

\paragraph{Quantising \& Regularising the Response Matrix.}  We introduce regularisation term $\norm{\vbf{W}}$, commonly used to prevent overfitting in NMF. To encourage binary solutions, we impose an additional regularisation term 
$\norm{\vbf{H}\circ\vbf{H}-\vbf{H}}$, where operator $\circ$ denotes the Hadamard product. 
Binary matrix factorisation (BMF) is a variation of NMF, where the input matrix and the two factorised matrices are all binary. 
Our approach is inspired by those of \cite{zhang2007binary} and \cite{zhang2010binary}. Our added term equals $\norm{\vbf{H}\circ\left(\vbf{H}-\vbf{1}\right)}$, which is minimised by (only) binary $\vbf{H}$.

\paragraph{TopicResponse Model.} We have the following regularisations:
\begin{itemize}
    \item $\norm{\vbf{1}_r\vbf{H}-\vbf{H}_{ideal}}$ to encourage a grade-guided $\vbf{H}$; 
    \item $\norm{\vbf{W}}$ to prevent overfitting; and
    \item $\norm{\vbf{H}\circ\vbf{H}-\vbf{H}}$ to encourage a binary item-response solution. 
\end{itemize}
These terms together with joint NMF-Rasch estimation yield final objective
    \begin{equation}\label{equ:fun2}
    \begin{aligned}
    f(\vbf{W},\vbf{H}, \vbf{\theta}, \vbf{\beta}) =& \norm{\vbf{V}-\vbf{WH}} 
    -
    \lambda_0f_R(\vbf{H, \theta, \beta})
    + 
    \lambda_1\norm{\vbf{W}} \\
    &+\lambda_2\norm{\vbf{1}_r\vbf{H}-\vbf{H}_{ideal}} 
    +
    \lambda_3\norm{\vbf{H}\circ\vbf{H}-\vbf{H}}
    \enspace,
    \end{aligned}
    \end{equation}
where $\lambda_1, \lambda_2, \lambda_3>0$ are user-specified regularisation parameters, with primal program
\begin{equation}\label{equ:nmfraschprog}
	\argmin_{\vbf{W},\vbf{H},\vbf{\theta},\vbf{\beta}} f(\vbf{W},\vbf{H}, \vbf{\theta}, \vbf{\beta}) 
\ \ \  \mbox{s.t.} \ \ \   \vbf{W}\geq \vbf{0},\ \vbf{H}\geq \vbf{0}\enspace.
\end{equation}

\paragraph{TopicResponse Fitting Procedure.}
A local optimum of program~\eqref{equ:nmfraschprog} is achieved via iteration
\begin{eqnarray}
w_{ij} &\leftarrow&
w_{ij}
\frac{(\vbf{V}\vbf{H}^T)_{ij}}
{(\vbf{WH}\vbf{H}^T+\lambda_0\vbf{W})_{ij}} \label{equ:w1} \\
h_{ij} &\leftarrow&
h_{ij}
\frac{2(\vbf{W}^T\vbf{V})_{ij}+8\lambda_2h_{ij}^3+6\lambda_2h_{ij}^2+2\lambda_1(\vbf{1}_r^T\vbf{H}_{ideal})_{ij}+\lambda_3(\vbf{\theta}-\vbf{\beta})_{ij}^+}
{2(\vbf{W}^T\vbf{WH})_{ij}+12\lambda_2h_{ij}^3+2\lambda_1(\vbf{1}_r\vbf{H})_{ij} + 2\lambda_2h_{ij} +\lambda_3 (\vbf{\theta}-\vbf{\beta})_{ij}^-} 
\label{equ:h1}
\\
\beta_{i} &\leftarrow&\beta_{i}-\frac{\sum_{j=1}^{n}(p_{ij}-h_{ij})}
{-\sum_{j=1}^{n}p_{ij}(1-p_{ij})}
\label{equ:beta1}
\\
\theta_{j} &\leftarrow&\theta_{j}-\frac{\sum_{i=1}^{k}(h_{ij}-p_{ij})}
{-\sum_{i=1}^{k}p_{ij}(1-p_{ij})}
\label{equ:theta1}
\end{eqnarray}
where 
\begin{eqnarray*}
	(\vbf{\theta}-\vbf{\beta})&=&(\vbf{\theta}-\vbf{\beta})^+ - (\vbf{\theta}-\vbf{\beta})^-\\
	(\vbf{\theta}-\vbf{\beta})_{ij}^+ &=& 
\begin{cases}(\vbf{\theta}-\vbf{\beta})_{ij} & \phantom{\infty}\text{if}\,\, (\vbf{\theta}-\vbf{\beta})_{ij} > 0 \\
0 & \phantom{\infty}\text{if}\,\, \text{otherwise} \\
\end{cases} \\
(\vbf{\theta}-\vbf{\beta})_{ij}^- &=& 
\begin{cases}
-(\vbf{\theta}-\vbf{\beta})_{ij} & \phantom{\infty}\text{if}\,\, (\vbf{\theta}-\vbf{\beta})_{ij} < 0 \\
0 & \phantom{\infty}\text{if}\,\, \text{otherwise} \\
\end{cases}
\end{eqnarray*}
$(\vbf{\theta}-\vbf{\beta})^+$ and $(\vbf{\theta}-\vbf{\beta})^-$ denote the positive part and negative part of matrix $(\vbf{\theta}-\vbf{\beta})$ respectively.  We next describe how these update rules are derived.

The update rules~\eqref{equ:beta1} and~\eqref{equ:theta1} can be obtained using Newton's method. The update rules~\eqref{equ:w1} and~\eqref{equ:h1} can be derived via the Karush-Kuhn-Tucker conditions necessary for local optimality. First we construct the unconstrained Lagrangian 
\begin{equation*}
\mathcal{L}(\vbf{W},\vbf{H},\vbf{\theta}, \vbf{\beta},\vbf{\alpha},\vbf{\gamma})=f(\vbf{W},\vbf{H},\vbf{\theta}, \vbf{\beta})+\text{tr}(\vbf{\alpha}\vbf{W})+\text{tr}(\vbf{\gamma}\vbf{H})\enspace,
\end{equation*}
where $\alpha_{ij}, \gamma_{ij}\leq 0$ are the Lagrangian dual variables for inequality constraints $w_{ij}\geq 0$ and $h_{ij}\geq 0$ respectively, and $\vbf{\alpha}=[\alpha_{ij}]$, $\vbf{\gamma}=[\gamma_{ij}]$ denote their corresponding matrices.
The KKT condition of stationarity requires  that the derivative of $\mathcal{L}$ with respect to $\vbf{H}$,  vanishes at a local optimum $\vbf{H}^\star,\vbf{W}^\star,\vbf{\alpha}^\star,\vbf{\gamma}^\star$:
\begin{equation*}
\begin{split}
\frac{\partial \mathcal{L}}{\partial \vbf{W}}=&2\left(\vbf{W}^\star\vbf{H}^\star{\vbf{H}}^{\star T}- \vbf{V}\vbf{H}^{\star T} + \lambda_0\vbf{W}^\star\right)  + \vbf{\alpha}^\star = \vbf{0} \enspace, \\
\frac{\partial \mathcal{L}}{\partial \vbf{H}} = & 2\Big(
\vbf{W}^{\star T}\vbf{W}^\star\vbf{H}^\star 
- \vbf{W}^{\star T}\vbf{V}
+ \lambda_1 \vbf{1}_r^T\vbf{1}_r\vbf{H}+\lambda_2 \vbf{H}^\star 
- \lambda_1 \vbf{1}_r^T\vbf{H}_{ideal} \Big)
\\&+ 4\lambda_{2}\vbf{H}^\star\circ\vbf{H}^\star\circ\vbf{H}^\star
- 6\lambda_2\vbf{H}^\star\circ\vbf{H}^\star 
-\lambda_3 \left(
(\vbf{\theta - \vbf{\beta}})^+ - (\vbf{\theta - \vbf{\beta}})^-
\right)
+ \vbf{\gamma}^\star \\
=& \vbf{0}\enspace.
\end{split}\end{equation*}
Complementary slackness $\gamma^\star_{ij}h^\star_{ij}=0$, implies:
\begin{eqnarray*}
        0 &=& \left(\vbf{V}\vbf{H}^{\star T} - \vbf{W}^\star\vbf{H}^\star\vbf{H}^{\star T} - \lambda_0\vbf{W}^\star\right)_{ij}w^\star_{ij} \enspace, \label{equ:W} \\
    0 &= &  \Big(2\vbf{W}^{\star T}\vbf{V} +
    6\lambda_2\vbf{H}^\star\circ\vbf{H}^\star + 2\lambda_1 \vbf{1}_r^T\vbf{H}_{ideal} -2\vbf{W}^{\star T}\vbf{W}^\star\vbf{H}^\star 
    - 2\lambda_1\vbf{1}_r^T\vbf{1}_r\vbf{H}^\star
    \nonumber \\ &&
    -4\lambda_{2}\vbf{H}^\star\circ\vbf{H}^\star\circ\vbf{H}^\star 
    -2\lambda_2 \vbf{H}^\star
    + \lambda_3 (\vbf{\theta}-\vbf{\beta})^+ - \lambda_3 (\vbf{\theta}-\vbf{\beta})^-
     \nonumber \\ &&
    + 8\lambda_{2}\vbf{H}^\star\circ\vbf{H}^\star\circ\vbf{H}^\star
    -
    8\lambda_{2}\vbf{H}^\star\circ\vbf{H}^\star\circ\vbf{H}^\star
    \Big)_{ij}h^\star_{ij} \enspace. \label{equ:H}
\end{eqnarray*}
These two equations lead to the updating rules~\eqref{equ:w1} and~\eqref{equ:h1}. 
Regarding the update rules~\eqref{equ:w1}, \eqref{equ:h1}, \eqref{equ:beta1} and~\eqref{equ:theta1} we have the following theorem:
\begin{thm}\label{thm:thm2}
    The objective function $f(\vbf{W},\vbf{H},\vbf{\theta}, \vbf{\beta})$ of \TopicResponse program~\eqref{equ:nmfraschprog} is non-increasing under update rules~\eqref{equ:w1}, \eqref{equ:h1}, \eqref{equ:beta1} and~\eqref{equ:theta1}.
\end{thm}

This result guarantees that the update rules of $\vbf{W}$, $\vbf{H}$, $\vbf{\theta}$ and $\vbf{\beta}$ eventually converge, and that the obtained solution will be a local optimum. 
The proof of Theorem~\ref{thm:thm2} is given in the Appendix.

\begin{algorithm}[b]
    \caption{\TopicResponse} \label{alg:nmfrasch}
    \renewcommand{\arraystretch}{2}
    \begin{algorithmic}[1]
        \REQUIRE ~~\\
        $\vbf{V}$, $\vbf{H}_{ideal}$, $\lambda_0$, $\lambda_1$, $\lambda_2$, $\lambda_3$,$k$;
        \ENSURE ~~\\
        A topic-student matrix, $\vbf{H}$, item difficulties $\vbf{\beta}$, person abilities $\vbf{\theta}$;
        
        \STATE Initialise $\vbf{W}$, $\vbf{H}$ using NMF;
        \STATE Normalise $\vbf{W}$, $\vbf{H}$ following ~\citep{zhang2007binary,zhang2010binary};
        \STATE Initialise $\vbf{\theta}$, $\vbf{\beta}$ based on Eq.~\eqref{equ:thetaini} and Eq.~\eqref{equ:betaini};
        \REPEAT
        \STATE Update $\vbf{W},\vbf{H},\vbf{\beta},\vbf{\theta}$ iteratively based on Eq.~(\ref{equ:w1}) to Eq.~(\ref{equ:theta1});
        \UNTIL{converged}
        \RETURN $\vbf{H}$;
    \end{algorithmic}
\end{algorithm}

\paragraph{Algorithm.} 
Our overall approach \TopicResponse is described as Algorithm~\ref{alg:nmfrasch}.
$\vbf{W}$ and $\vbf{H}$ are initialised using plain NMF~\citep{lee1999learning,lee2001algorithms}, then normalised~\citep{zhang2007binary,zhang2010binary}. $\vbf{\theta}$ and $\vbf{\beta}$ are initialised based on Eq.~\eqref{equ:thetaini} and Eq.~\eqref{equ:betaini}, where $x_{ij}$ is replaced by $h_{ij}$. At optimisation completion, estimates for topics, item difficulties and person abilities can be obtained together. Code for \TopicResponse is available from the authors' websites.

\section{Experiments}
\label{sec:exp}
We report on extensive experiments evaluating the effectiveness of \TopicResponse on real MOOCs.
In our experiments, we use the first offerings of three Coursera MOOCs from education, economics and computer science offered by The University of Melbourne: \textit{Assessment and Teaching of 21st Century Skills} delivered in 2014, \textit{Principles of Macroeconomics} delivered in 2013, and \textit{Discrete Optimisation} delivered in 2013. We denote these three courses by EDU, ECON and OPT respectively. 

\subsection{Dataset Preparation}
We focus on the students who contributed posts or comments in forums.
For each student, we aggregate all the posts and comments that they contributed.
After stemming and removing stop words, a word-student matrix with normalised tf-idf in [0,1] is produced.
The statistics of words and students before and after preprocessing, the dominated words, and the sparsity of word-student matrix (the percentage of non-zeros values) for three MOOCs are displayed in Table~\ref{tab:stat}.

\subsection{Baseline and Evaluation Metrics}
We compare our algorithm \TopicResponse with the baseline algorithm Grade-Guided NMF (\GGNMF), which minimises the following objection function

\begin{table}[b]
        \centering
        \caption{Statistics of our three Coursera MOOC datasets.}
        \label{tab:stat}
	\begin{tabular}{lccp{1.5cm}p{3.5cm}p{1.0cm}} 
            \toprule
            MOOC  & \#Students & \#Words & \#Words after preprocessing & Dominated words & Word-student matrix sparsity\\
            \toprule
            EDU & 1,749 & 28,931 & 18,391 & student, learn, skill, work, teacher, use, assess, teach, problem, collabor & 0.59\%\\ 
            ECON & 1,551 & 26,370  & 21,412 & gdp, would, econom, think, product, good, one, economi, increas, invest & 0.50\%\\ 
            OPT & 1,092 & 19,284 & 16,128 & use, solut, get, time, one, tri, python, work, optim, would & 0.85\%\\ 
            \bottomrule\end{tabular}
\end{table}

\begin{equation*}\label{equ:fun1}
\begin{aligned}
f_G(\vbf{W},\vbf{H}) =& \norm{\vbf{V}-\vbf{WH}} + 
\lambda_1\norm{\vbf{W}} +
\lambda_2\norm{\vbf{1}_r\vbf{H}-\vbf{H}_{ideal}} +
\lambda_3\norm{\vbf{H}\circ\vbf{H}-\vbf{H}}
\end{aligned}
\end{equation*}
A local optimum can be obtained using the Karush-Kuhn-Tucker conditions. Like \TopicResponse, \GGNMF regularises $\vbf{H}$ by considering the students' grades as an indicator of academic ability. 
The difference is that \TopicResponse optimises the Rasch estimation and NMF simultaneously, 
while in \GGNMF, the students' topic responses $\vbf{H}$ are first obtained, and then are passed through the Rasch model.
We evaluate the two algorithms in terms of the following metrics. \\[-0.5em]

\noindent\textit{Quality of factorisation.} We measure $\norm{\vbf{V}-\vbf{WH}}$ so as to record how well the factorisation approximates the student-word matrix.\\[-0.5em]

\noindent\textit{Measuring student academic ability.} Quality of constraint on students' topic participation, based on grades:  $\norm{\vbf{1}_r\vbf{H}-\vbf{H}_{ideal}}$.\\[-0.5em]

\noindent\textit{Negative log-likelihood.} Log-likelihood measures fit of the Rasch model to the entire dataset. For convenience, we look at the negative log-likelihood, which should be minimised: smaller is better. This measure is our main focus for Rasch, as it is important to examine the model-level fit before looking at item-level fit.\\[-0.5em]
    
\noindent\textit{Item infit.} As illustrated in Section~\ref{sec:infit}, item infit examines the fit of a particular item, with non-fitting items suitable for further refinement. We use the conventional acceptance range of [0.7, 1.3].




\begin{table}[b!]
\begin{minipage}[t!]{0.48\textwidth}
    \centering
    \caption{Hyperparameter settings. \label{tab:paras}}
    \begin{tabular}{ll}
        \toprule
        Param.		& 	Values Explored (\textbf{Default}) \\
        \toprule
        $\lambda_0$	&	$[0.01,\boldsymbol{0.1}, 0.2, 0.3, 0.4, 0.5]$ \\
        $\lambda_1$	&	$[10^{-3},10^{-2},\boldsymbol{10^{-1}},10^{0},10^{1},10^{2}]$ \\
        $\lambda_2$	&	$[10^{-3},10^{-2},10^{-1},\boldsymbol{10^{0}},10^{1},10^{2}]$ \\
        $\lambda_3$	&	$[10^{-3},10^{-2},10^{-1},\boldsymbol{10^{0}},10^{1},10^{2}]$ \\
        $k$	&	$[5,\boldsymbol{10},15,20,25,30]$ \\
        \bottomrule
    \end{tabular}
\end{minipage}\hfill
\begin{minipage}[t!]{0.48\textwidth}
    \centering
        \includegraphics[scale=0.38]{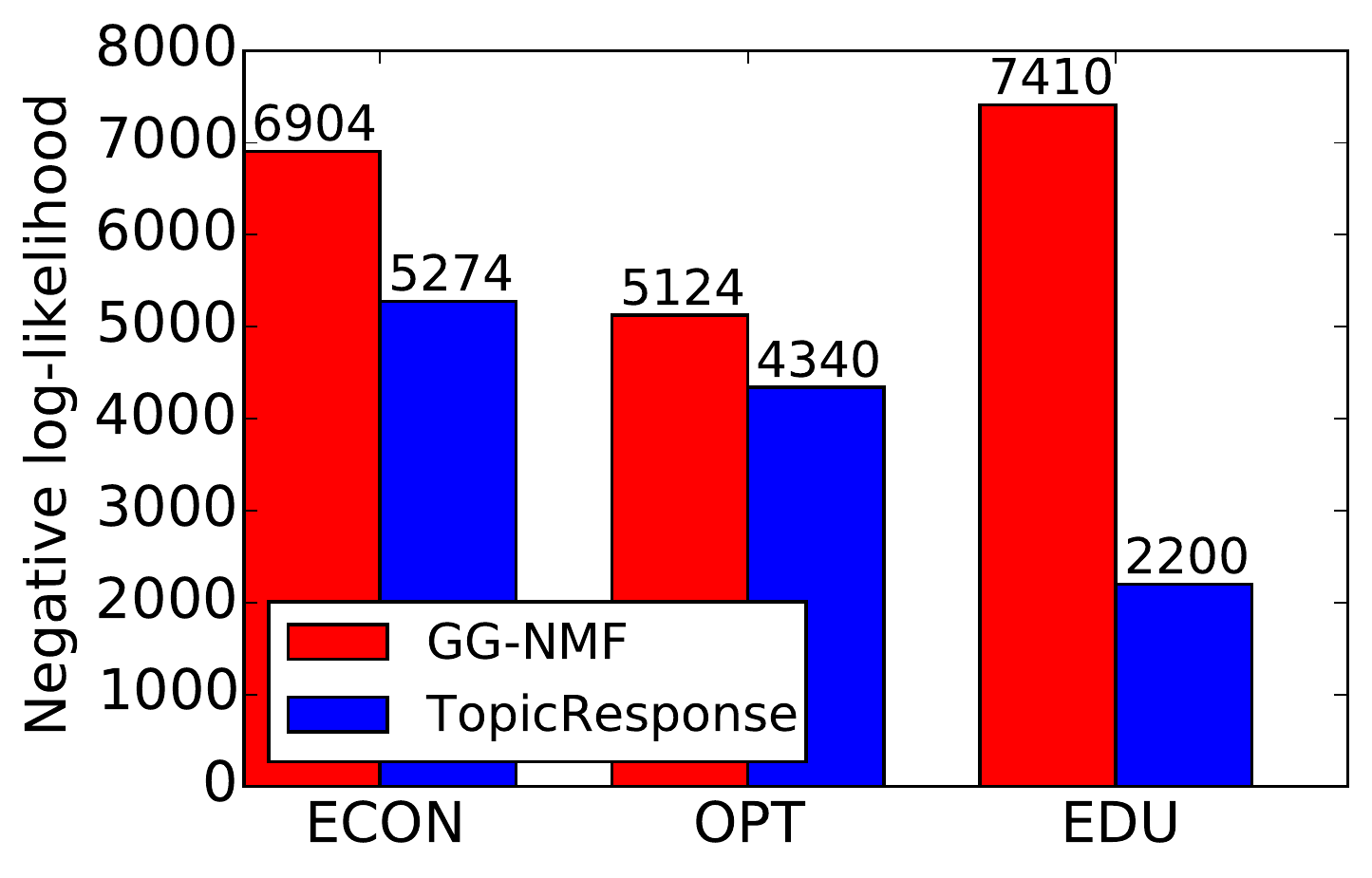}
	\captionof{figure}{Negative log-likelihood as goodness of fit; Smaller is better. \label{fig:c12likelihood}}
    \end{minipage}
\end{table}

\subsection{Hyperparameter Settings}
Table~\ref{tab:paras} presents the parameter values used for our parameter sensitivity experiments, where the default values shown in boldface are used in experiments unless noted otherwise.

\subsection{Main Results for GG-NMF and TopicResponse}
In the first group of experiments, we examine the performance of \GGNMF (baseline) and \TopicResponse in terms of negative log-likelihood, the quality of factorisation $\vbf{WH}$ in approximation $\vbf{V}$ given by $\norm{\vbf{V}-\vbf{WH}}$, and the supervision soft constraint $\norm{\vbf{1}_r\vbf{H}-\vbf{H}_{ideal}}$. 
For \GGNMF, the factorisation and Rasch estimation are separated, where topic-student response matrix $\vbf{H}$ is first obtained using \GGNMF, and then is taken as input to Rasch estimation. For \TopicResponse, the negative log-likelihood is optimised together with factorisation. 
The parameters are set using the boldface default values in Table~\ref{tab:paras}. Figure~\ref{fig:c12likelihood} displays the negative log-likelihood of \GGNMF and \TopicResponse.
 
It can be seen that \TopicResponse can yield superior negative log-likelihood, implying better fit between the topic-student response matrix $\vbf{H}$ and the Rasch model. \TopicResponse therefore provides greater confidence that other item-level fit statistics such as infit, will be acceptable. \emph{Jointly optimising the matrix factorisation and Rasch estimation can bring us closer to global optima.}

\begin{figure}[t]
    \begin{subfigure}[t]{0.5\textwidth}
        \includegraphics[scale=0.38, left]{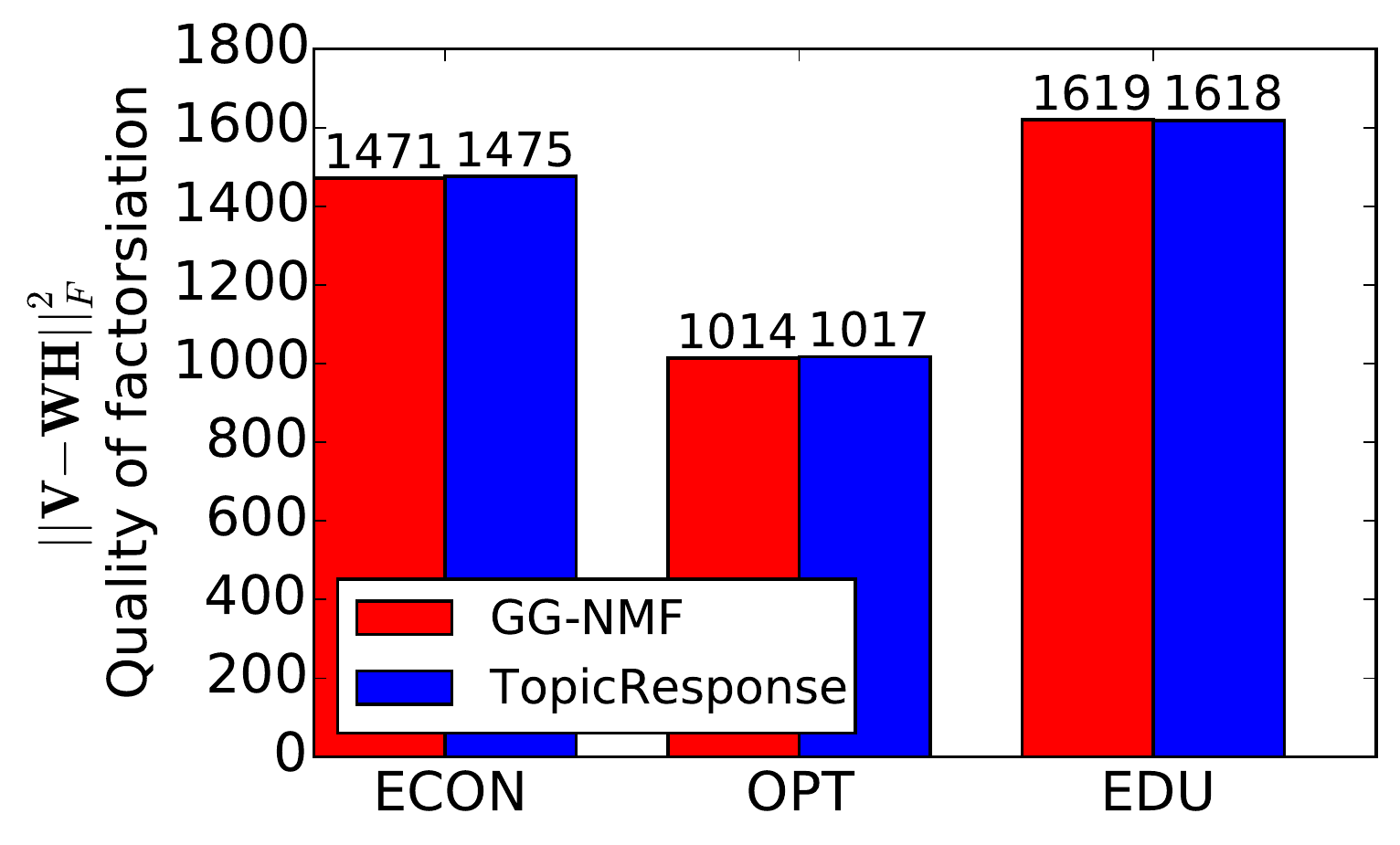}
        \caption{Quality of factorisation, $\norm{\vbf{V}-\vbf{WH}}$}
    \end{subfigure}
    ~
    \begin{subfigure}[t]{0.5\textwidth}
        \includegraphics[scale=0.38, right]{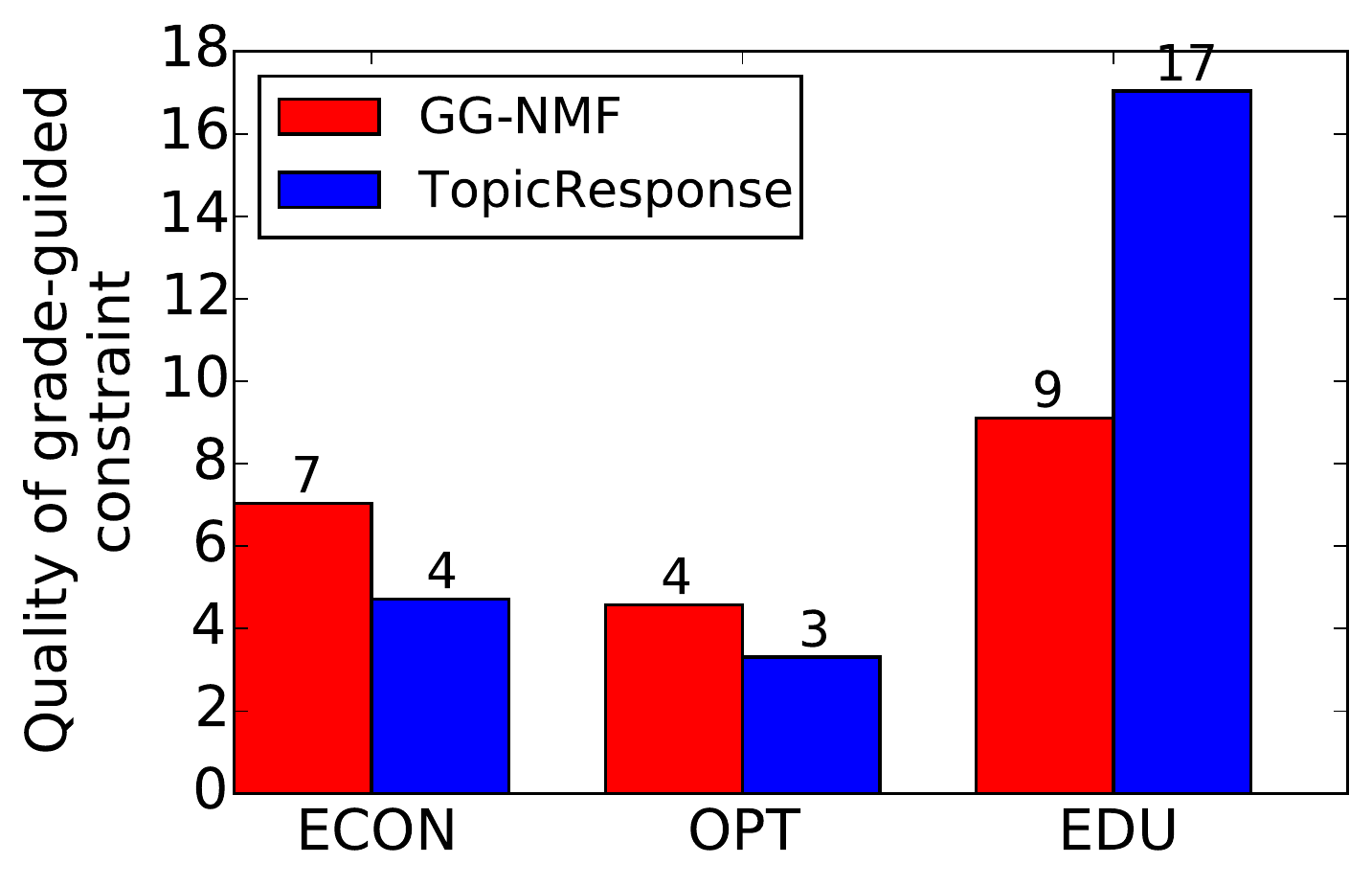}
        \caption{Quality of graded-guided constraint, $\norm{\vbf{1}_r\vbf{H}-\vbf{H}_{ideal}}$}
    \end{subfigure}
    \caption{Performance of \GGNMF and \TopicResponse in terms of $\norm{\vbf{V-WH}}$ and  $\norm{\vbf{1}_r\vbf{H}-\vbf{H}_{ideal}}$; Smaller is better.}
    \label{fig:c12}
\end{figure}

We present the results on quality of approximation $\norm{\vbf{V-WH}}$ and supervision term $\norm{\vbf{1}_r\vbf{H}-\vbf{H}_{ideal}}$, in Figure~\ref{fig:c12}. From these plots, we can see that without sacrificing approximation performance in terms of $\norm{\vbf{V-WH}}$, \TopicResponse obtains superior $\norm{\vbf{1}_r\vbf{H}-\vbf{H}_{ideal}}$ (while obtaining excellent negative log-likelihoods as above). This performance again demonstrates that optimising the factorisation and Rasch estimation globally can be superior to optimising them separately. 
We therefore conclude that \TopicResponse is preferable to \GGNMF; we focus on results for \TopicResponse in the remainder of our experiments.

\subsection{Item Infit, Item Difficulty and Student Ability}
We further examine the infit of each item, which indicates if the set of topics conform to the Rasch model, and is appropriate for measurement.
As illustrated in Section~\ref{sec:infit}, a conventional acceptable range of infit is 0.7 to 1.3. 
As an example, we show the item infit in Figure~\ref{fig:infit} on OPT MOOC.
We can see that the infit of each item is in the acceptable range, with most very close to the (ideal) expected value of 1.0, indicating that the set of topics conform to the Rasch model and is appropriate for measuring student ability. 

\begin{figure}[t]
	\begin{minipage}[t!]{0.47\textwidth}
        \centering
        \includegraphics[scale=0.38]{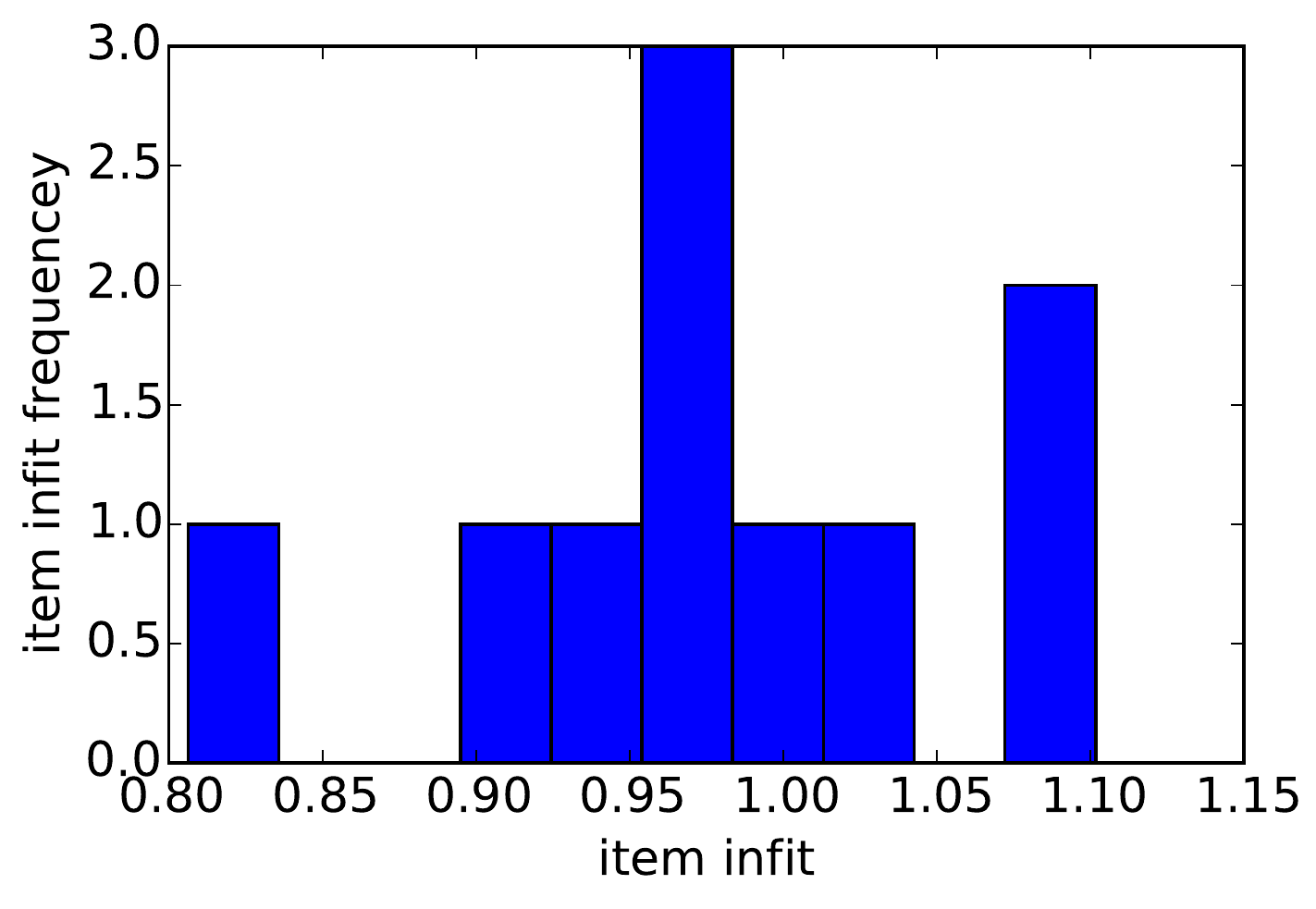}
        \captionsetup{justification=raggedright,
            singlelinecheck=false}
        \captionof{figure}{Item infit histogram for OPT MOOC; infit closer to 1 is better.}
        \label{fig:infit}
\end{minipage}\hfill
\begin{minipage}[t!]{0.51\textwidth}
    \centering
    \includegraphics[scale=0.63]{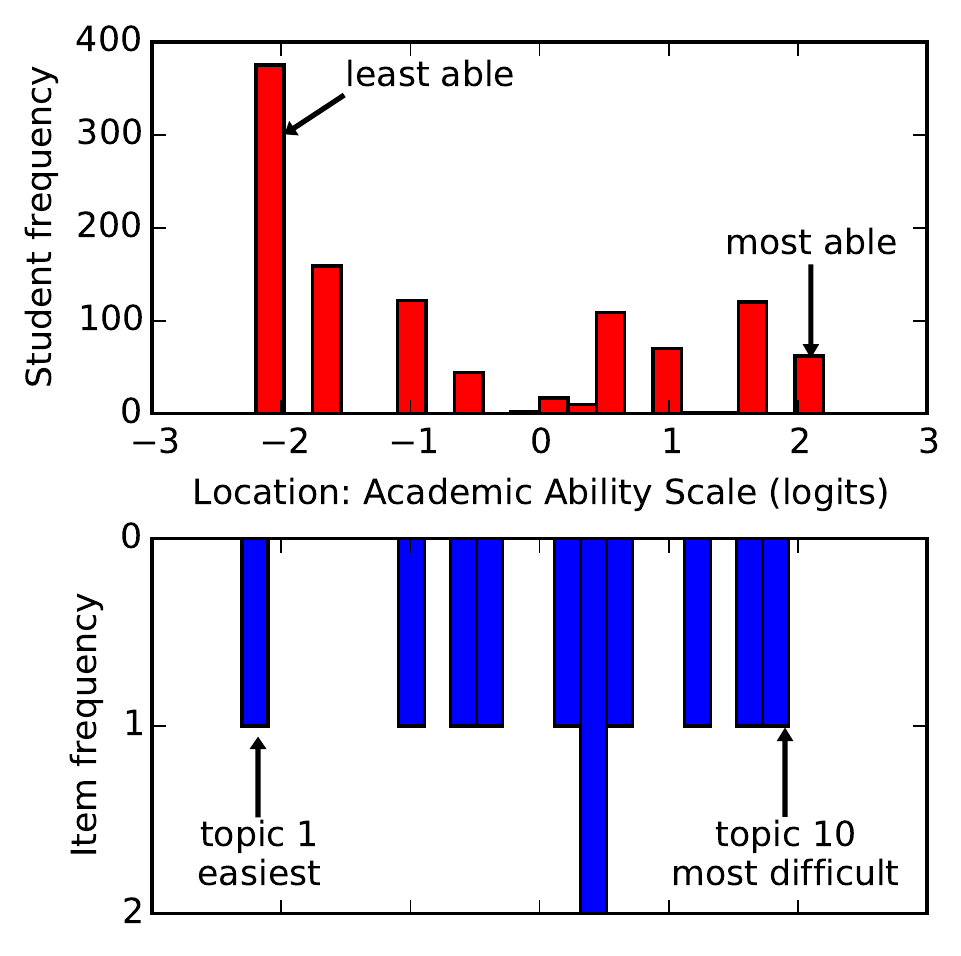}
    \caption{Histograms of OPT MOOC student ability location (top) and item difficulty location (bottom).}
    \label{fig:perosnitemmap}
    \end{minipage}
\end{figure}

Additionally, we examine item difficulties and student abilities. Figure~\ref{fig:perosnitemmap} displays the histogram of item difficulty and student ability along a common scale.
According to the Rasch model, the higher a person's ability relative to the difficulty of a topic, the higher the probability that person posts on that topic.
It can be seen that most students with low ability (around -2 logits), only dominate the ``easiest'' topic (topic 1 with difficult -2.3 logits); this topic concerns general problem solving. In other words, these students are likely to post only on topic 1, and unlikely to post on other topics. By comparison, the most able students with abilities around 2, with high probability contribute to all the topics.

\subsection{Topic Interpretation and Discussion}
We qualitatively examine topic interpretation, in order to assess educational meaningfulness. Well-scaled topics can potentially be used for curriculum refinement. 
Table~\ref{tab:rasch2topic1} presents the topics generated using \TopicResponse, alongside inferred difficulties. 
Topics are interpreted by an instructor who teaches a similar course. 
As the topics are not all course content-related, we envision that instructors examine discovered topics prior to using all for refining curriculum or taking other actions. 
Additionally, the inferred student ability levels and topic difficulty levels could be potentially used for personalised feedback, by tailoring appropriate topics of course content or forum discussion to students with their individual ability level taken into account.
For example, most students (lowest ability) only discuss solving problem in general, as shown in Figure~\ref{fig:perosnitemmap}. If they cannot obtain sufficient help from forum discussions, they may be prone to drop out without further topic exploration. Therefore, in intervening with at-risk students, it is advisable to leverage discovered topics to better focus measures. Such services may be useful in preventing dropout in early stages (when most dropouts typically occur).

\begin{table*}[ht!]
    \small
    \centering
    \caption{Topics and difficulty levels, by \TopicResponse on OPT MOOC.}
    \begin{tabular}{cp{5cm}p{2.5cm}c} 
        \toprule
        No. & Topics   & Interpretation & Inferred difficulty \\
        \toprule
        1 & use time problem get solut one optim algorithm tri work   & Solving in general & -2.30\\
        2 & cours thank would lectur realli great assign good like think &  Course feedback & -0.93 \\  
        3 & python use run program solver java matlab instal command work &   Python/Java/Matlab (How to start) & -0.63\\ 
        4 & problem thank solut get grade knapsack got feedback optim solv &   How knapscak problem is solved and graded& -0.44 \\  
        5 & memori dp use column bb implement solv algorithm bound tabl & Comparing algorithms memory/time &  0.23\\   
        6 & color node graph random edg greedi opt search swap iter & Graph coloring & 0.31\\
        7 & item valu weight capac estim take solut calcul best list &  Knapsack problem  & 0.33\\
        8 & file pi line solver data submit lib urllib2 solveit open &  Using solvers &  0.52\\
        9 & video http class load lecture org problem coursera optimization 001 &   Platform & 1.17 \\
        10 & submit assign assignment error messag view assignment\_id detail class coursera   & Assignment submission & 1.73\\
        \bottomrule
    \end{tabular}%
    \label{tab:rasch2topic1}%
\end{table*}%

\subsection{Parameter Sensitivity}\label{sec:para}
To validate the robustness of \TopicResponse to parameter settings, a series of sensitivity experiments were conducted. The parameter settings are shown in Table~\ref{tab:paras}. 
Negative log-likelihoods, $\norm{\vbf{1}_r\vbf{H}-\vbf{H}_{ideal}}$, $\norm{\vbf{V-WH}}$ and $\norm{\vbf{H}\circ\vbf{H}-\vbf{H}}$ are examined in these experiments. 
Due to space limitations, we report here results for $\lambda_0$ on the OPT MOOC. The reader is referred to Appendix~\ref{app:B} for results on parameters $\lambda_1$,$\lambda_2$,$\lambda_3$, $k$ on all three MOOCs.

\paragraph{Effect of Parameter $\lambda_0$.}
As can be seen in Figure~\ref{fig:lambdaraschb}, as $\lambda_0$ is increased \TopicResponse performs better in terms of negative log-likelihood, and performs worse in terms of the other three metrics due to the regularisation on the Rasch model. By contrast, the performance of \GGNMF does not change as there is no regularisation term on its Rasch estimation. Overall, \TopicResponse performs well when $\lambda_0$ varies between 0.1 and 0.2.

\begin{figure}[!htb]
    \centering
    \begin{subfigure}[t]{0.45\textwidth}
        \centering
        \includegraphics[scale=0.4]{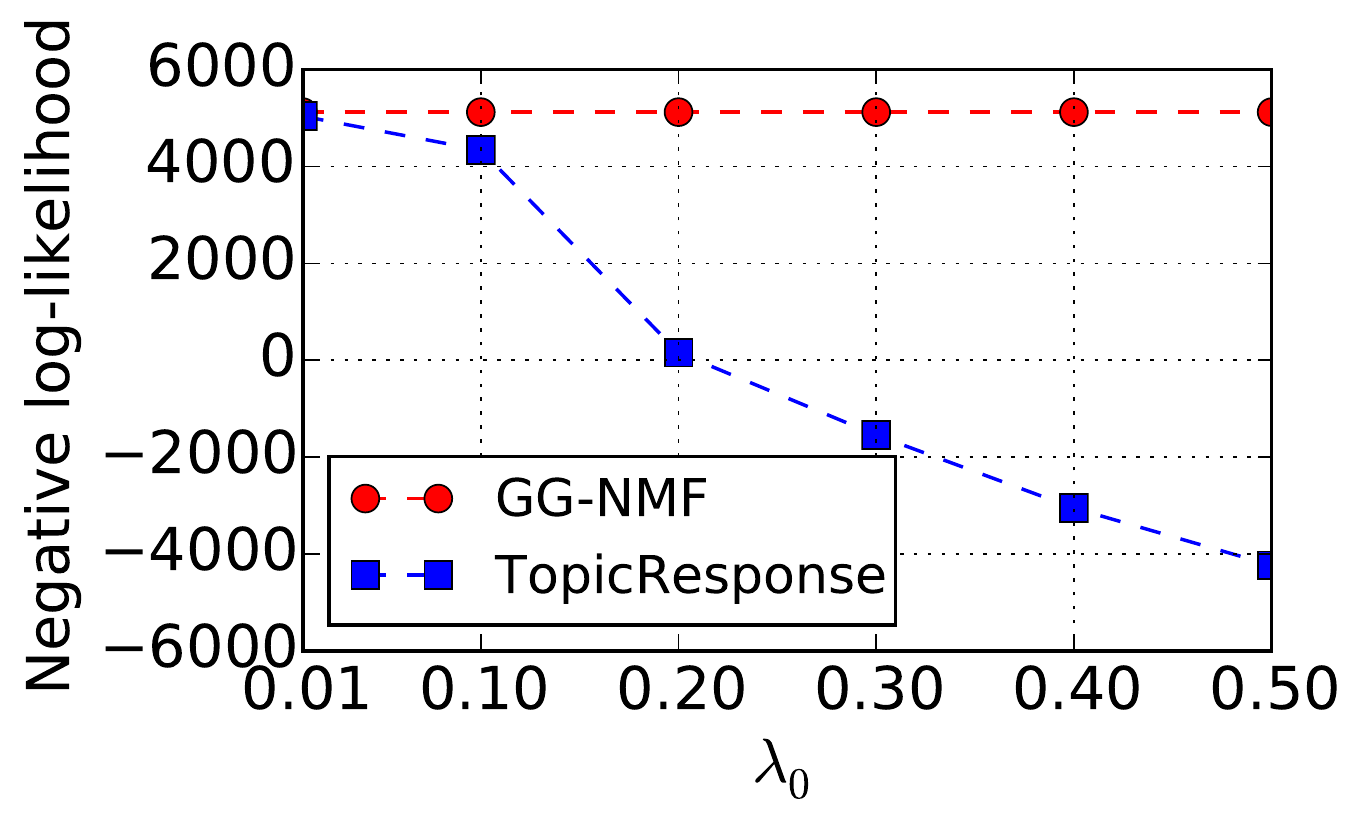}
        \caption{Negative log-ikelihood on OPT}
    \end{subfigure}
    ~
    \begin{subfigure}[t]{0.45\textwidth}
        \centering
        \includegraphics[scale=0.4]{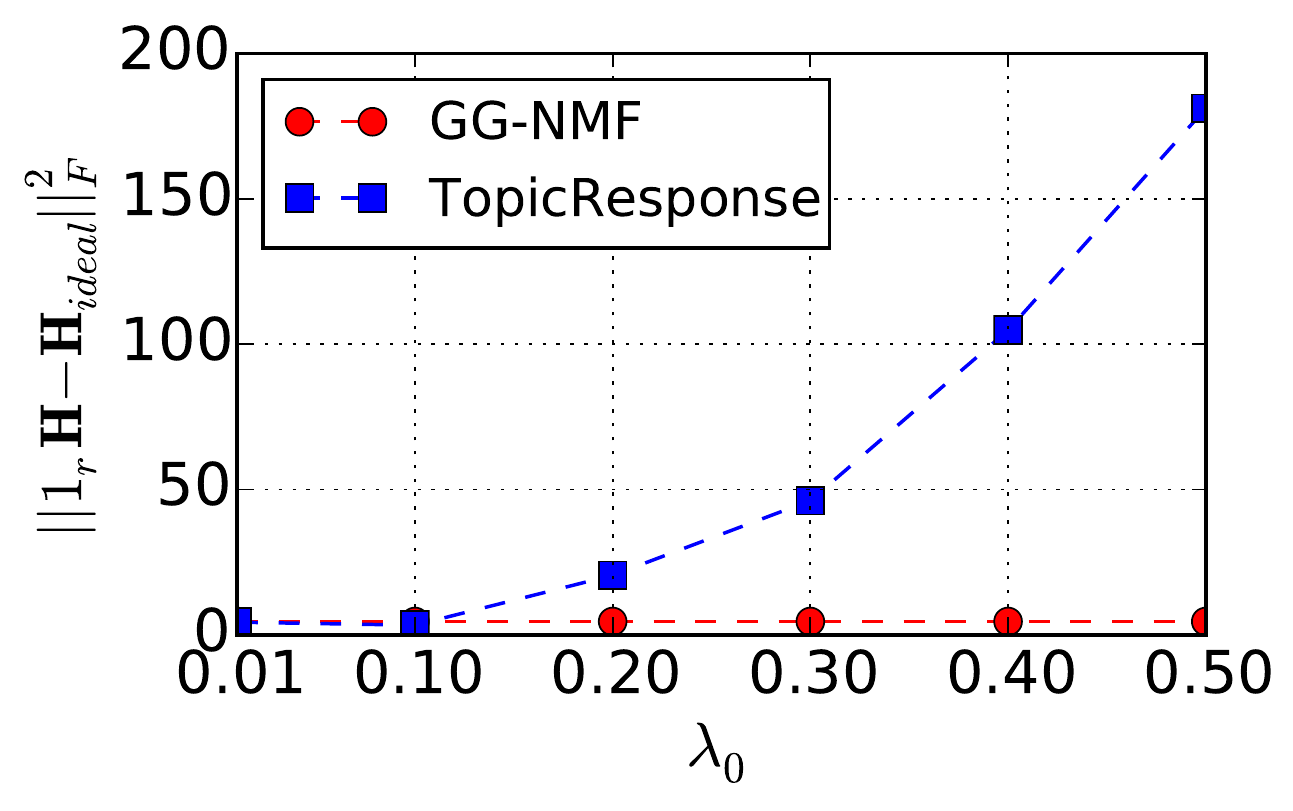}
        \caption{$\norm{\vbf{1}_r\vbf{H}-\vbf{H}_{ideal}}$ on OPT}
    \end{subfigure}
    ~
    \begin{subfigure}[t]{0.45\textwidth}
        \centering
        \includegraphics[scale=0.4]{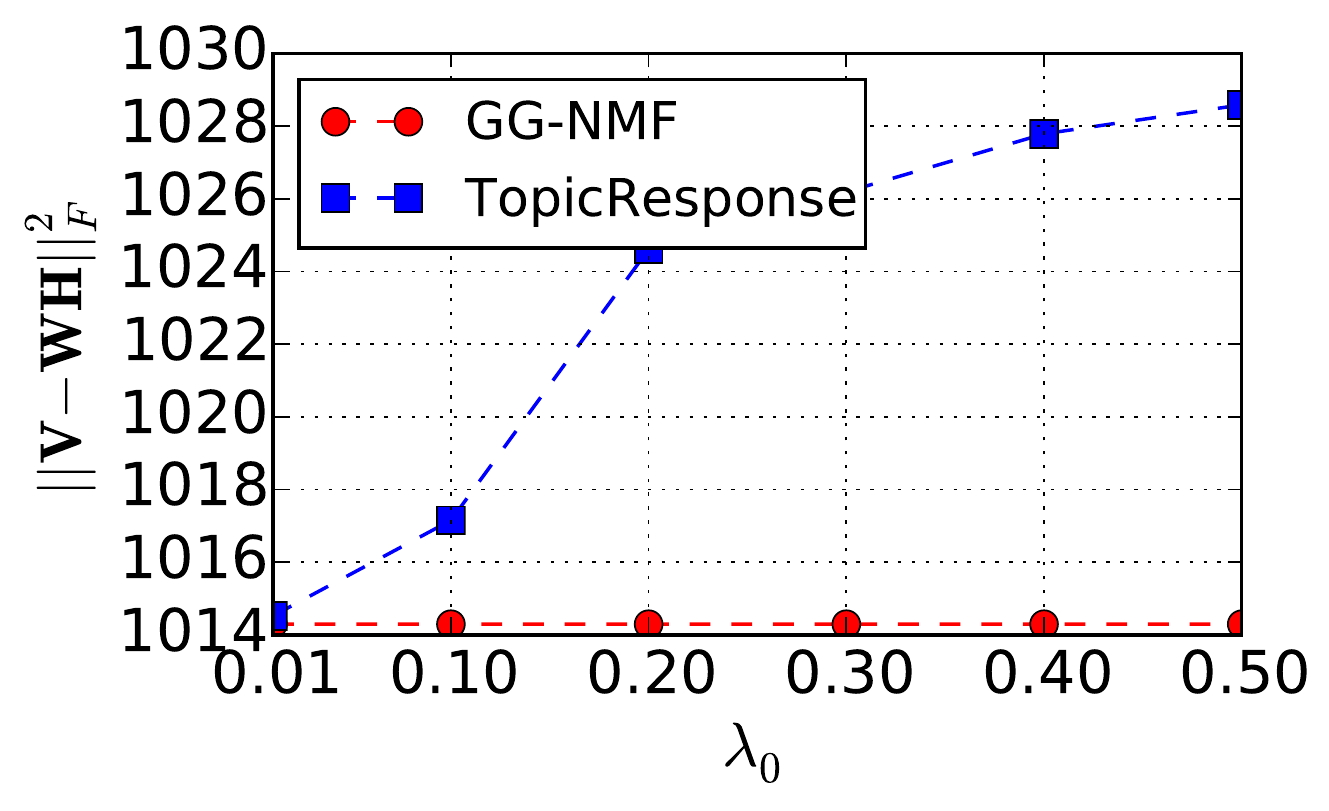}
        \caption{$\norm{\vbf{V}-\vbf{WH}}$ on OPT}
    \end{subfigure}
    ~
    \begin{subfigure}[t]{0.45\textwidth}
        \centering
        \includegraphics[scale=0.4]{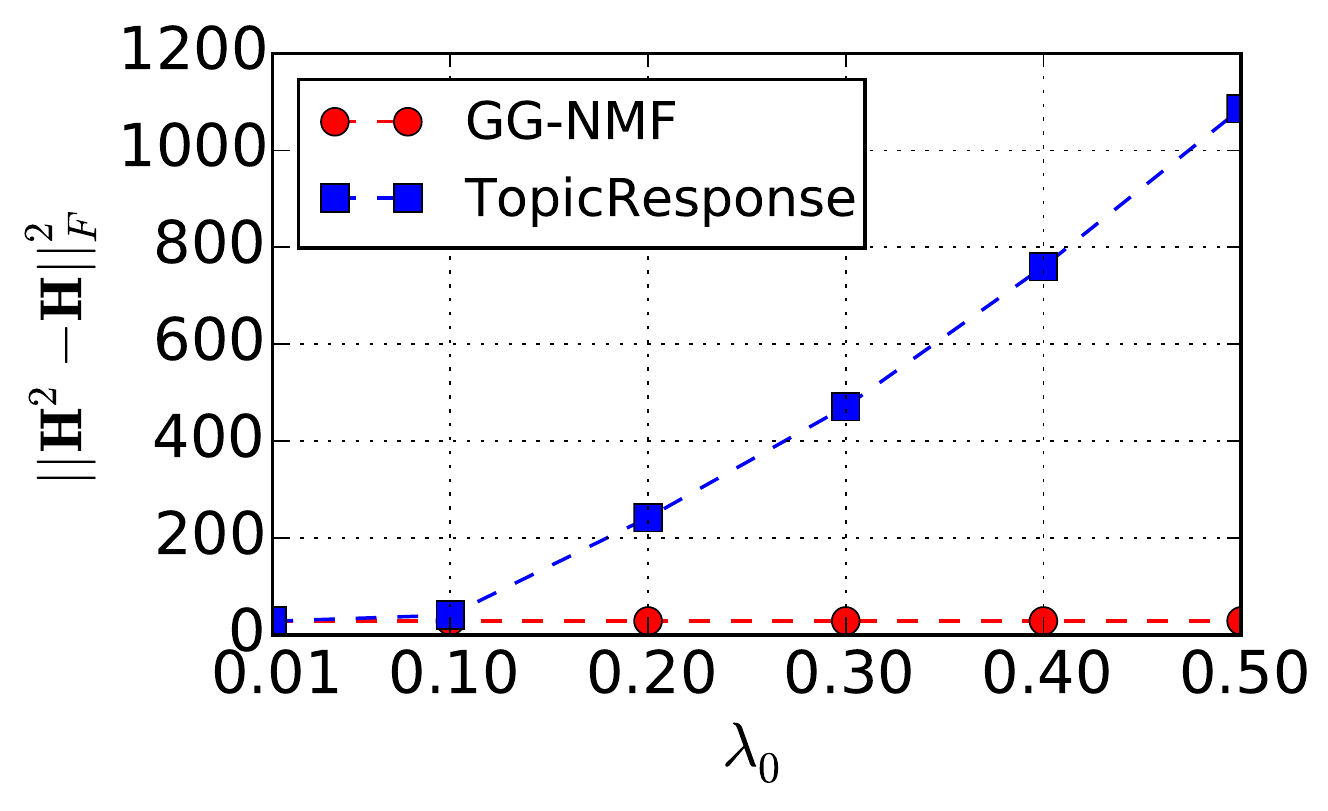}
        \caption{$\norm{\vbf{H}\circ\vbf{H}-\vbf{H}}$ on OPT}
    \end{subfigure}
    \caption{Performance of \GGNMF and \TopicResponse on OPT with varying $\lambda_0$.}
    \label{fig:lambdaraschb}
\end{figure}

\section{Conclusion and Future Work}
We have examined the suitability of content-based items (topics) discovered from MOOC forum discussions, for modelling student abilities. 
Our central tenet is that topics can be regarded as useful items for measuring latent skills, if student responses to these topics fit the Rasch item-response theory model, and if the discovered topics are further interpretable to domain experts. 
We propose to jointly optimise NMF and Rasch modelling, in order to discover Rasch-scaled topics.
We provide a quantitative validation on three Coursera MOOCs, demonstrating that \TopicResponse yields better global fit to the Rasch model (observed with lower negative log-likelihood), maintains good quality of factorisation approximation, while measuring the students' academic abilities (reflected by the grade-guided constraint on students' participation on topics). We also provide qualitative examination of topic interpretation with inferred difficulty levels on a Discrete Optimisation MOOC. The results on goodness of fit and our qualitative examination, together suggest potential applications in curriculum refinement, student assessment and personalised feedback. 

We opted to study the relatively simple Rasch model, as it forms the basis of very many subsequent models in the literature. One direction for extension, is that for any model (like Rasch), that fits parameters via maximum-likelihood estimation (or risk minimisation in general), the model can be augmented with NMF as an additional regularisation. For example, such an extension should be straightforward for polychotomous observations, hierarchical models on latent skills, models that include more flexible per-student variation, etc. These represent fruitful direction for future research. Another possible extension could involve augmenting the $\vbf{W},\vbf{H}$ matrices in the NMF or Rasch objective terms with manually-crafted items, to make effective use of prior knowledge.

\label{sec:con}

\begin{acknowledgements}
We thank Jeffrey Chan for discussions related to this work, and the anonymous reviewers and editor
for their thoughtful feedback. 
This work is supported by Data61, and the Australian Research Council (DE160100584). 
\end{acknowledgements}

\bibliographystyle{spbasic}      
\bibliography{mandb}   

\begin{thebibliography}{28}
\providecommand{\natexlab}[1]{#1}
\providecommand{\url}[1]{{#1}}
\providecommand{\urlprefix}{URL }
\expandafter\ifx\csname urlstyle\endcsname\relax
  \providecommand{\doi}[1]{DOI~\discretionary{}{}{}#1}\else
  \providecommand{\doi}{DOI~\discretionary{}{}{}\begingroup
  \urlstyle{rm}\Url}\fi
\providecommand{\eprint}[2][]{\url{#2}}

\bibitem[{Bachrach et~al(2012)Bachrach, Graepel, Minka, and
  Guiver}]{bachrach2012grade}
Bachrach Y, Graepel T, Minka T, Guiver J (2012) How to grade a test without
  knowing the answers---a {B}ayesian graphical model for adaptive crowdsourcing
  and aptitude testing. In: Proceedings of the 29th International Conference on
  Machine Learning (ICML-12), pp 1183--1190

\bibitem[{Baker and Kim(2004)}]{baker2004item}
Baker FB, Kim SH (2004) Item response theory: Parameter estimation techniques.
  CRC Press

\bibitem[{Bergner et~al(2012)Bergner, Droschler, Kortemeyer, Rayyan, Seaton,
  and Pritchard}]{bergner2012model}
Bergner Y, Droschler S, Kortemeyer G, Rayyan S, Seaton D, Pritchard DE (2012)
  Model-based collaborative filtering analysis of student response data:
  Machine-learning item response theory. International Educational Data Mining
  Society

\bibitem[{Bertsekas(1999)}]{bertsekas1999nonlinear}
Bertsekas DP (1999) Nonlinear programming. Athena Scientific

\bibitem[{Bond and Fox(2001)}]{bond2001applying}
Bond TG, Fox CM (2001) Applying the Rasch model: Fundamental measurement in the
  human sciences. Lawrence Erlbaum Associates Publishers

\bibitem[{Champaign et~al(2014)Champaign, Colvin, Liu, Fredericks, Seaton, and
  Pritchard}]{champaign2014correlating}
Champaign J, Colvin KF, Liu A, Fredericks C, Seaton D, Pritchard DE (2014)
  Correlating skill and improvement in 2 {MOOC}s with a student's time on
  tasks. In: Proceedings of the First ACM Conference on Learning@Scale
  Conference, ACM, pp 11--20

\bibitem[{Chaturvedi et~al(2014)Chaturvedi, Goldwasser, and
  Daum{\'e}~III}]{chaturvedi2014predicting}
Chaturvedi S, Goldwasser D, Daum{\'e}~III H (2014) Predicting instructor's
  intervention in {MOOC} forums. In: ACL (1), pp 1501--1511

\bibitem[{Chen et~al(2005)Chen, Lee, and Chen}]{chen2005personalized}
Chen CM, Lee HM, Chen YH (2005) Personalized e-learning system using item
  response theory. Computers \& Education 44(3):237--255

\bibitem[{Colvin et~al(2014)Colvin, Champaign, Liu, Fredericks, and
  Pritchard}]{colvin2014comparing}
Colvin KF, Champaign J, Liu A, Fredericks C, Pritchard DE (2014) Comparing
  learning in a {MOOC} and a blended on-campus course. In: Educational Data
  Mining 2014

\bibitem[{Gillani et~al(2014)Gillani, Eynon, Osborne, Hjorth, and
  Roberts}]{gillani2014communication}
Gillani N, Eynon R, Osborne M, Hjorth I, Roberts S (2014) Communication
  communities in {MOOC}s. arXiv preprint arXiv:14034640

\bibitem[{Guttman(1950)}]{guttman1950basis}
Guttman L (1950) The basis for scalogram analysis. In: Stouffer S (ed)
  Measurement and Prediction: The American Soldier, Wiley, New York

\bibitem[{He et~al(2016)He, Rubinstein, Bailey, Zhang, Milligan, and
  Chan}]{he2016moocs}
He J, Rubinstein BI, Bailey J, Zhang R, Milligan S, Chan J (2016) {MOOC}s meet
  measurement theory: A topic-modelling approach. In: Thirtieth AAAI Conference
  on Artificial Intelligence

\bibitem[{Jenders et~al(2016)Jenders, Krestel, and Naumann}]{jendersanswer}
Jenders M, Krestel R, Naumann F (2016) Which answer is best? {P}redicting
  accepted answers in {MOOC} forums. WWW'2016 Companion

\bibitem[{Lee and Seung(1999)}]{lee1999learning}
Lee DD, Seung HS (1999) Learning the parts of objects by non-negative matrix
  factorization. Nature 401(6755):788--791

\bibitem[{Lee and Seung(2001)}]{lee2001algorithms}
Lee DD, Seung HS (2001) Algorithms for non-negative matrix factorization. In:
  Advances in Neural Information Processing Systems, pp 556--562

\bibitem[{Linacre(2002)}]{linacre2002infit}
Linacre JM (2002) What do infit and outfit, mean-square and standardized mean.
  Rasch Measurement Transactions 16(2):878

\bibitem[{Linacre(2006)}]{linacre2006misfit}
Linacre JM (2006) Misfit diagnosis: Infit outfit mean-square standardized.
  Retrieved June 1:2006

\bibitem[{Milligan(2015)}]{milligan2015crowd}
Milligan S (2015) Crowd-sourced learning in {MOOC}s: {L}earning analytics meets
  measurement theory. In: Proceedings of the Fifth International Conference on
  Learning Analytics And Knowledge, ACM, pp 151--155

\bibitem[{Ramesh et~al(2015)Ramesh, Kumar, Foulds, and
  Getoor}]{ramesh2015weakly}
Ramesh A, Kumar SH, Foulds J, Getoor L (2015) Weakly supervised models of
  aspect-sentiment for online course discussion forums. In: Annual Meeting of
  the Association for Computational Linguistics (ACL)

\bibitem[{Rasch(1993)}]{rasch1993probabilistic}
Rasch G (1993) Probabilistic models for some intelligence and attainment tests.
  ERIC

\bibitem[{Scholten(2011)}]{scholten2011admissible}
Scholten AZ (2011) Admissible statistics from a latent variable perspective.
  Theory \& Psychology 18:111--117

\bibitem[{Wen et~al(2014)Wen, Yang, and Rose}]{wen2014sentiment}
Wen M, Yang D, Rose C (2014) Sentiment analysis in {MOOC} discussion forums:
  What does it tell us? In: Educational Data Mining 2014

\bibitem[{Wright and Masters(1982)}]{wright1982rating}
Wright BD, Masters GN (1982) Rating Scale Analysis. Rasch Measurement. ERIC

\bibitem[{Wright et~al(1994)Wright, Linacre, Gustafson, and
  Martin-Lof}]{wright1994reasonable}
Wright BD, Linacre JM, Gustafson J, Martin-Lof P (1994) Reasonable mean-square
  fit values. Rasch measurement transactions 8(3):370

\bibitem[{Yang et~al(2014)Yang, Adamson, and Ros{\'e}}]{yang2014question}
Yang D, Adamson D, Ros{\'e} CP (2014) Question recommendation with constraints
  for massive open online courses. In: Proceedings of the 8th ACM Conference on
  Recommender systems, ACM, pp 49--56

\bibitem[{Yang et~al(2015)Yang, Wen, Howley, Kraut, and
  Rose}]{yang2015exploring}
Yang D, Wen M, Howley I, Kraut R, Rose C (2015) Exploring the effect of
  confusion in discussion forums of massive open online courses. In:
  Proceedings of the Second (2015) ACM Conference on Learning@ Scale, ACM, pp
  121--130

\bibitem[{Zhang et~al(2007)Zhang, Ding, Li, and Zhang}]{zhang2007binary}
Zhang Z, Ding C, Li T, Zhang X (2007) Binary matrix factorization with
  applications. In: Data Mining, 2007. ICDM 2007. Seventh IEEE International
  Conference on, IEEE, pp 391--400

\bibitem[{Zhang et~al(2010)Zhang, Li, Ding, Ren, and Zhang}]{zhang2010binary}
Zhang ZY, Li T, Ding C, Ren XW, Zhang XS (2010) Binary matrix factorization for
  analyzing gene expression data. Data Mining and Knowledge Discovery
  20(1):28--52

\end{thebibliography}

\appendix

\section{Proof of Theorem~\ref{thm:thm2}}

The update rules for $\beta_{i}$ and $\theta_j$ are derived using the Newton-Raphson method~\citep{bertsekas1999nonlinear}, where the convergence to a local optimum is guaranteed. Here, we focus on the proof for the update rule for $h_{ij}$. The update rule for $w_{ij}$ can be proved similarly.
We closely follow the procedure described in \citep{lee2001algorithms}, where an auxiliary function similar to that used in the Expectation-Maximization (EM) algorithm is used for proof.
\begin{defn}[\citealt{lee2001algorithms}]
     $G(h,h^\prime)$ is an auxiliary function for $F(h)$ if the conditions
    \begin{eqnarray*}
        G(h,h^\prime)\geq F(h)\enspace, \qquad G(h,h)=F(h)\enspace,
    \end{eqnarray*}
    are satisfied.
\end{defn}
\begin{lemma}[\citealt{lee2001algorithms}]
	If $G$ is an auxiliary function, then $F$ is non-increasing under the update
    \begin{eqnarray}\label{equ:argmin}
        h_{t+1}=\argmin_hG(h,h^t) \enspace.
    \end{eqnarray}
\end{lemma}
\begin{proof}
The result follows from noting $F(h^{t+1})\leq G(h^{t+1}, h^t)\leq G(h^t, h^t)=F(h^t)$.\qed
\end{proof}
For any element $h_{ij}$ in $\vbf{H}$, let $F_{h_{ij}}$ denote the part of $f(\vbf{W},\vbf{H}, \vbf{\theta}, \vbf{\beta})$ in Eq.~(\ref{equ:fun2}) relevant to $h_{ij}$. Since the update is essentially element-wise, it is sufficient to show that each $F_{h_{ij}}$ is non-increasing under the update rule of Eq.~(\ref{equ:h1}). To prove this, we define the auxiliary function regarding $h_{ij}$ as follows.

\begin{lemma}
    Function
    \begin{eqnarray}\label{equ:auxiliary1}
    \begin{split}
    G(h_{ij}, h_{ij}^t)= F_{h_{ij}}(h_{ij}^t)+F_{h_{ij}}^{\prime}(h_{ij}^t)(h_{ij}-h_{ij}^t)
    +\varphi_{ij} (h_{ij}-h_{ij}^t)^2\enspace,
    \end{split}
    \end{eqnarray}
    where
    \begin{eqnarray*}
        \varphi_{ij}=\frac{2(\vbf{W}^T\vbf{WH})_{ij}+2\lambda_1(\vbf{1}_r^T\vbf{1}_r\vbf{H})_{ij}+12\lambda_2(h_{ij}^t)^3+2\lambda_2h_{ij}^t + \lambda_3 (\vbf{\theta}-\vbf{\beta})_{ij}^-}{2h_{ij}^t}
    \end{eqnarray*} 
    is an auxiliary function for $F_{h_{ij}}$.
\end{lemma}
\begin{proof}
It is obvious that $G(h_{ij},h_{ij})=F_{h_{ij}}$. So we need only prove that $G(h_{ij}, h_{ij}^t)\geq F_{h_{ij}}$. Considering the Taylor series expansion of $F_{h_{ij}}$,
\begin{eqnarray*}
    F_{h_{ij}}= F_{h_{ij}}(h_{ij}^t)+F_{h_{ij}}^{\prime}(h_{ij}^t)(h_{ij}-h_{ij}^t)
    +\frac{1}{2}F_{h_{ij}}^{\prime\prime}(h_{ij}^t)(h_{ij}-h_{ij}^t)^2 \enspace,
\end{eqnarray*}
$G(h_{ij}, h_{ij}^t)\geq F_{h_{ij}}$ is equivalent to 
$\varphi_{ij}\geq \frac{1}{2}F_{h_{ij}}^{\prime\prime}(h_{ij}^t)$, where
\begin{eqnarray*}
	F_{h_{ij}}^{\prime\prime}(h_{ij}^t)=2(\vbf{W}^T\vbf{W})_{ii}+2\lambda_1(\vbf{1}_r^T\vbf{1}_r)_{ii}+12\lambda_2(h_{ij}^t)^2
    -12\lambda_2h_{ij}^t+2\lambda_2\enspace.
\end{eqnarray*}
To prove the above inequality, we have
\begin{eqnarray*}
    \varphi_{ij}h_{ij}^t&=&(\vbf{W}^T\vbf{WH})_{ij}+\lambda_1(\vbf{1}_r^T\vbf{1}_r\vbf{H})_{ij}+6\lambda_2(h_{ij}^t)^3+\lambda_2h_{ij}^t + 0.5\lambda_3 (\vbf{\theta}-\vbf{\beta})_{ij}^-\\
    &=&\sum_{l=1}^{k}(\vbf{W}^T\vbf{W})_{il}h_{lj}^t+
    \lambda_1\sum_{l=1}^{k}(\vbf{1}_r^T\vbf{1}_r)_{il}h_{lj}^t +
    +6\lambda_2(h_{ij}^t)^3+\lambda_2h_{ij}^t + 0.5 \lambda_3 (\vbf{\theta}-\vbf{\beta})_{ij}^-\\
    &\geq& (\vbf{W}^T\vbf{W})_{ii}h_{ij}^t+\lambda_1(\vbf{1}_r^T\vbf{1}_r)_{ii}h_{ij}^t+6\lambda_2(h_{ij}^t)^3+\lambda_2h_{ij}^t-12\lambda_2h_{ij}^t\\	
    &\geq& h_{ij}^t\big((\vbf{W}^T\vbf{W})_{ii}+\lambda_1(\vbf{1}_r^T\vbf{1}_r)_{ii}+6\lambda_2(h_{ij}^t)^2-6\lambda_2h_{ij}^t+\lambda_2\big)\\
    &=&\frac{1}{2}F_{h_{ij}}^{\prime\prime}(h_{ij}^t)h_{ij}^t\enspace.
\end{eqnarray*}
Thus, $G(h_{ij}, h_{ij}^t)\geq F_{h_{ij}}$ as claimed.\qed
\end{proof}

Replacing $G(h_{ij}, h_{ij}^t)$ in Eq.~(\ref{equ:argmin}) by Eq.~(\ref{equ:auxiliary1}) and setting $\frac{\partial G(h_{ij}, h_{ij}^t)}{\partial h_{ij}}$ to be 0 results in the update rule in Eq.~(\ref{equ:h1}). 
Since Eq.~(\ref{equ:auxiliary1}) is an auxiliary function, $F_{h_{ij}}$ is non-increasing under this update rule.

\section{Experimental Results of Parameter Sensitivity on Regularisation Parameters $\lambda_1$,$\lambda_2$,$\lambda_3$, and $k$}\label{app:B}

\textbf{a) Effect of Parameter $\lambda_1$}:
    As we can see from Figure~\ref{fig:lambdaw}, GG-NMF and TopicResponse are not sensitive to $\lambda_1$, performing stably with varying $\lambda_1$. TopicResponse constantly performs better in terms of negative log-likelihood while maintaining the comparable performance in terms of the other three metrics.
    
    \textbf{b) Effect of Parameter $\lambda_2$}:
    It can be seen from Figure~\ref{fig:lambda1hideal} that GG-NMF and TopicResponse perform well in terms of $\norm{\vbf{1}_r\vbf{H}-\vbf{H}_{ideal}}$ (Figure~\ref{fig:lambdahidealhideal1} to Figure~\ref{fig:lambdahidealhideal3}) and $\norm{\vbf{H}\circ\vbf{H}-\vbf{H}}$ (Figure~\ref{fig:lambdahidealhbinary1} to Figure~\ref{fig:lambdahidealhbinary3}) when $\lambda_2$ varies from $10^{0}$ to $10^{2}$, and from $10^{-3}$ to $10^{0}$  respectively. 
    $\norm{\vbf{V-WH}}$ gets worse as $\lambda_1$ increases, but does not change a lot compared to $\norm{\vbf{1}_r\vbf{H}-\vbf{H}_{ideal}}$ and $\norm{\vbf{H}\circ\vbf{H}-\vbf{H}}$.
    As $\lambda_2$ increases, the performance of GG-NMF and TopicResponse in terms of negative log-likelihood decrease, and TopicResponse constantly performs better than GG-NMF. Overall, $\lambda_2$ with values around 1.0 is good for GG-NMF and TopicResponse.

    \textbf{c) Effect of Parameter $\lambda_3$}:
    It can be seen that GG-NMF and TopicResponse perform well in terms of $\norm{\vbf{1}_r\vbf{H}-\vbf{H}_{ideal}}$ (Figure~\ref{fig:lambdahbinaryhideal1} to Figure~\ref{fig:lambdahbinaryhideal3}) and $\norm{\vbf{H}\circ\vbf{H}-\vbf{H}}$ (Figure~\ref{fig:lambdahbinaryhbinary1} to Figure~\ref{fig:lambdahbinaryhbinary3}) when $\lambda_3$ varies from $10^{-1}$ to $10^{0}$, and from $10^{0}$ to $10^{2}$  respectively. 
    Similar to $\lambda_2$, $\lambda_3$ does not affect $\norm{\vbf{V-WH}}$ significantly. TopicResponse constantly achieves better negative log-likelihood than GG-NMF. Overall, $\lambda_3$ with values around 1.0 is good for GG-NMF and TopicResponse.

    \textbf{d) Effect of the number of topics $k$}:
    It can be seen from Figure~\ref{fig:lambdatopic} that TopicResponse constantly outperforms GG-NMF in terms of negative log-likelihood, while getting slightly worse performance in the other three metrics. This is reasonable, as GG-NMF has more constraints and hence
    the model itself is less likely to perform as well as
    the less constrained GG-NMF in other metrics. 
    Overall, GG-NMF and TopicResponse perform well in the experiments when $k$ is set to 10 or 15. We choose 10 as the value of $k$ since a smaller number of topics are easier to analyse.
    
    \begin{figure}[!htb]
        \centering
        \begin{subfigure}[t]{0.32\textwidth}
            \centering
            \includegraphics[scale=0.30]{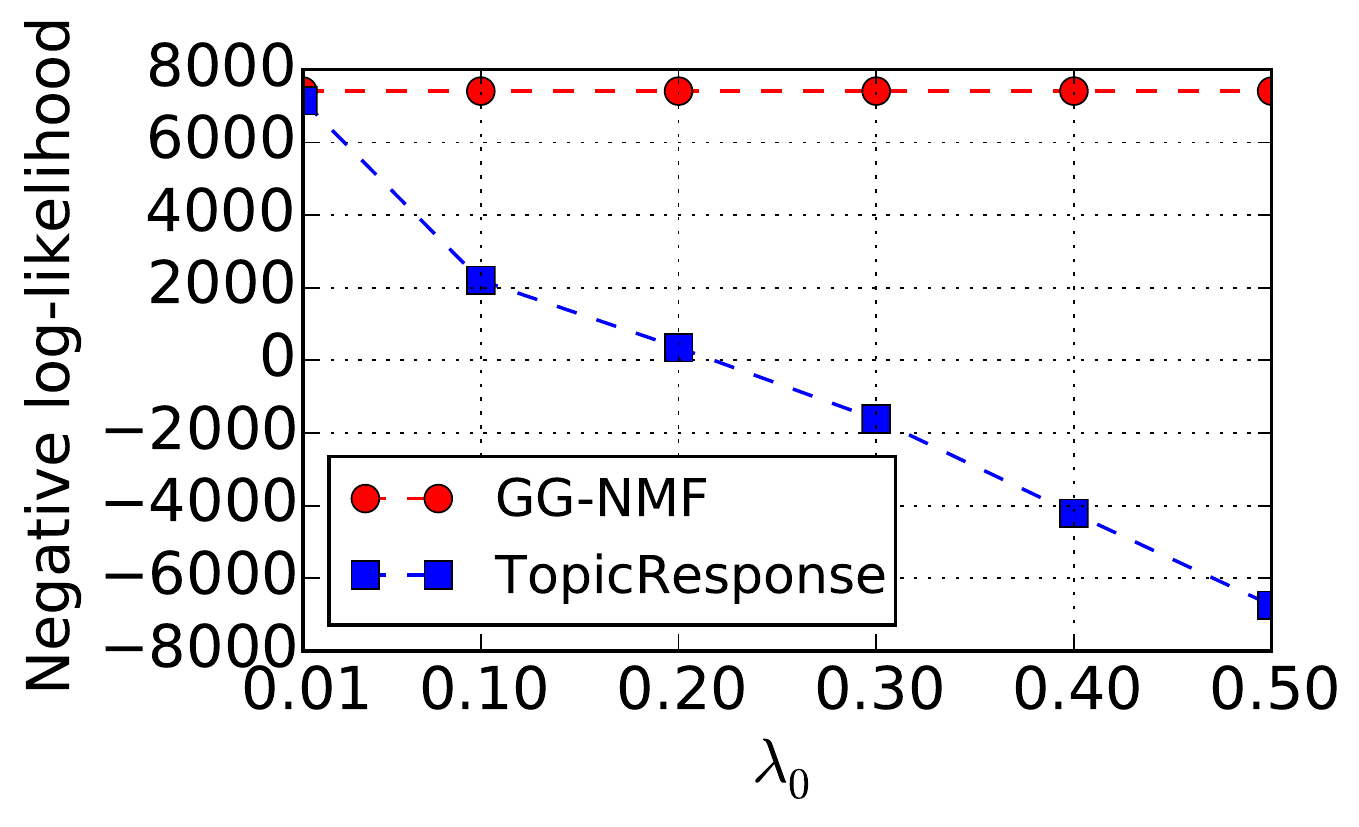}
            \caption{Negative log-ikelihood on EDU}
        \end{subfigure}%
        ~
        \begin{subfigure}[t]{0.32\textwidth}
            \centering
            \includegraphics[scale=0.30]{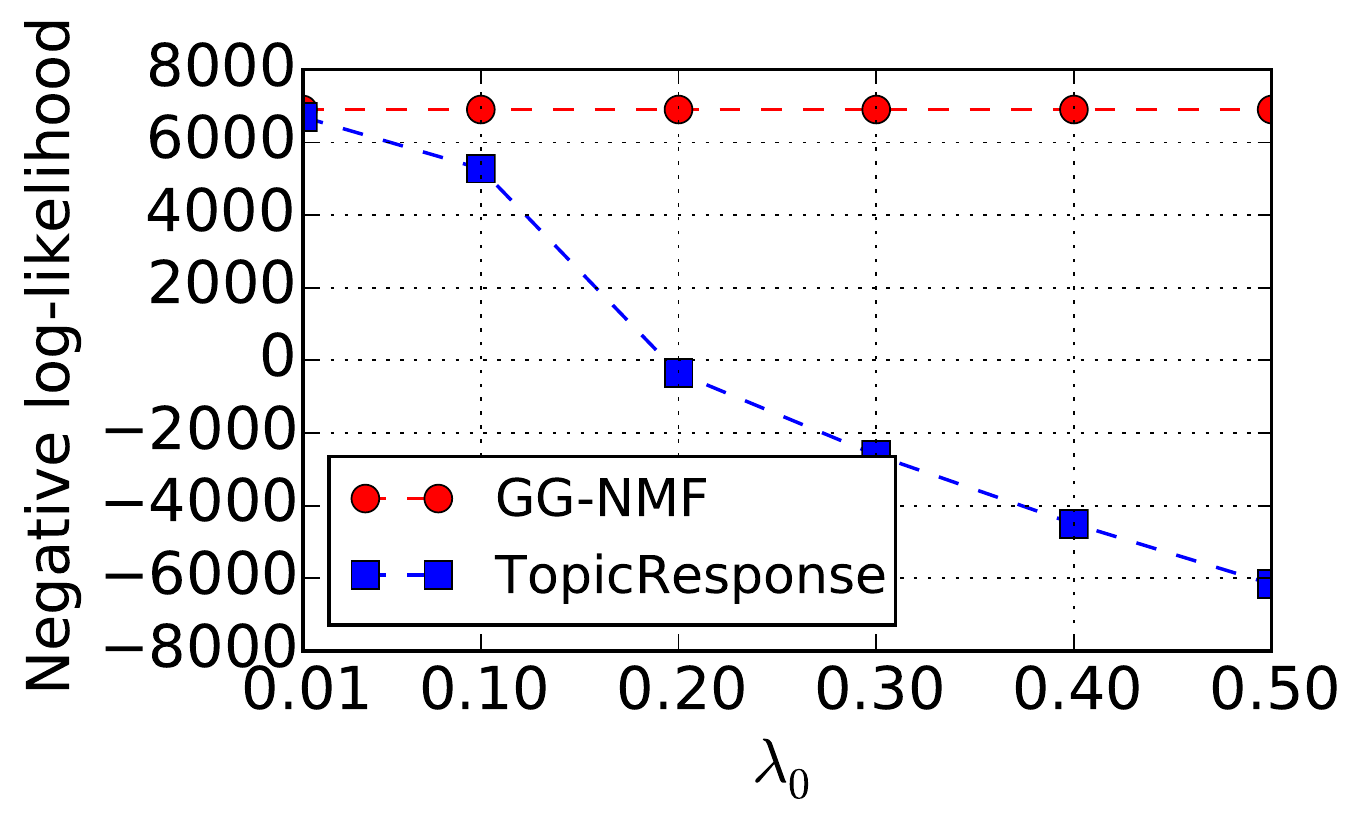}
            \caption{Negative log-ikelihood on ECON}
        \end{subfigure}
        ~
        \begin{subfigure}[t]{0.32\textwidth}
            \centering
            \includegraphics[scale=0.30]{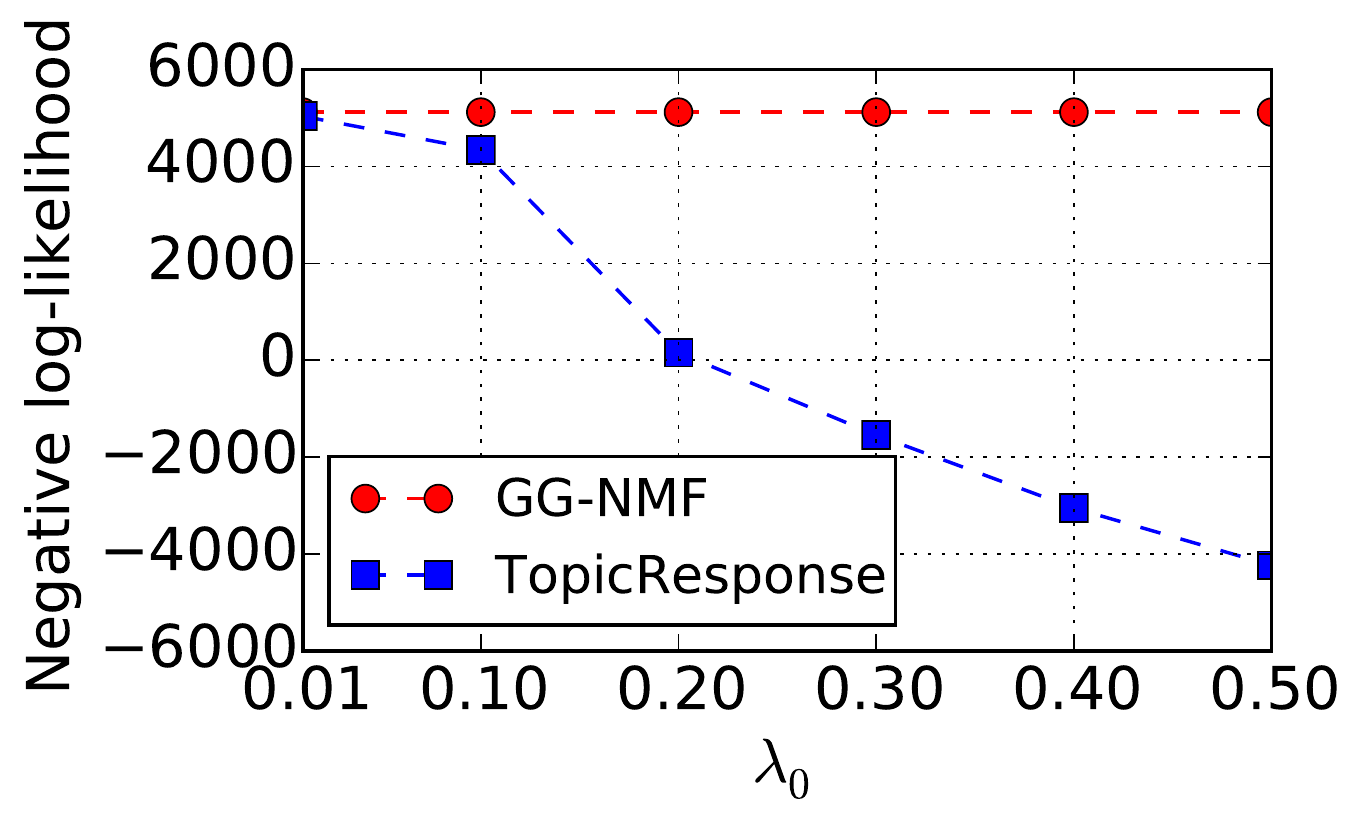}
            \caption{Negative log-ikelihood on OPT}
        \end{subfigure}
        ~
        \begin{subfigure}[t]{0.32\textwidth}
            \centering
            \includegraphics[scale=0.30]{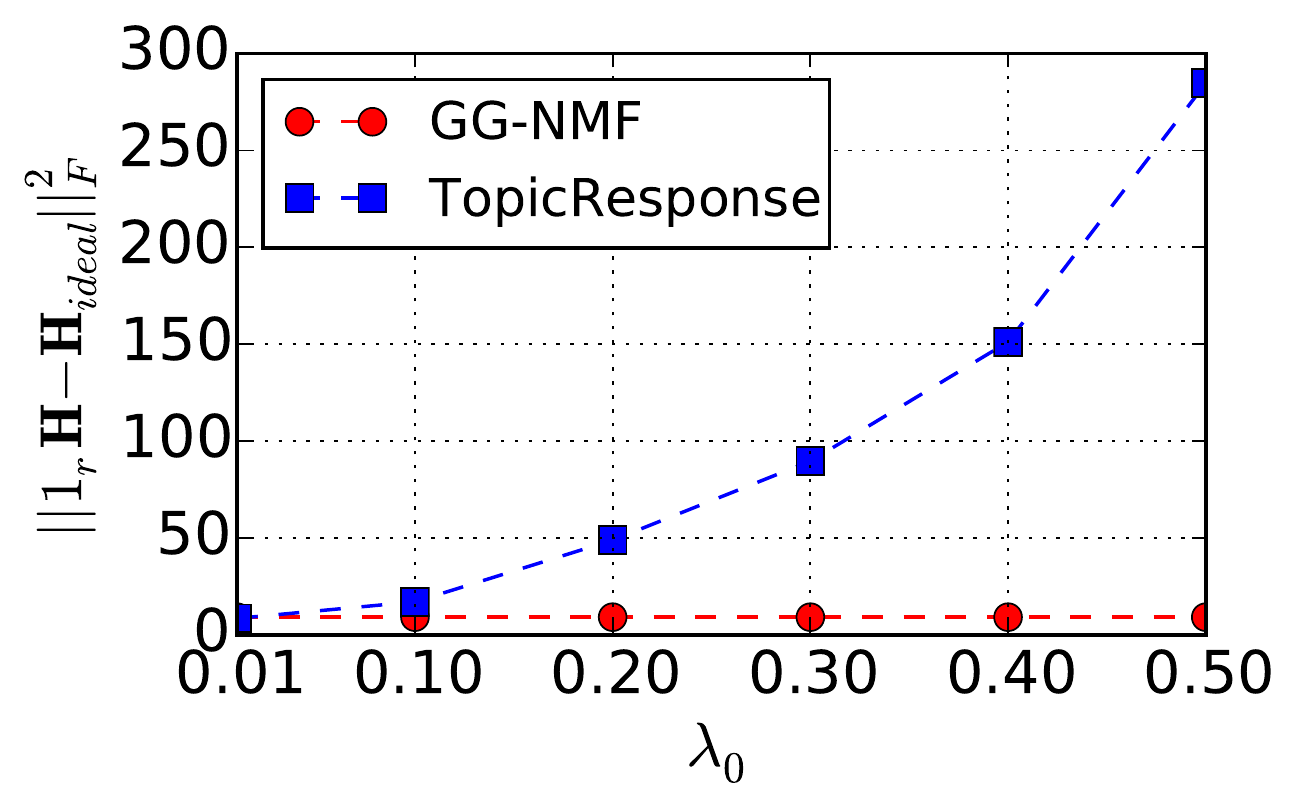}
            \caption{$\norm{\vbf{1}_r\vbf{H}-\vbf{H}_{ideal}}$ on EDU}
        \end{subfigure}%
        ~
        \begin{subfigure}[t]{0.32\textwidth}
            \centering
            \includegraphics[scale=0.30]{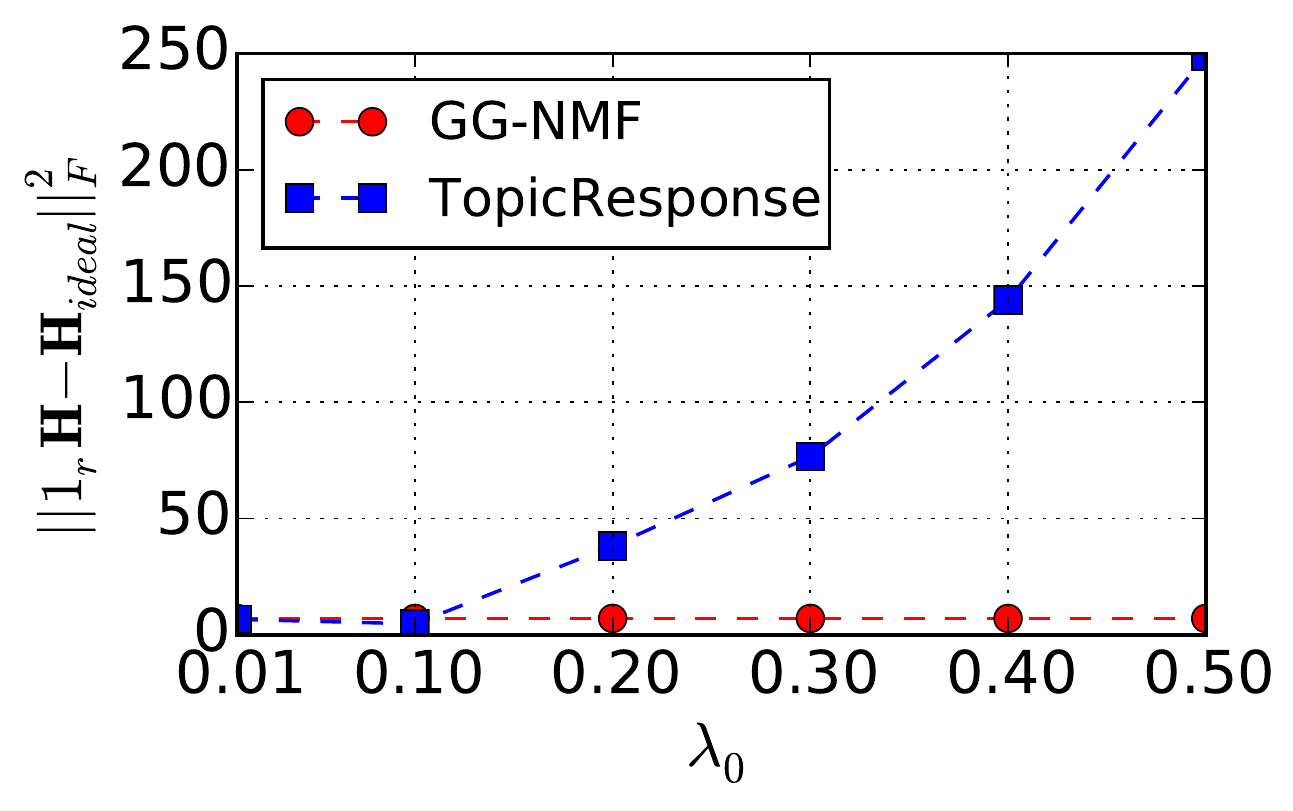}
            \caption{$\norm{\vbf{1}_r\vbf{H}-\vbf{H}_{ideal}}$ on ECON}
        \end{subfigure}
        ~
        \begin{subfigure}[t]{0.32\textwidth}
            \centering
            \includegraphics[scale=0.30]{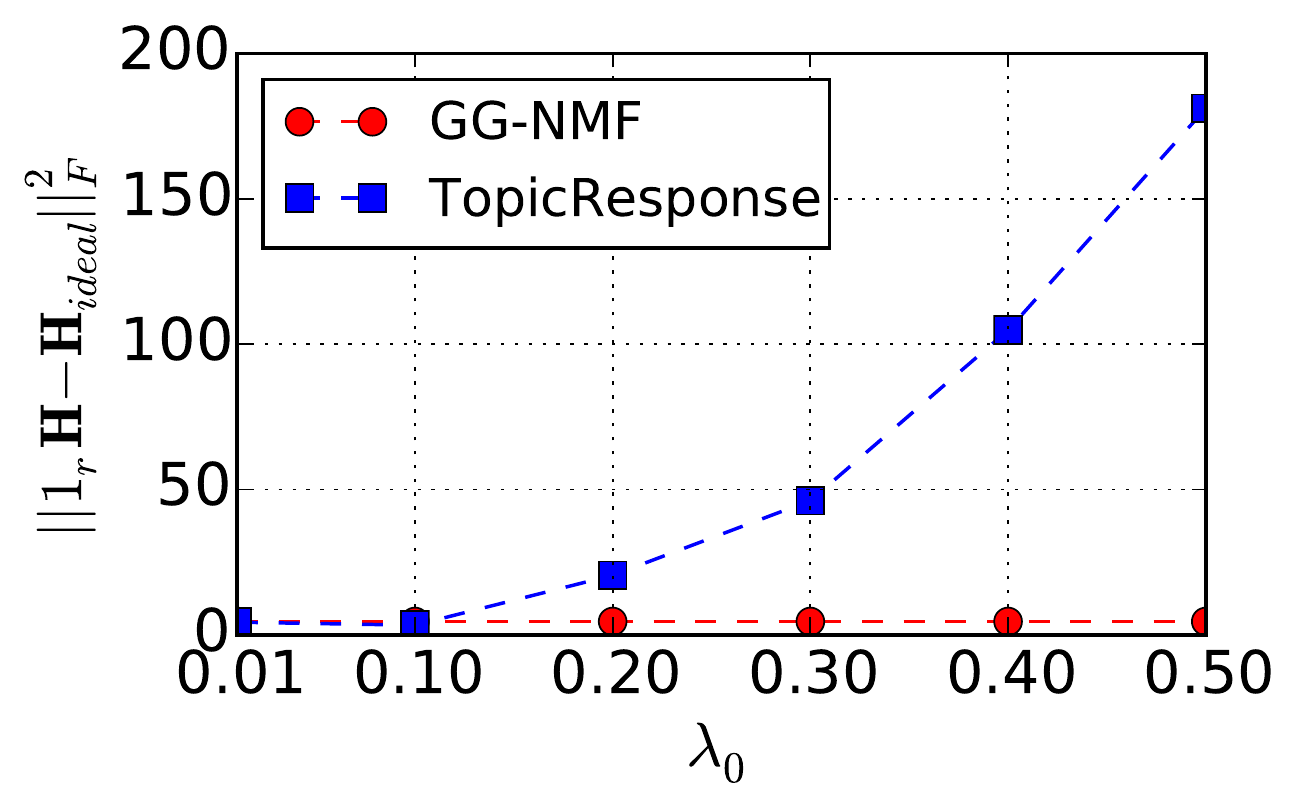}
            \caption{$\norm{\vbf{1}_r\vbf{H}-\vbf{H}_{ideal}}$ on OPT}
        \end{subfigure}
        ~
        \begin{subfigure}[t]{0.32\textwidth}
            \centering
            \includegraphics[scale=0.30]{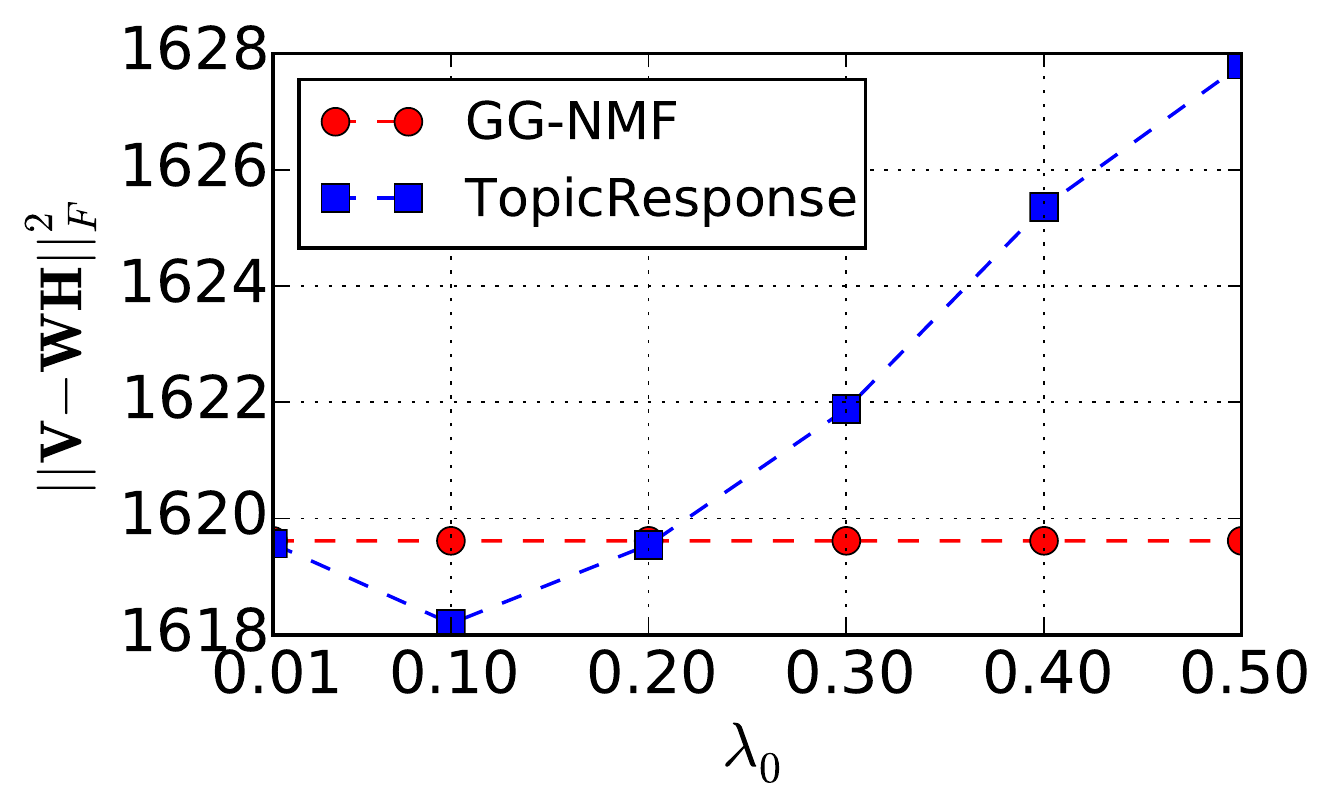}
            \caption{$\norm{\vbf{V}-\vbf{WH}}$ on EDU}
        \end{subfigure}%
        ~
        \begin{subfigure}[t]{0.32\textwidth}
            \centering
            \includegraphics[scale=0.30]{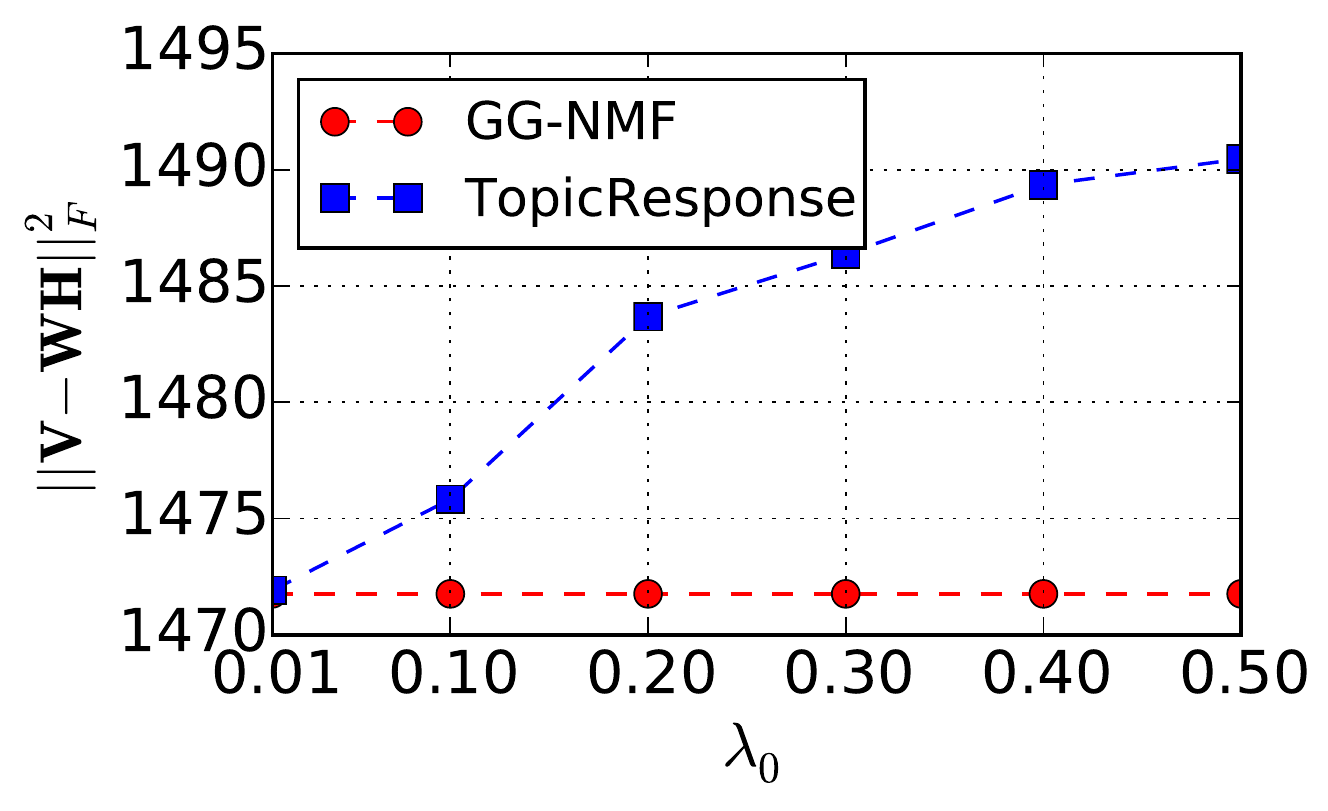}
            \caption{$\norm{\vbf{V}-\vbf{WH}}$ on ECON}
        \end{subfigure}
        ~
        \begin{subfigure}[t]{0.32\textwidth}
            \centering
            \includegraphics[scale=0.30]{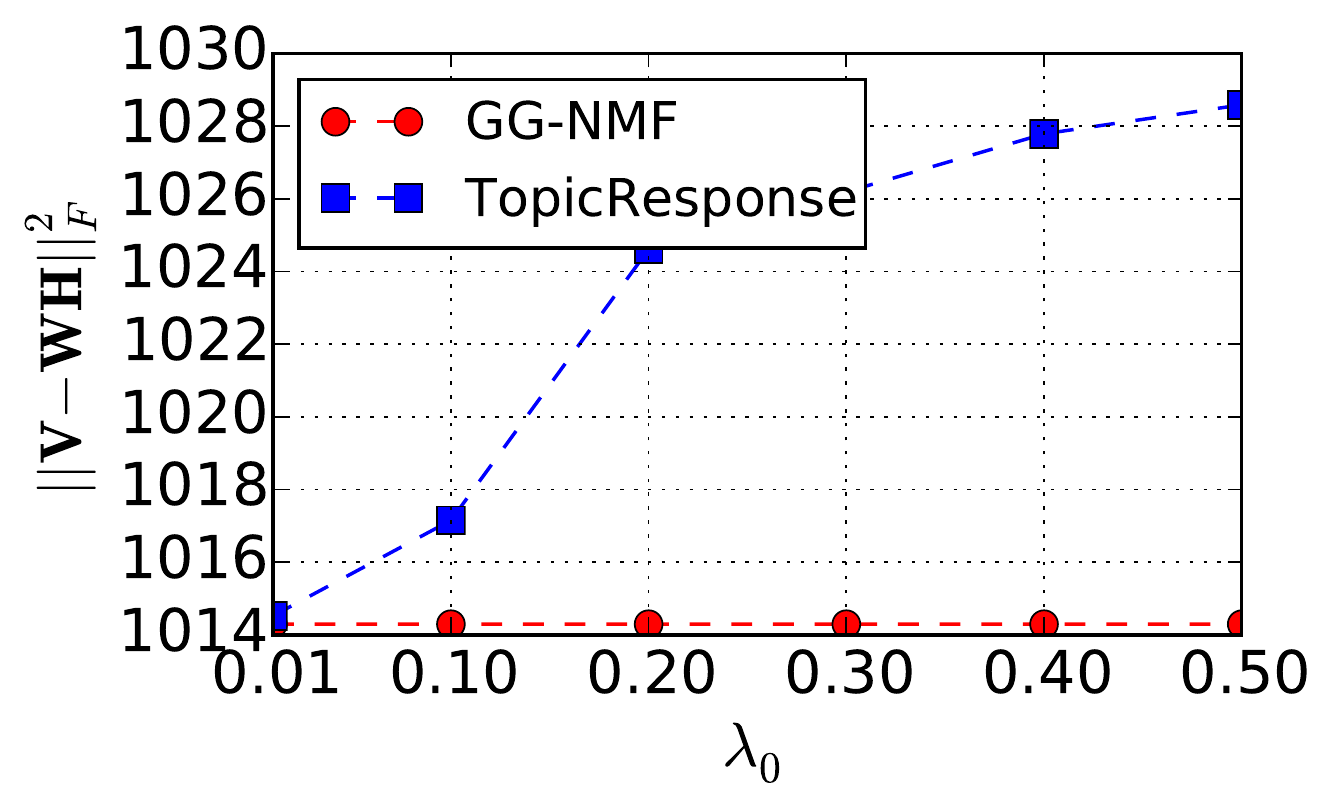}
            \caption{$\norm{\vbf{V}-\vbf{WH}}$ on OPT}
        \end{subfigure}
        ~
        \begin{subfigure}[t]{0.32\textwidth}
            \centering
            \includegraphics[scale=0.30]{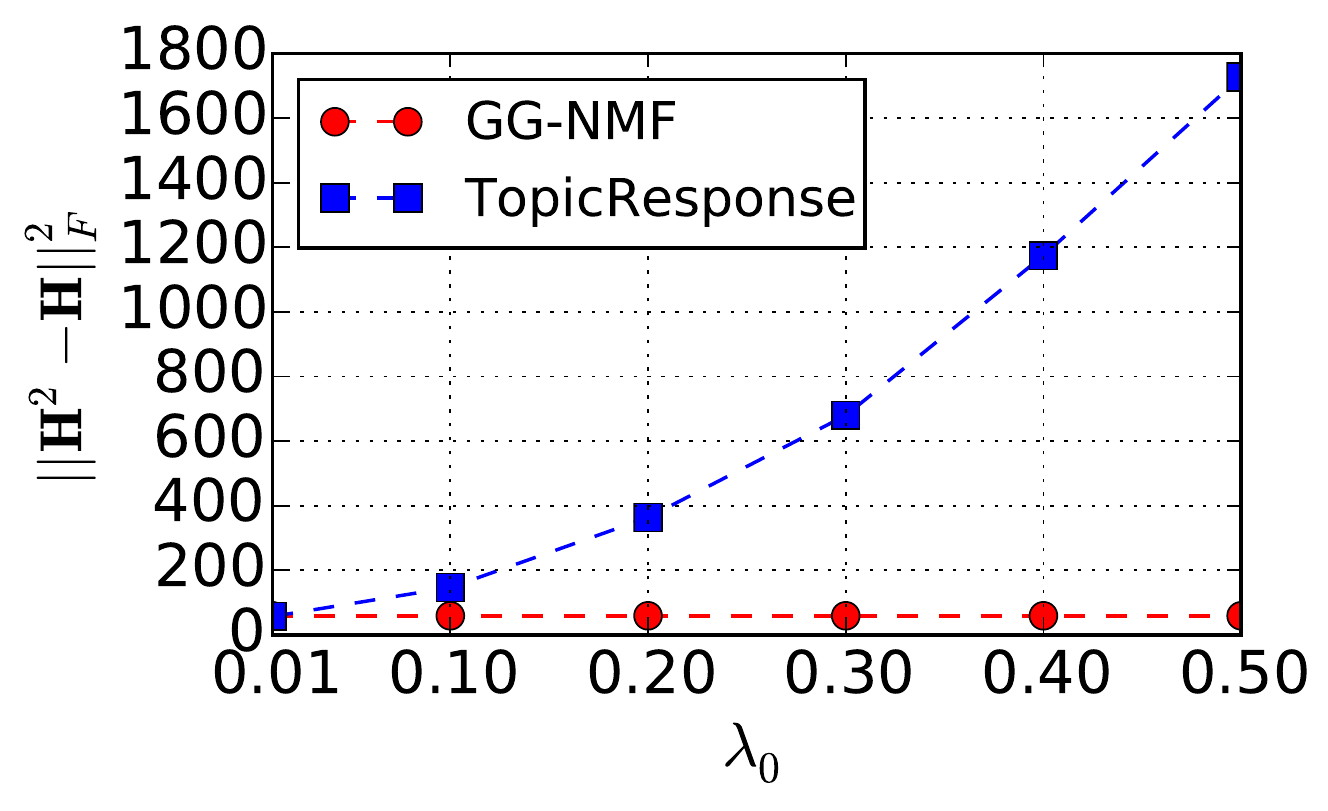}
            \caption{$\norm{\vbf{H}\circ\vbf{H}-\vbf{H}}$ on EDU}
        \end{subfigure}%
        ~
        \begin{subfigure}[t]{0.32\textwidth}
            \centering
            \includegraphics[scale=0.30]{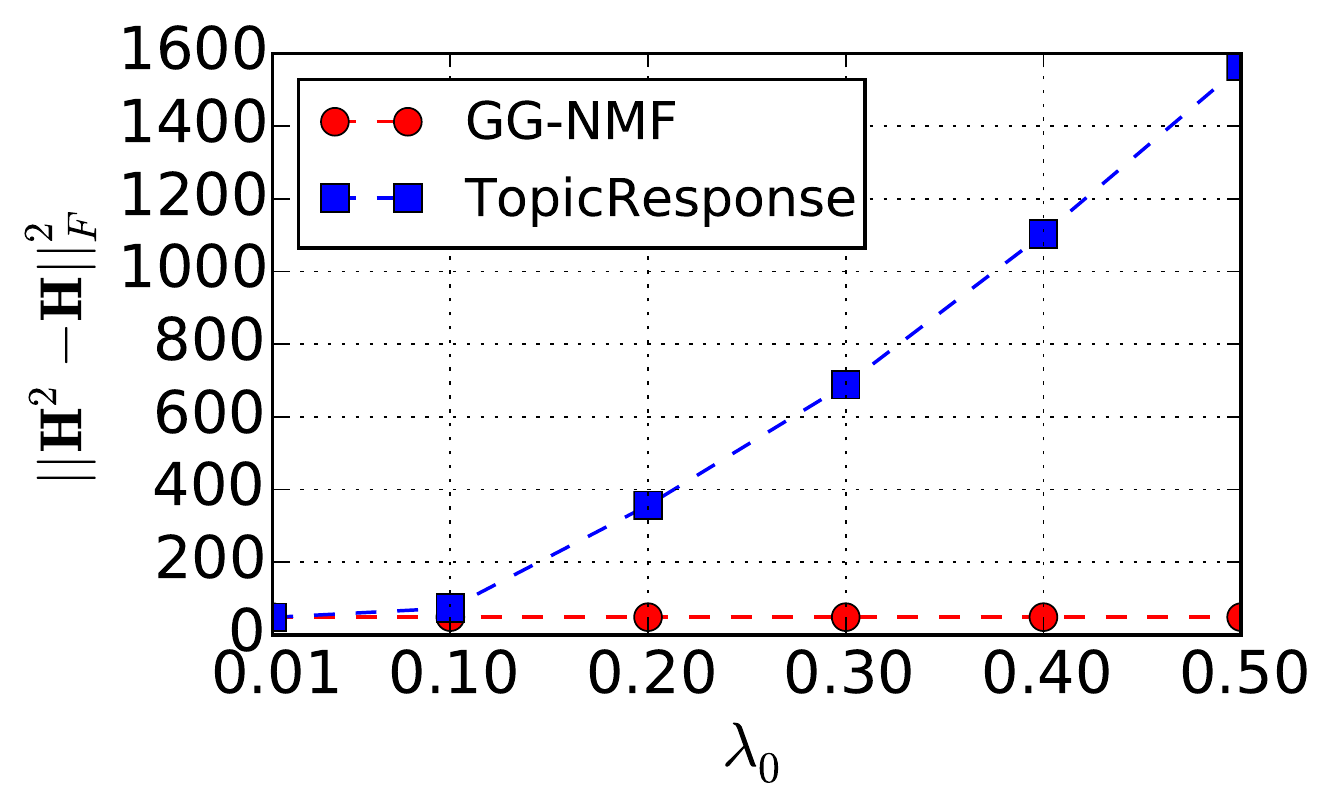}
            \caption{$\norm{\vbf{H}\circ\vbf{H}-\vbf{H}}$ on ECON}
        \end{subfigure}
        ~
        \begin{subfigure}[t]{0.32\textwidth}
            \centering
            \includegraphics[scale=0.30]{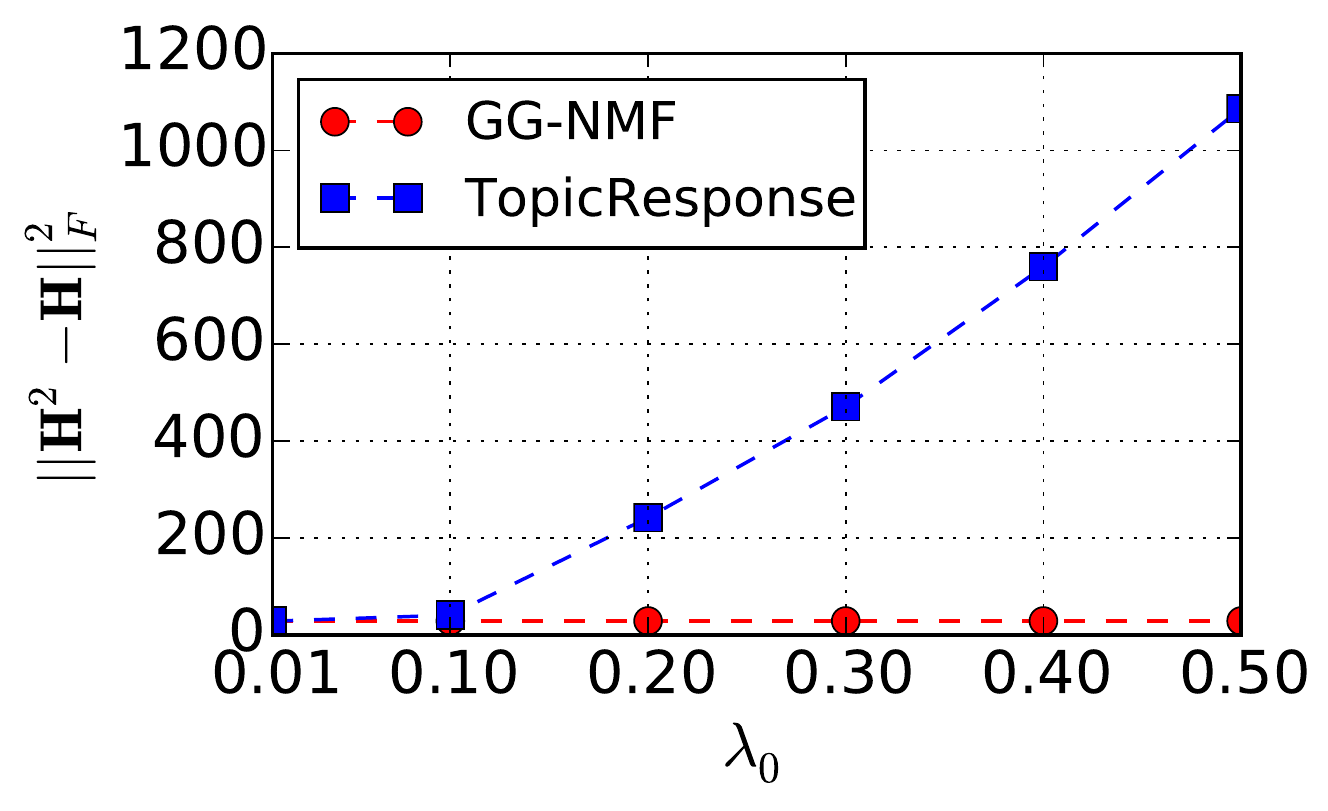}
            \caption{$\norm{\vbf{H}\circ\vbf{H}-\vbf{H}}$ on OPT}
        \end{subfigure}
        \caption{Performance of GG-NMF and TopicResponse with varying $\lambda_0$.}
        \label{fig:lambdarasch}
    \end{figure}
    
    \begin{figure}[!htb]
        \centering
        \begin{subfigure}[t]{0.32\textwidth}
            \centering
            \includegraphics[scale=0.30]{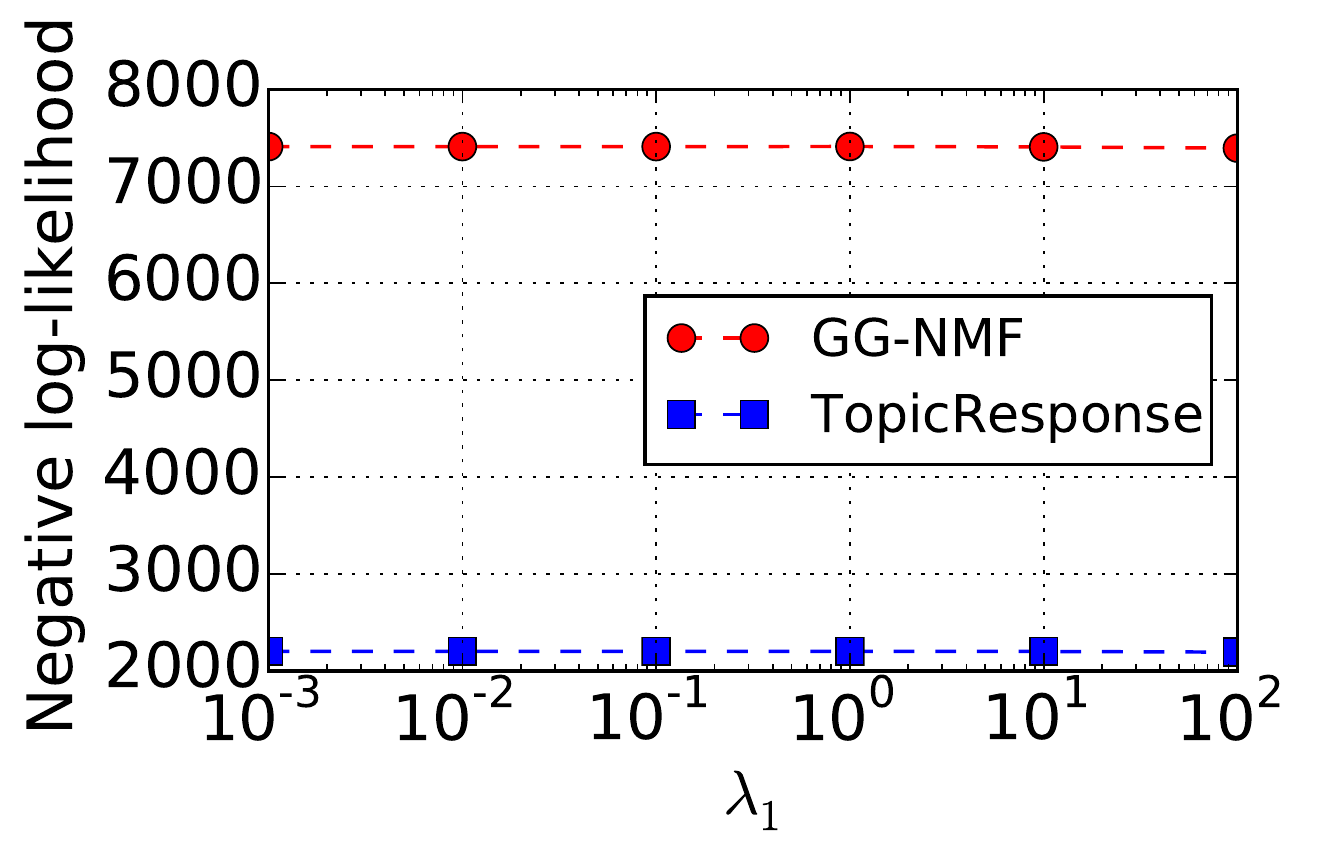}
            \caption{Negative log-likelihood on EDU}
        \end{subfigure}%
        ~
        \begin{subfigure}[t]{0.32\textwidth}
            \centering
            \includegraphics[scale=0.30]{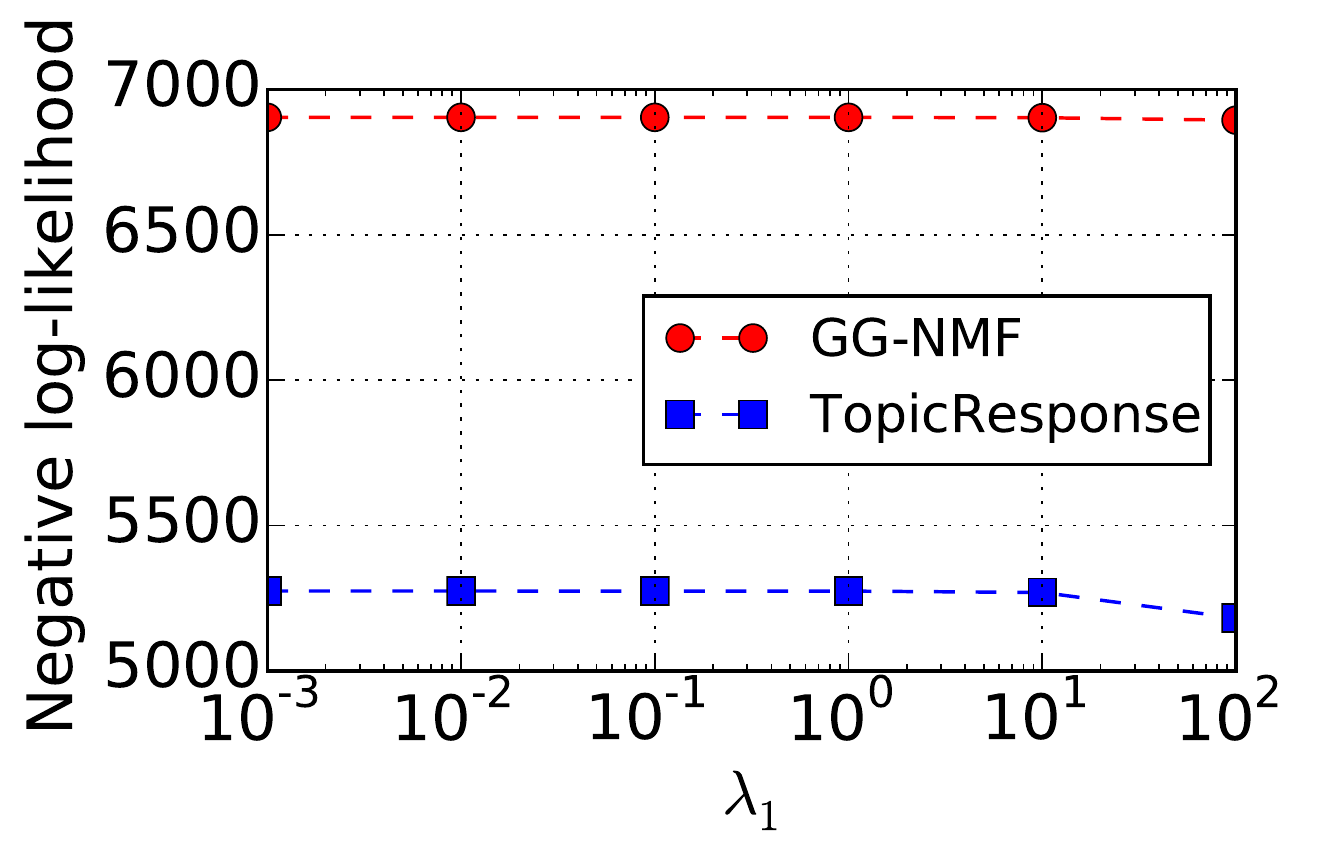}
            \caption{Negative log-likelihood on ECON}
        \end{subfigure}
        ~
        \begin{subfigure}[t]{0.32\textwidth}
            \centering
            \includegraphics[scale=0.30]{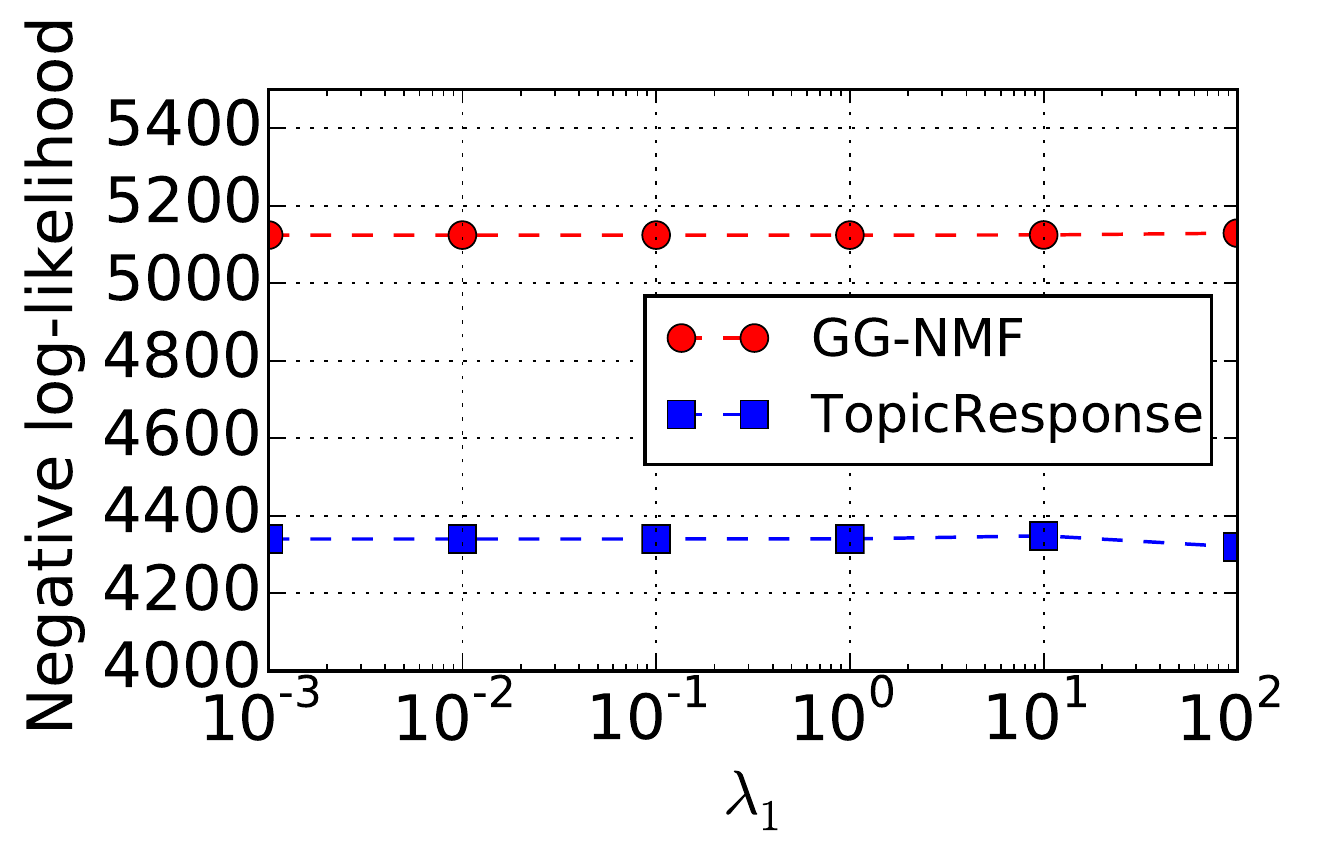}
            \caption{Negative log-likelihood on OPT}
        \end{subfigure}
        ~
        \begin{subfigure}[t]{0.32\textwidth}
            \centering
            \includegraphics[scale=0.30]{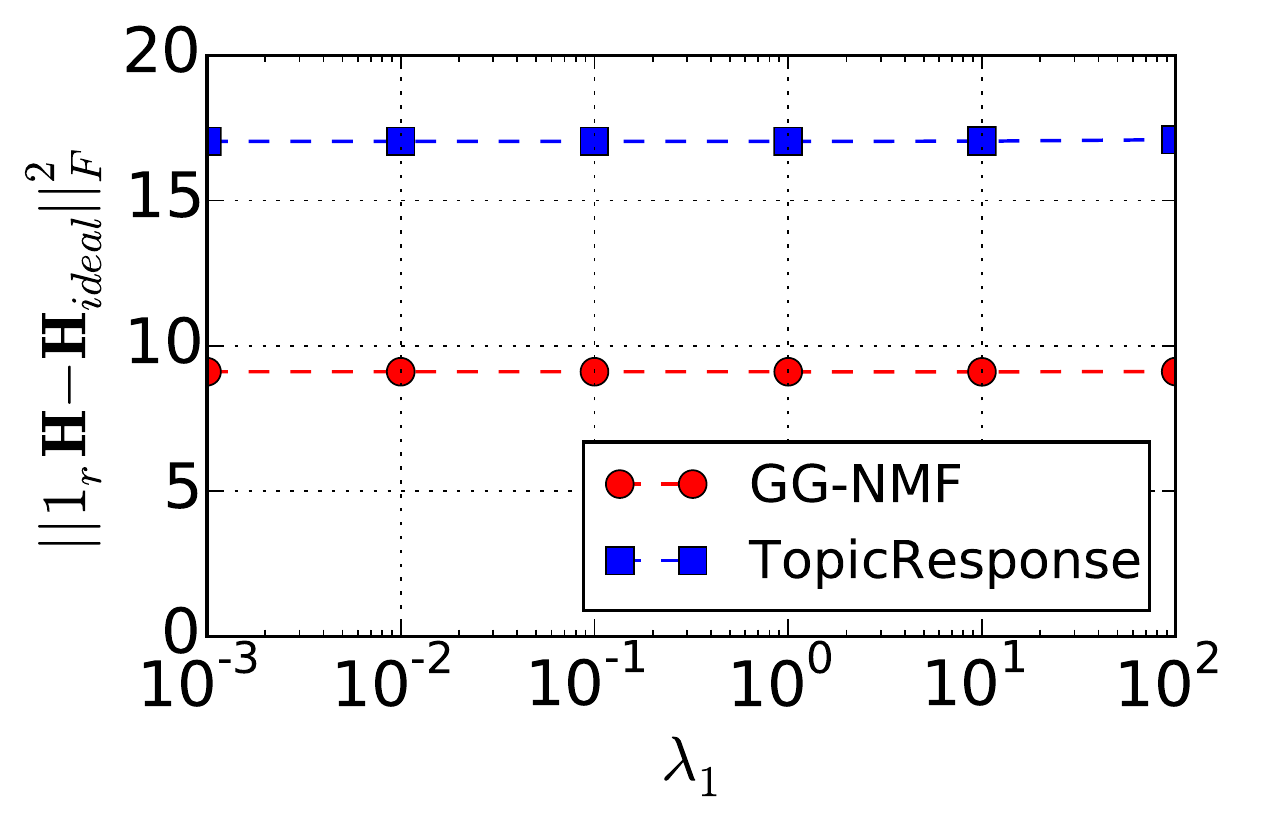}
            \caption{$\norm{\vbf{1}_r\vbf{H}-\vbf{H}_{ideal}}$ on EDU}
        \end{subfigure}%
        ~
        \begin{subfigure}[t]{0.32\textwidth}
            \centering
            \includegraphics[scale=0.30]{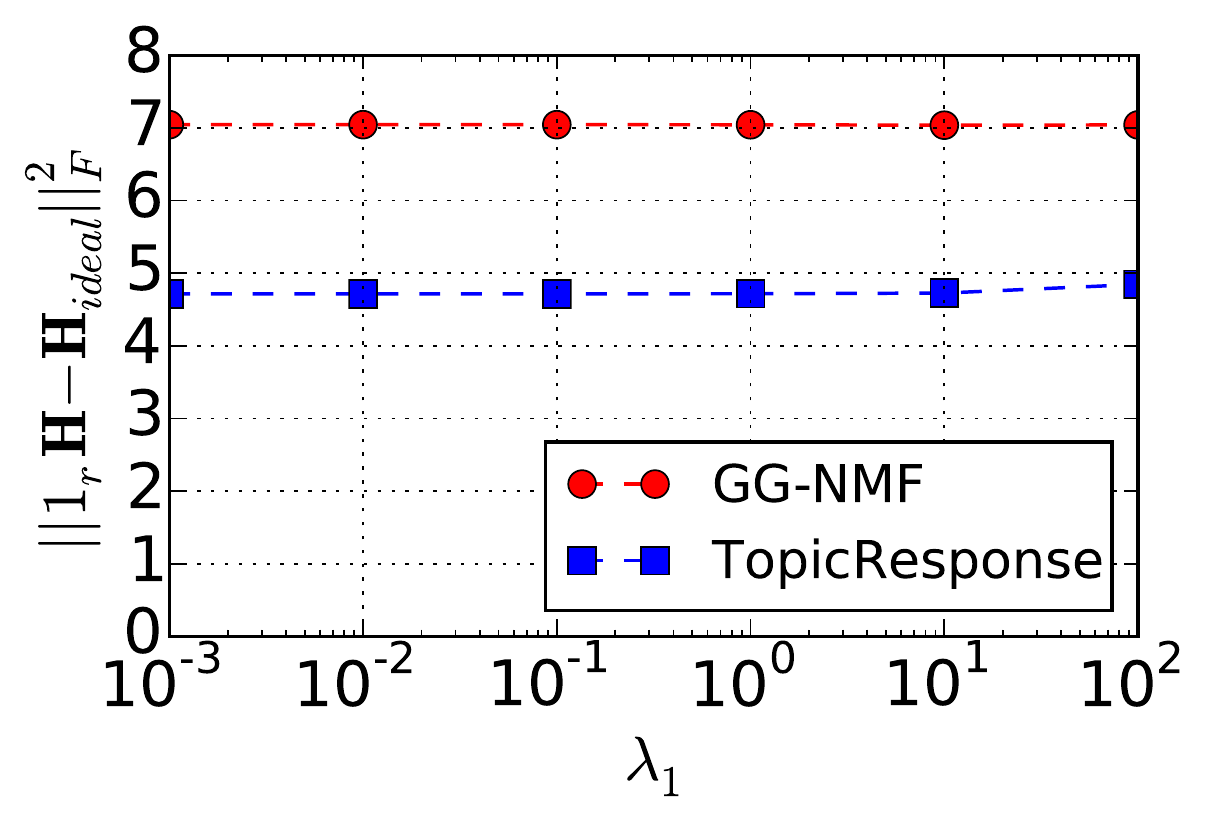}
            \caption{$\norm{\vbf{1}_r\vbf{H}-\vbf{H}_{ideal}}$ on ECON}
        \end{subfigure}
        ~
        \begin{subfigure}[t]{0.32\textwidth}
            \centering
            \includegraphics[scale=0.30]{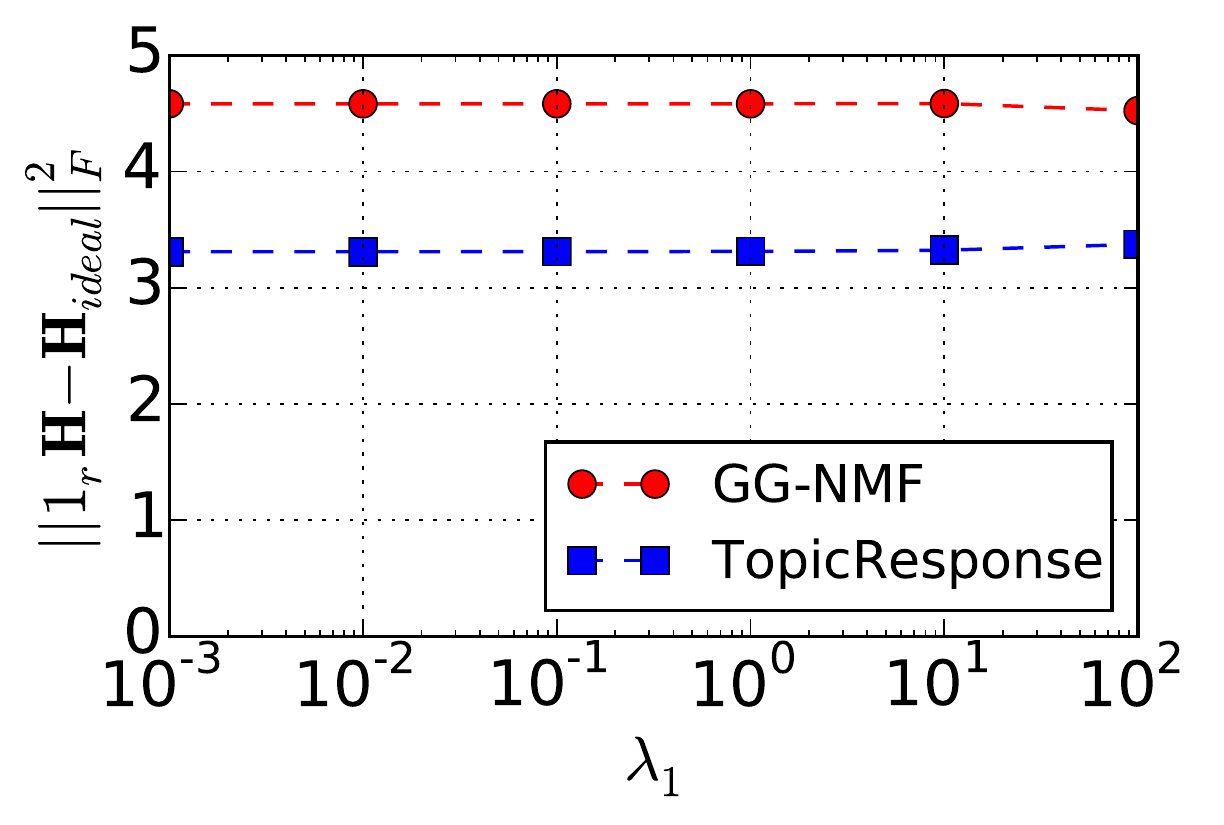}
            \caption{$\norm{\vbf{1}_r\vbf{H}-\vbf{H}_{ideal}}$ on OPT}
        \end{subfigure}
        ~
        \begin{subfigure}[t]{0.32\textwidth}
            \centering
            \includegraphics[scale=0.30]{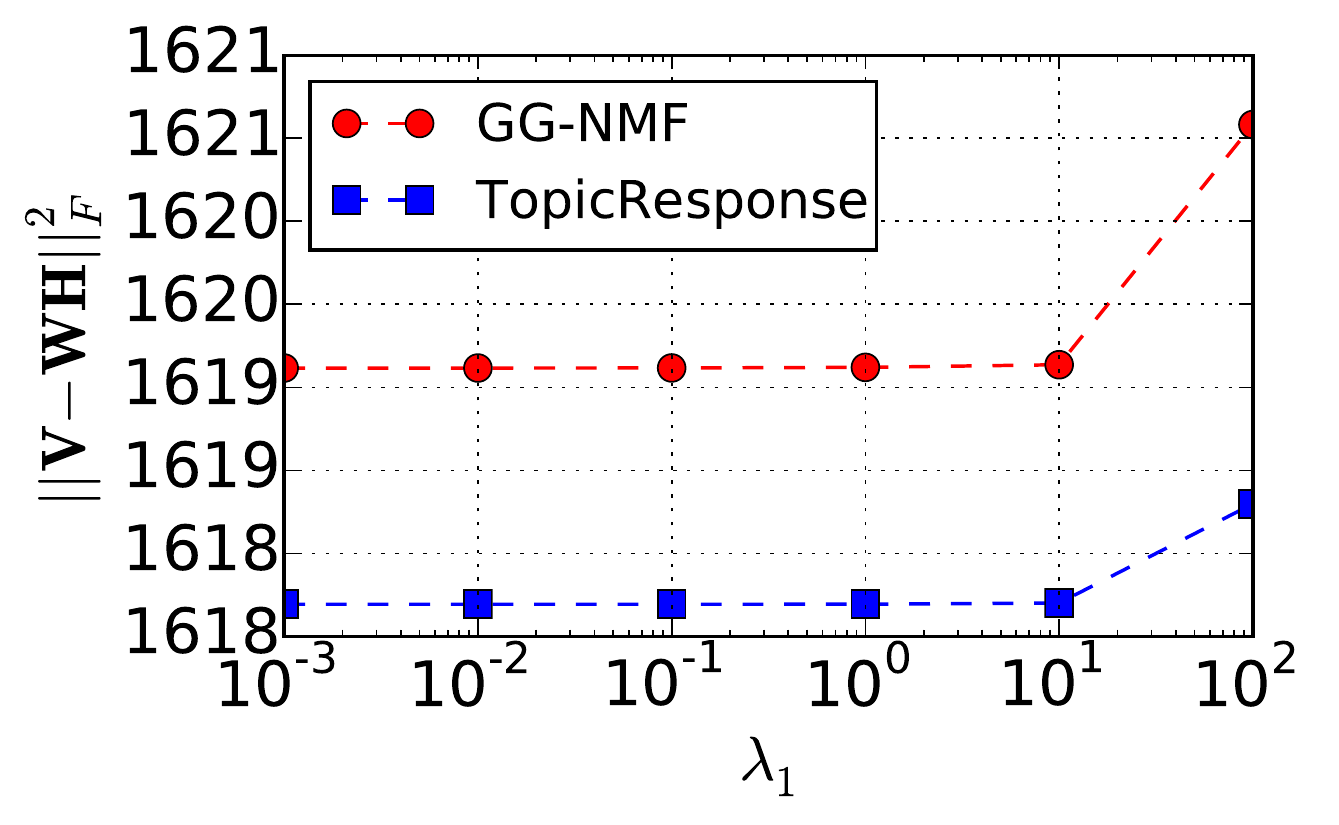}
            \caption{$\norm{\vbf{V}-\vbf{WH}}$ on EDU}
        \end{subfigure}%
        ~
        \begin{subfigure}[t]{0.32\textwidth}
            \centering
            \includegraphics[scale=0.30]{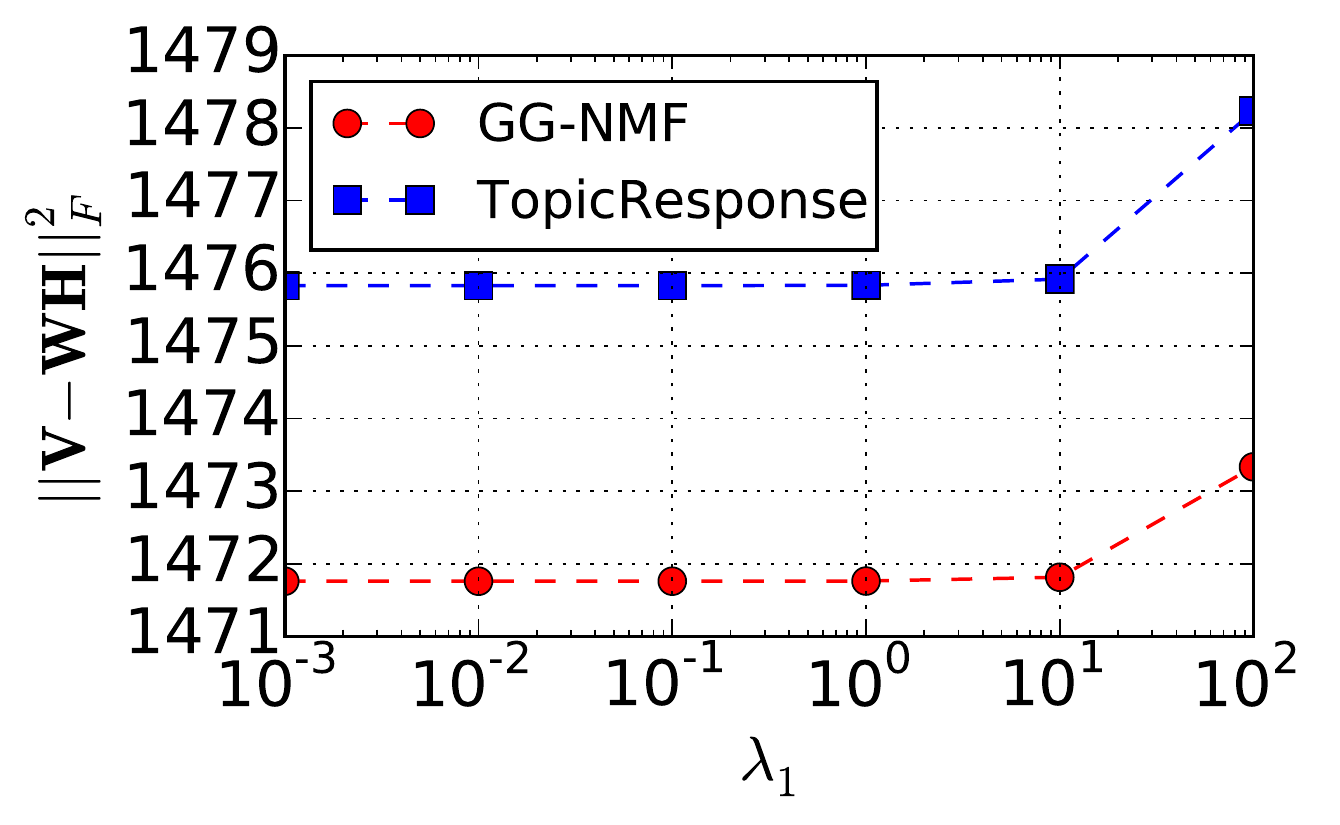}
            \caption{$\norm{\vbf{V}-\vbf{WH}}$ on ECON}
        \end{subfigure}
        ~
        \begin{subfigure}[t]{0.32\textwidth}
            \centering
            \includegraphics[scale=0.30]{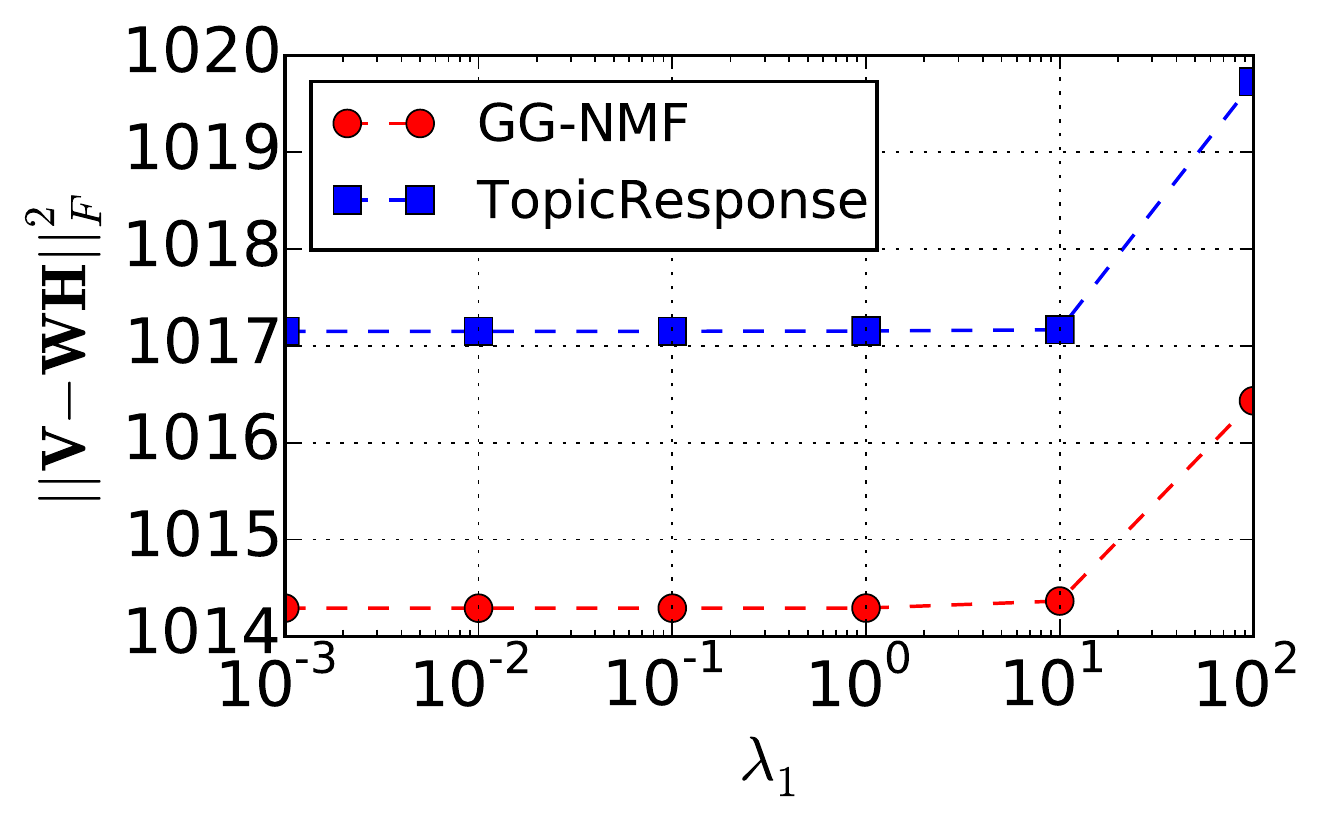}
            \caption{$\norm{\vbf{V}-\vbf{WH}}$ on OPT}
        \end{subfigure}
        ~
        \begin{subfigure}[t]{0.32\textwidth}
            \centering
            \includegraphics[scale=0.30]{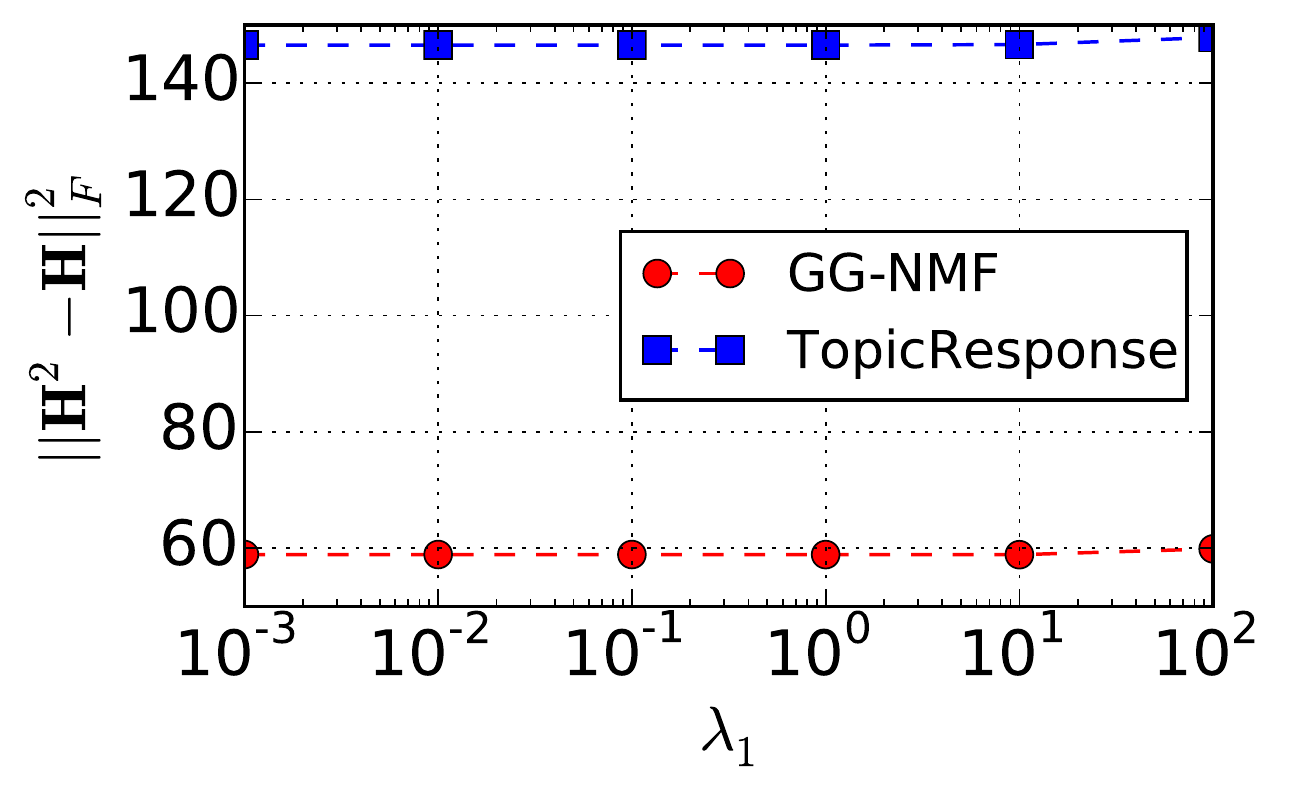}
            \caption{$\norm{\vbf{H}\circ\vbf{H}-\vbf{H}}$ on EDU}
        \end{subfigure}%
        ~
        \begin{subfigure}[t]{0.32\textwidth}
            \centering
            \includegraphics[scale=0.30]{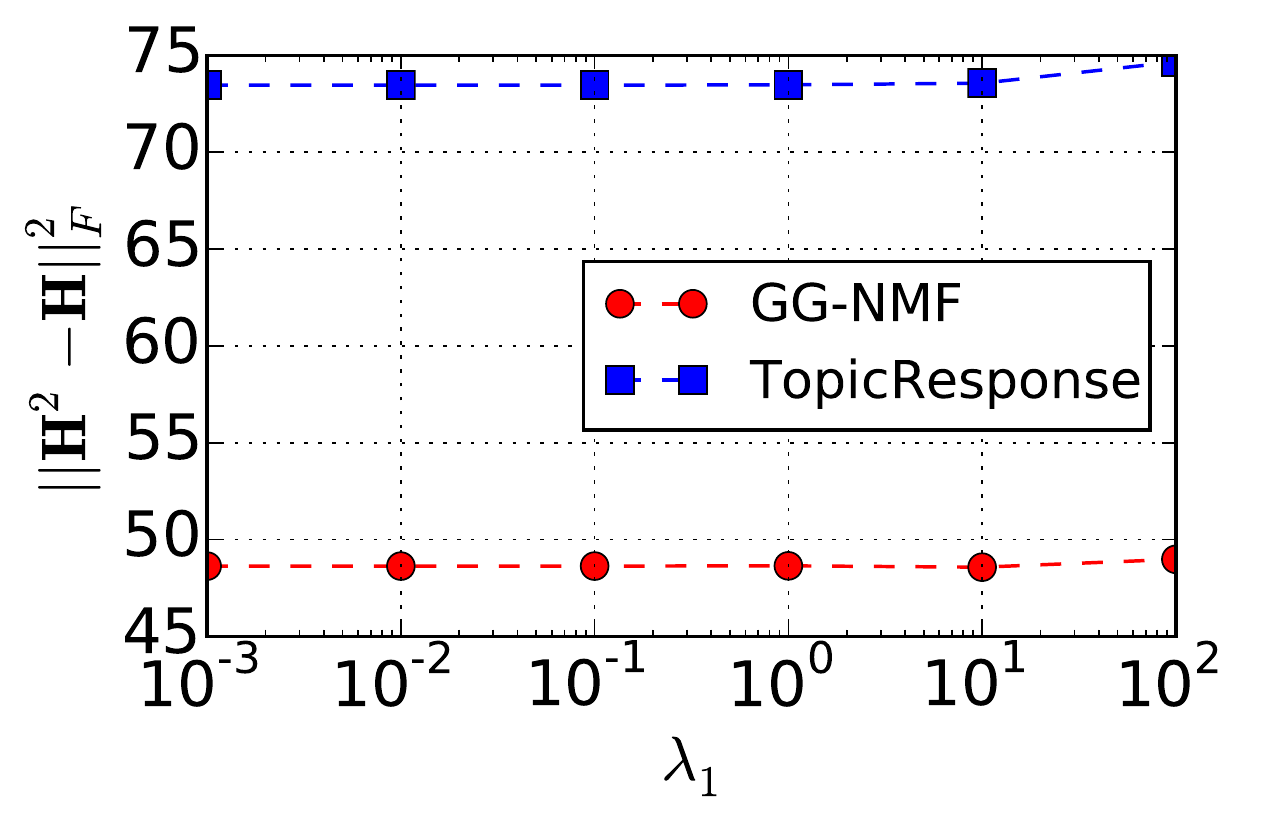}
            \caption{$\norm{\vbf{H}\circ\vbf{H}-\vbf{H}}$ on ECON}
        \end{subfigure}
        ~
        \begin{subfigure}[t]{0.32\textwidth}
            \centering
            \includegraphics[scale=0.30]{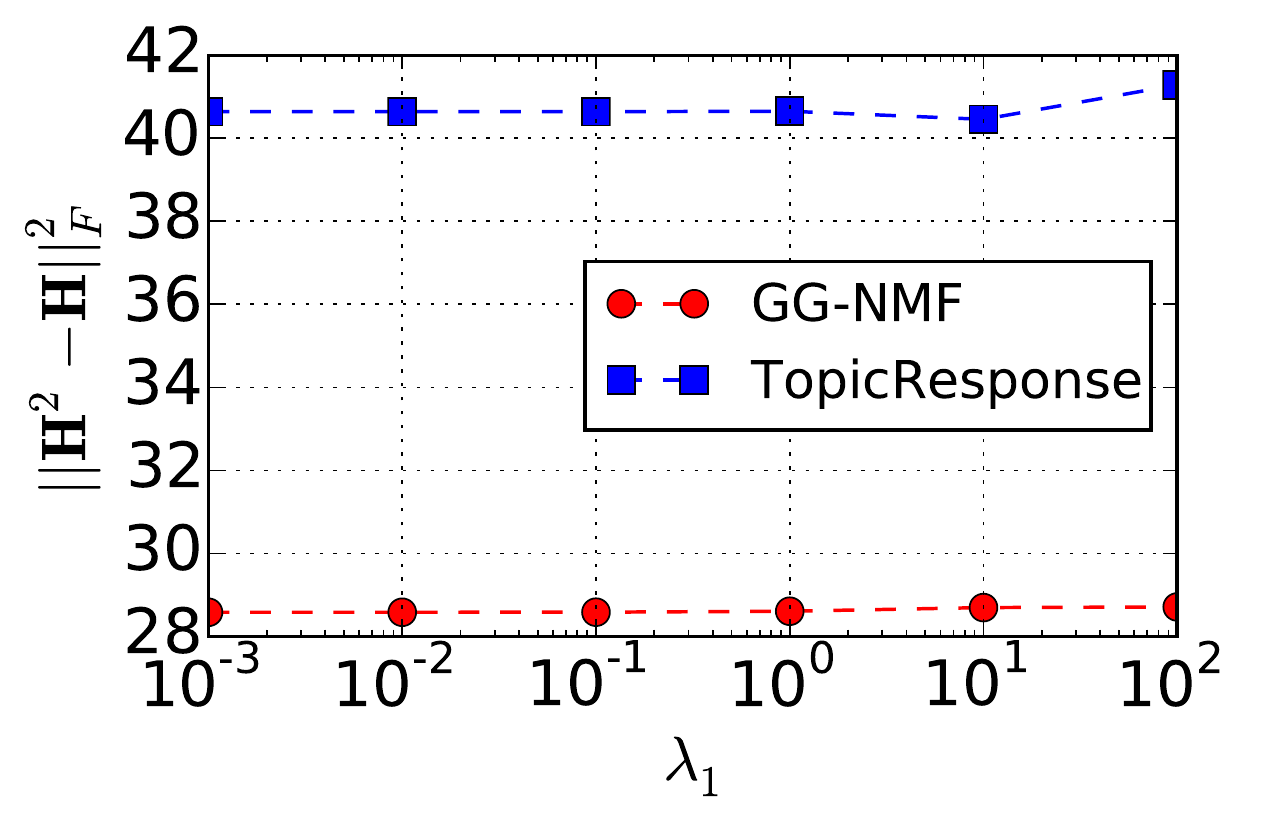}
            \caption{$\norm{\vbf{H}\circ\vbf{H}-\vbf{H}}$ on OPT}
        \end{subfigure}
        \caption{Performance of GG-NMF and TopicResponse with varying $\lambda_1$.}
        \label{fig:lambdaw}
    \end{figure}
    
    \begin{figure}[!htb]
        \centering
        \begin{subfigure}[t]{0.32\textwidth}
            \centering
            \includegraphics[scale=0.30]{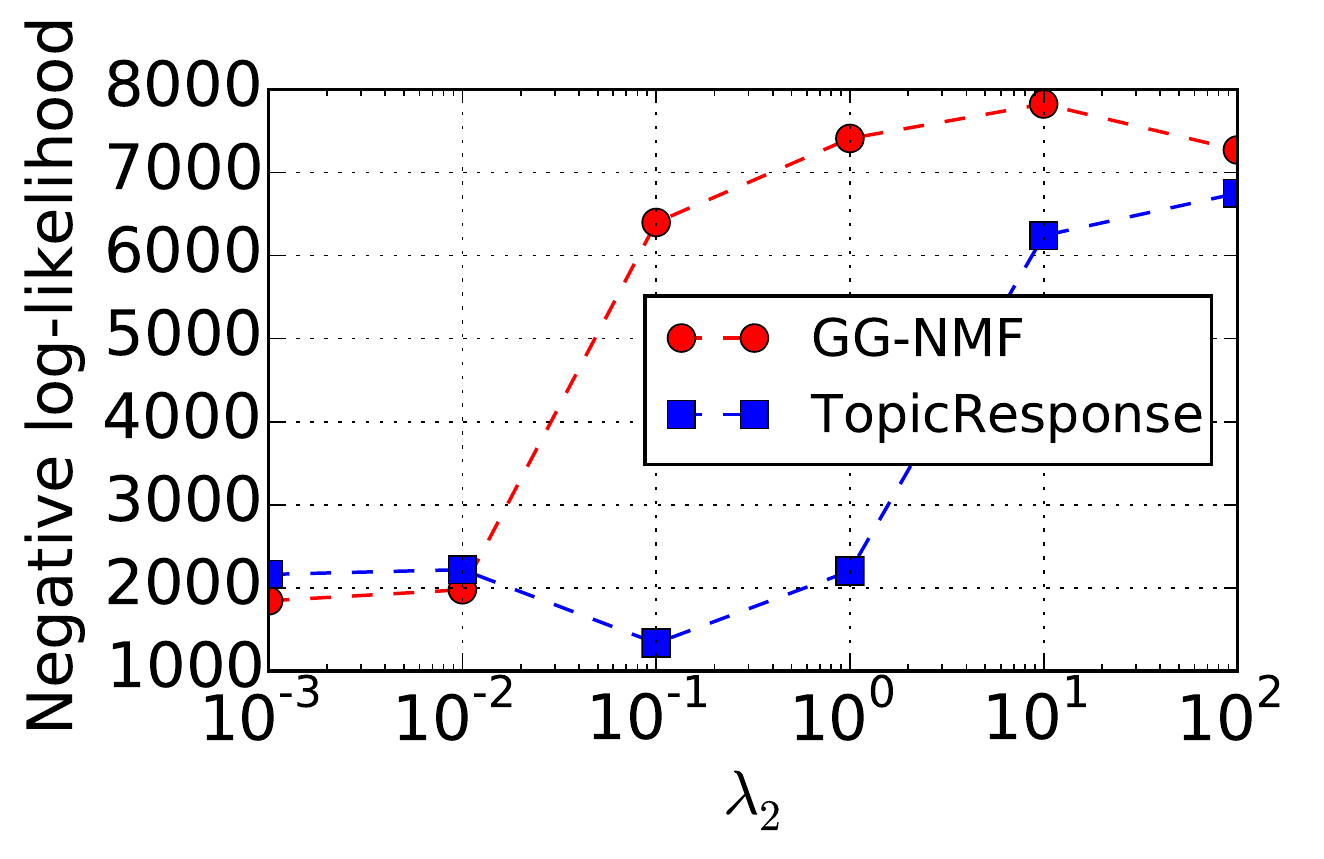}
            \caption{Negative log-likelihood on EDU}
        \end{subfigure}%
        ~
        \begin{subfigure}[t]{0.32\textwidth}
            \centering
            \includegraphics[scale=0.30]{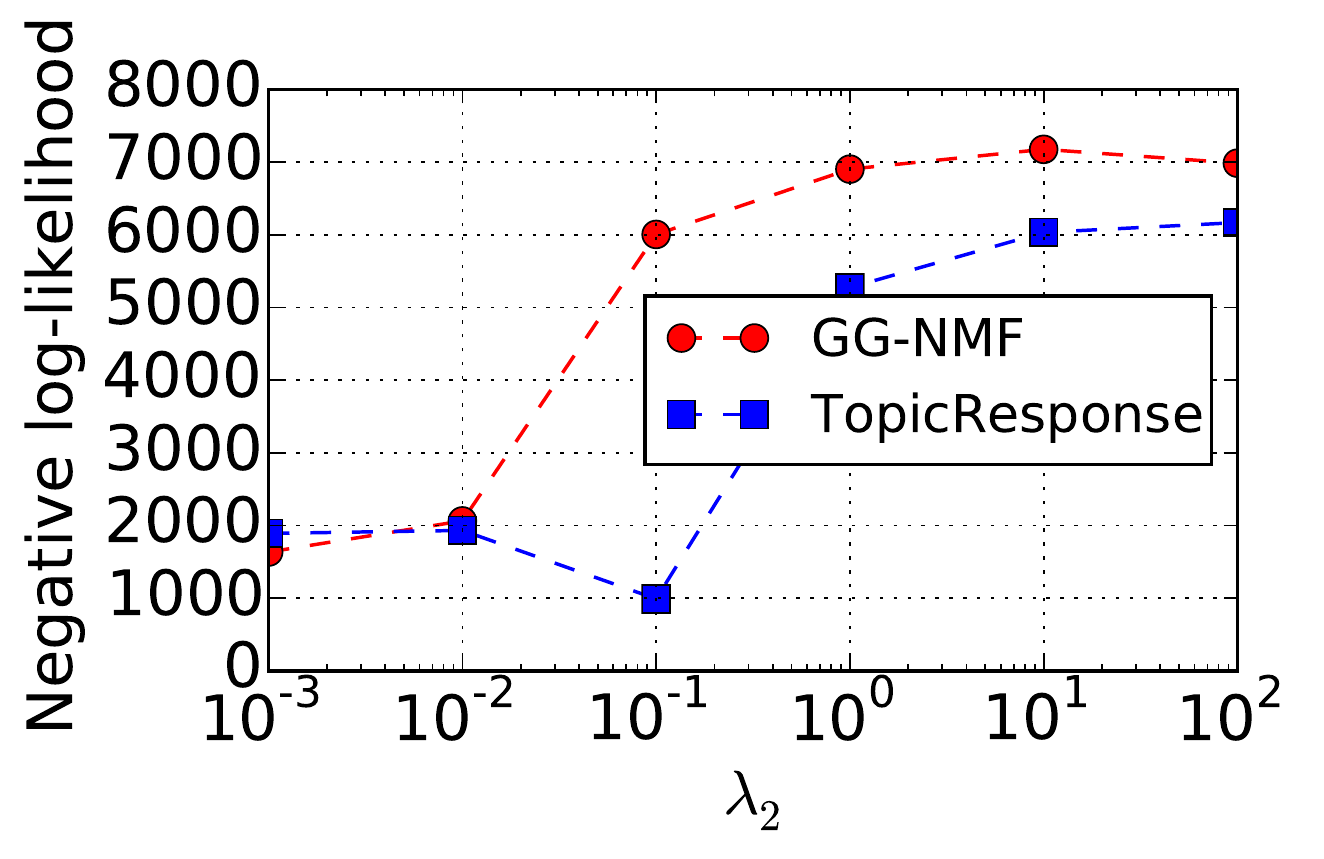}
            \caption{Negative log-likelihood on ECON}
        \end{subfigure}
        ~
        \begin{subfigure}[t]{0.32\textwidth}
            \centering
            \includegraphics[scale=0.30]{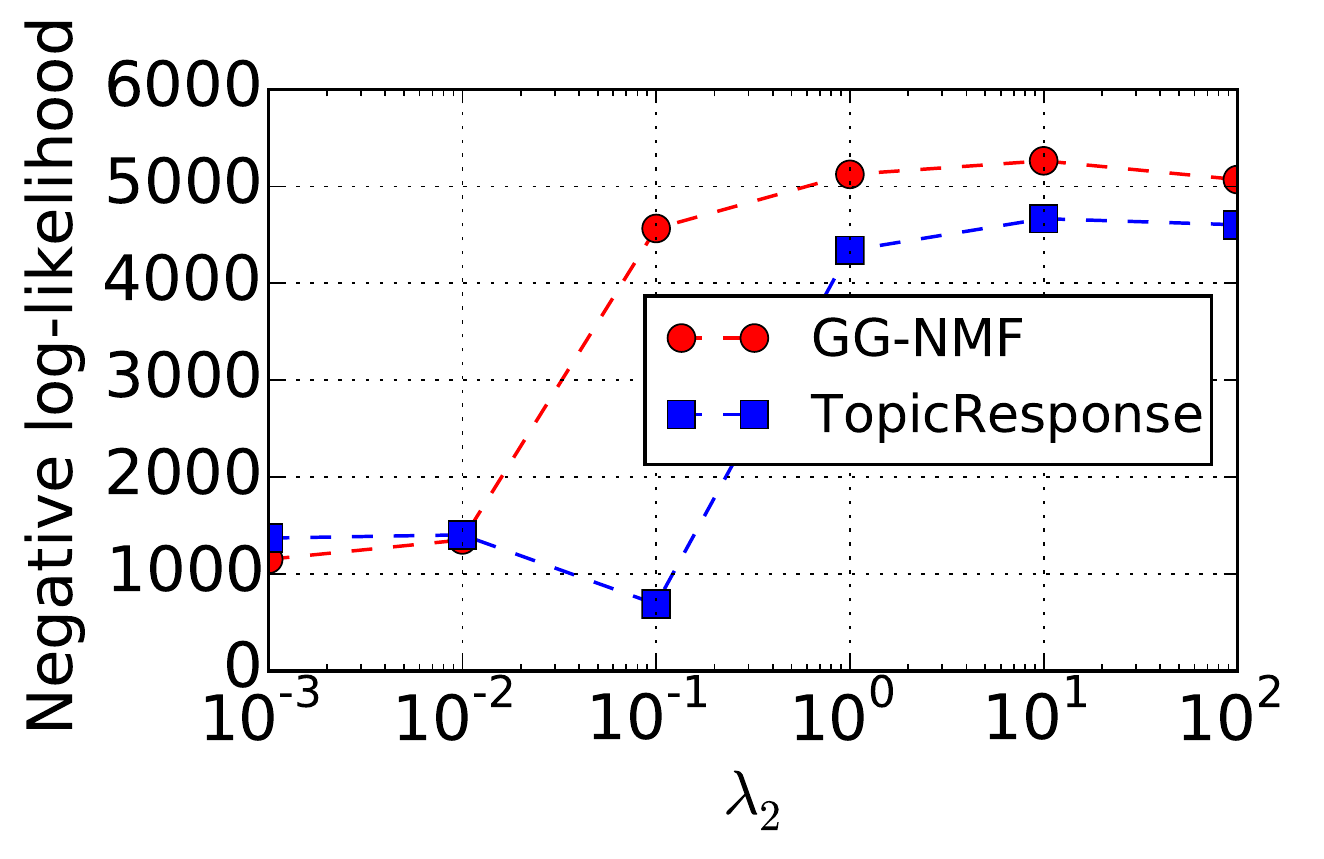}
            \caption{Negative log-likelihood on OPT}
        \end{subfigure}
        ~
        \begin{subfigure}[t]{0.32\textwidth}
            \centering
            \includegraphics[scale=0.30]{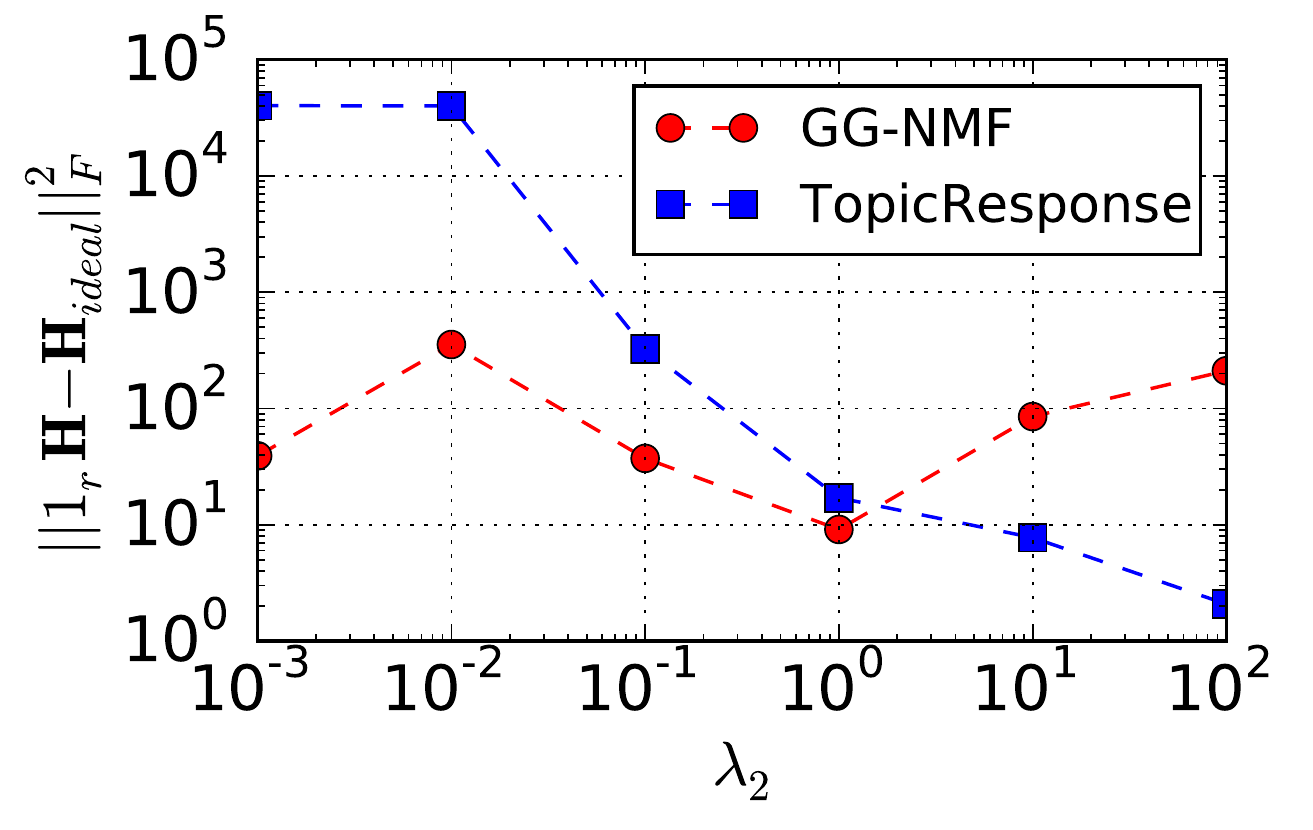}
            \caption{$\norm{\vbf{1}_r\vbf{H}-\vbf{H}_{ideal}}$ on EDU}
            \label{fig:lambdahidealhideal1}
        \end{subfigure}%
        ~
        \begin{subfigure}[t]{0.32\textwidth}
            \centering
            \includegraphics[scale=0.30]{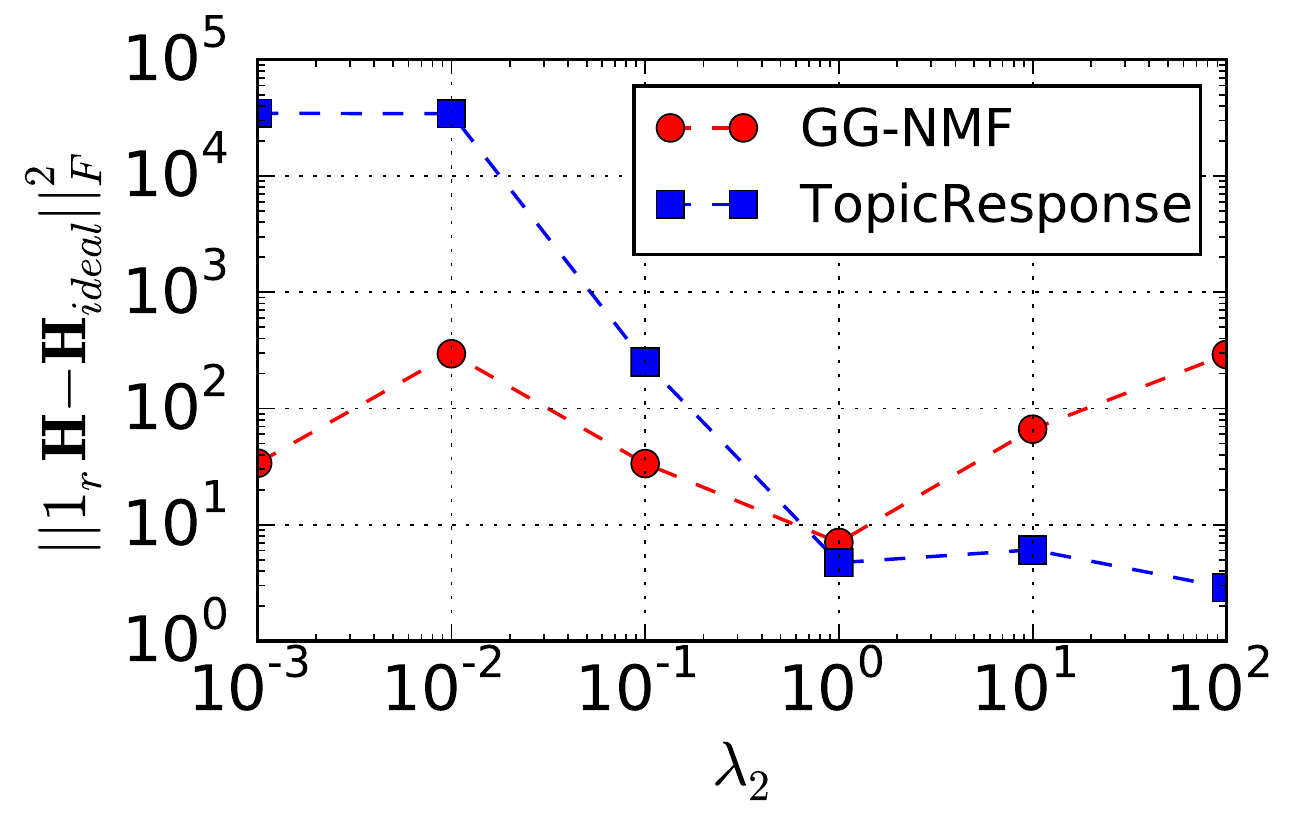}
            \caption{$\norm{\vbf{1}_r\vbf{H}-\vbf{H}_{ideal}}$ on ECON}
            \label{fig:lambdahidealhideal2}
        \end{subfigure}
        ~
        \begin{subfigure}[t]{0.32\textwidth}
            \centering
            \includegraphics[scale=0.30]{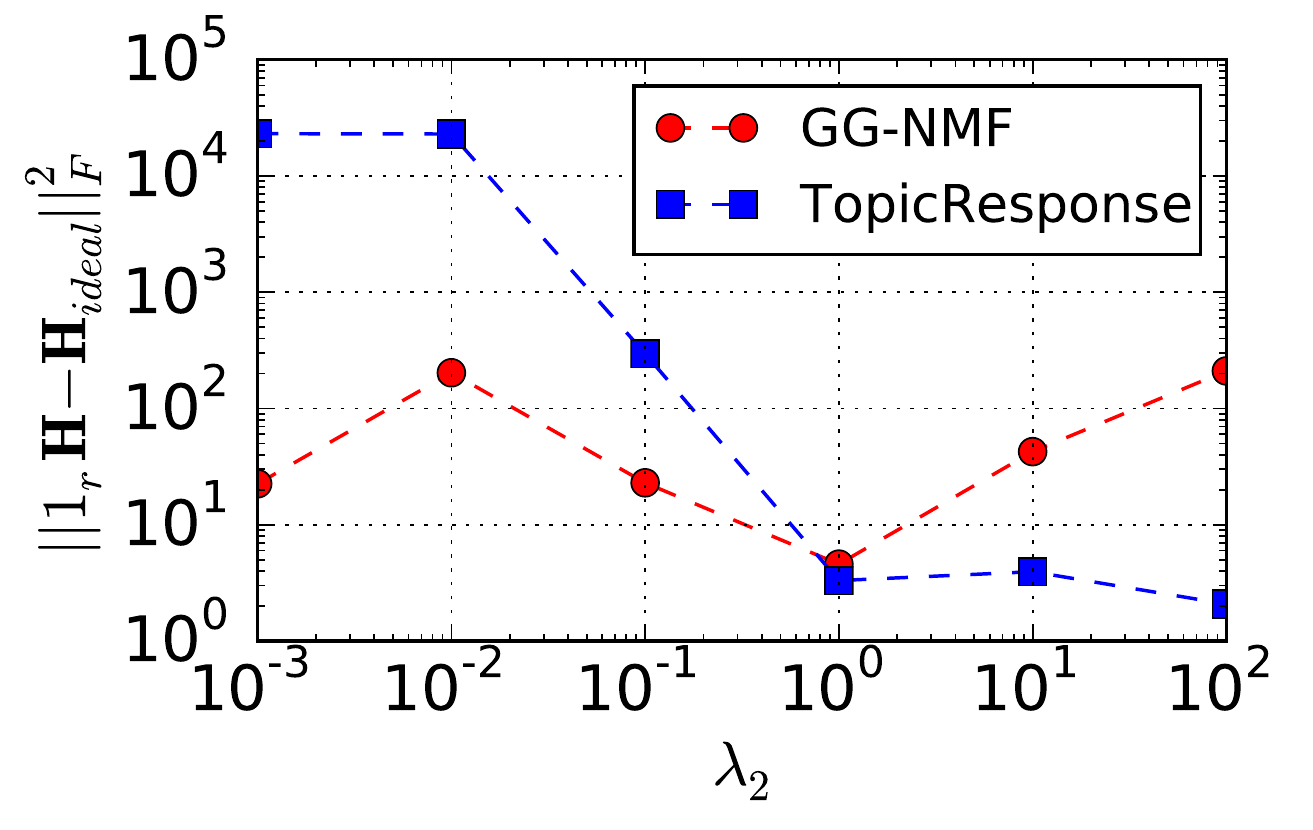}
            \caption{$\norm{\vbf{1}_r\vbf{H}-\vbf{H}_{ideal}}$ on OPT}
            \label{fig:lambdahidealhideal3}
        \end{subfigure}
        ~
        \begin{subfigure}[t]{0.32\textwidth}
            \centering
            \includegraphics[scale=0.30]{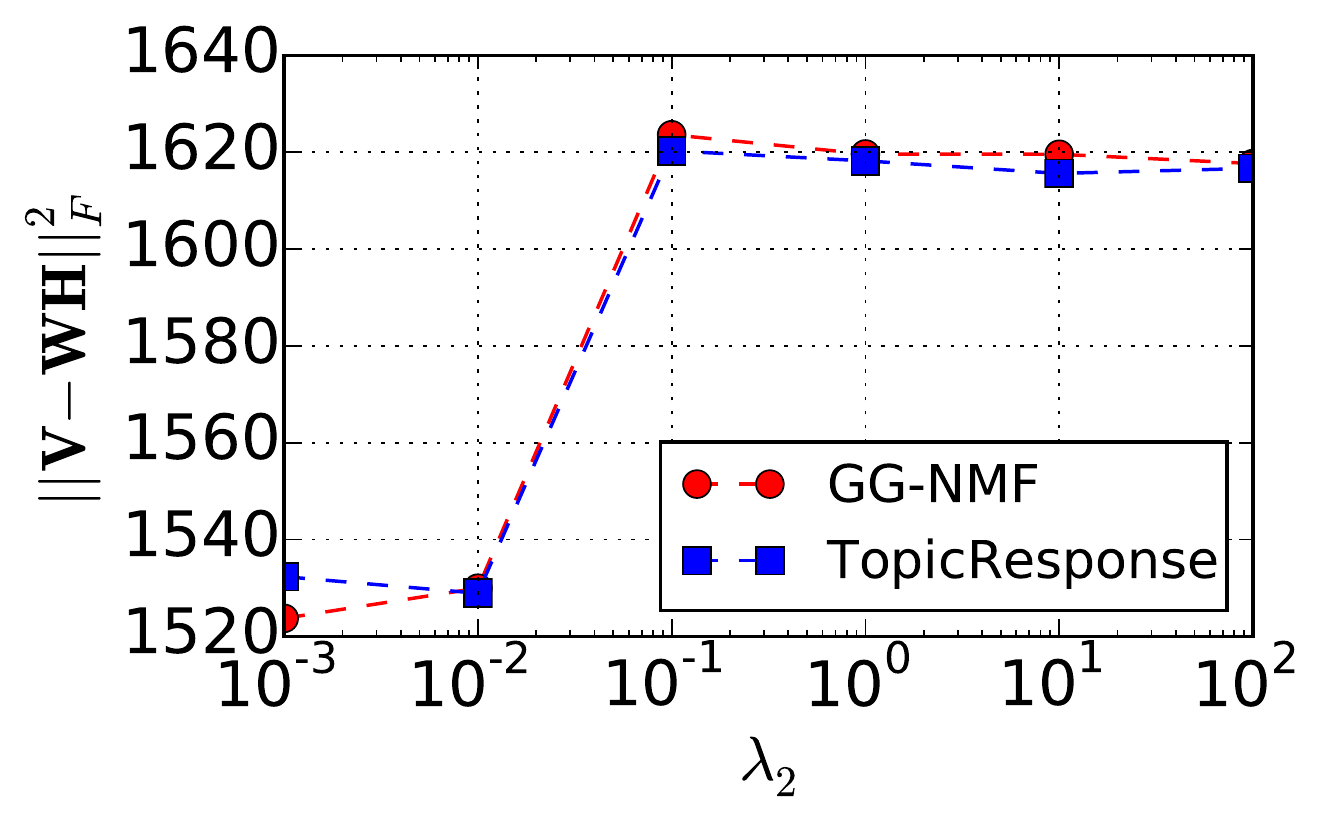}
            \caption{$\norm{\vbf{V}-\vbf{WH}}$ on EDU}
        \end{subfigure}%
        ~
        \begin{subfigure}[t]{0.32\textwidth}
            \centering
            \includegraphics[scale=0.30]{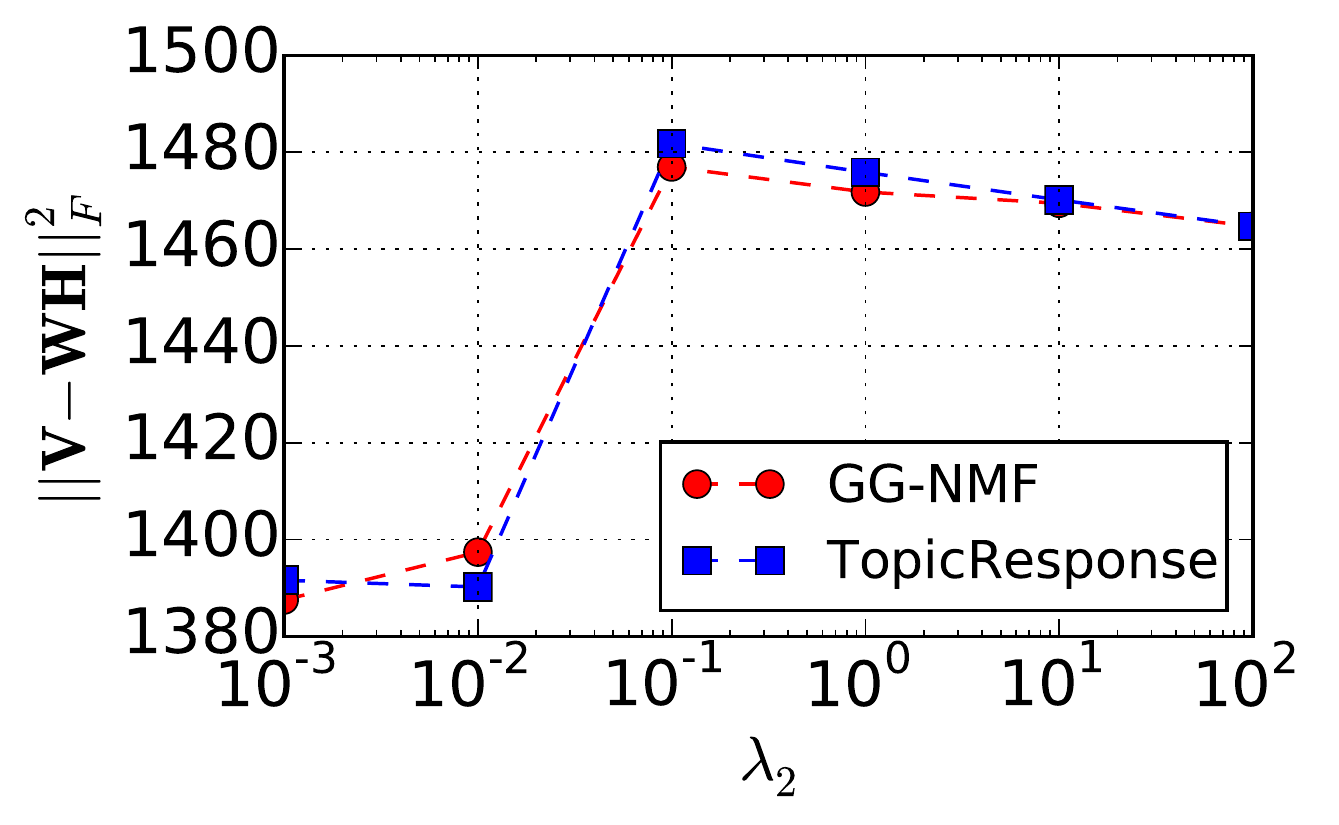}
            \caption{$\norm{\vbf{V}-\vbf{WH}}$ on ECON}
        \end{subfigure}
        ~
        \begin{subfigure}[t]{0.32\textwidth}
            \centering
            \includegraphics[scale=0.30]{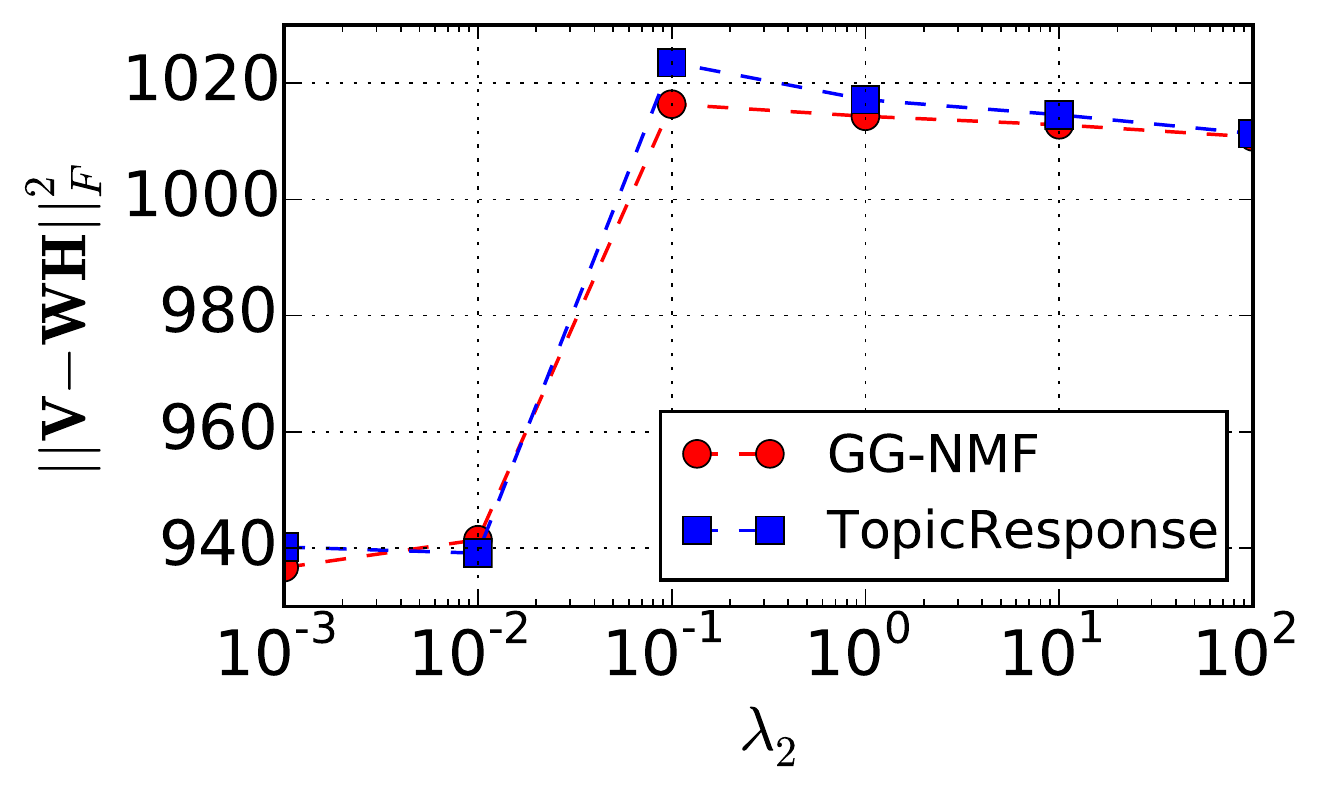}
            \caption{$\norm{\vbf{V}-\vbf{WH}}$ on OPT}
        \end{subfigure}
        ~
        \begin{subfigure}[t]{0.32\textwidth}
            \centering
            \includegraphics[scale=0.30]{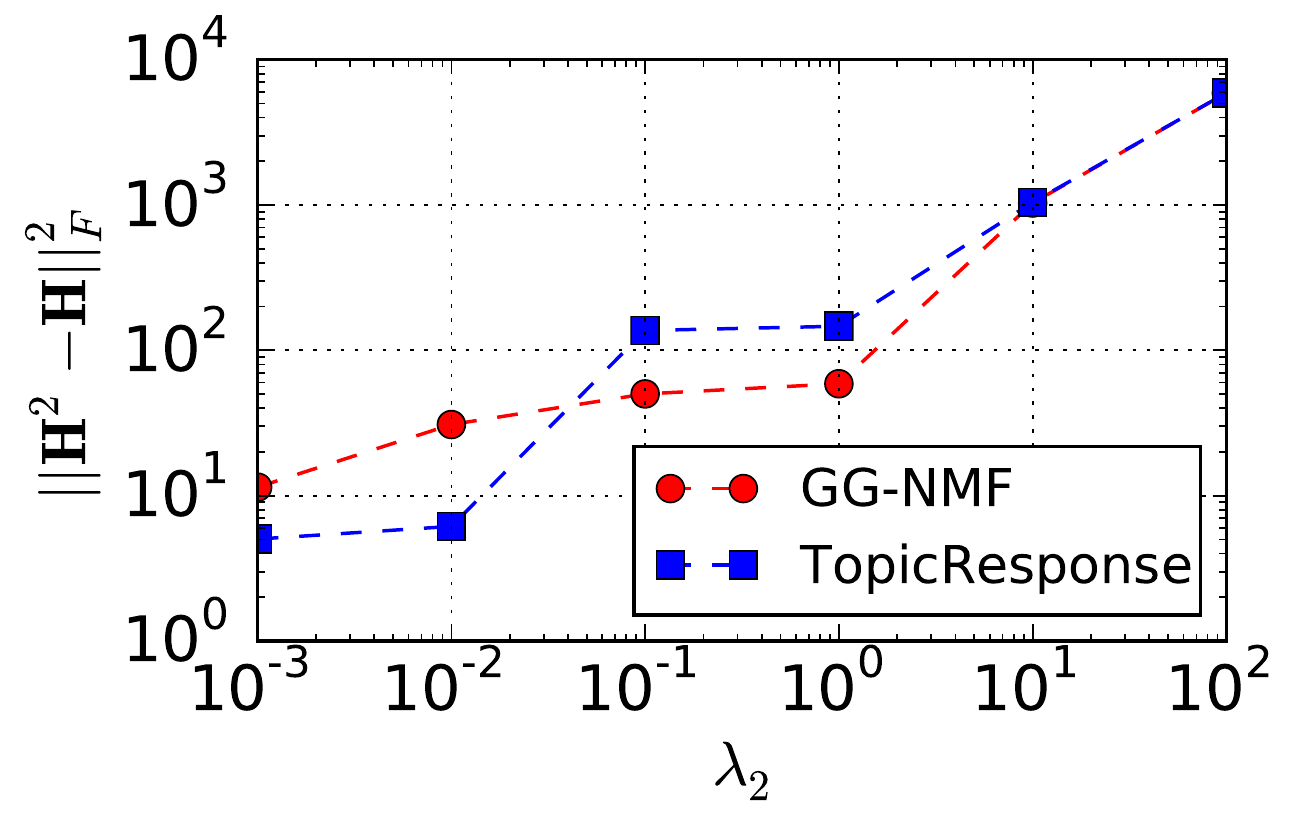}
            \caption{$\norm{\vbf{H}\circ\vbf{H}-\vbf{H}}$ on EDU}
            \label{fig:lambdahidealhbinary1}
        \end{subfigure}%
        ~
        \begin{subfigure}[t]{0.32\textwidth}
            \centering
            \includegraphics[scale=0.30]{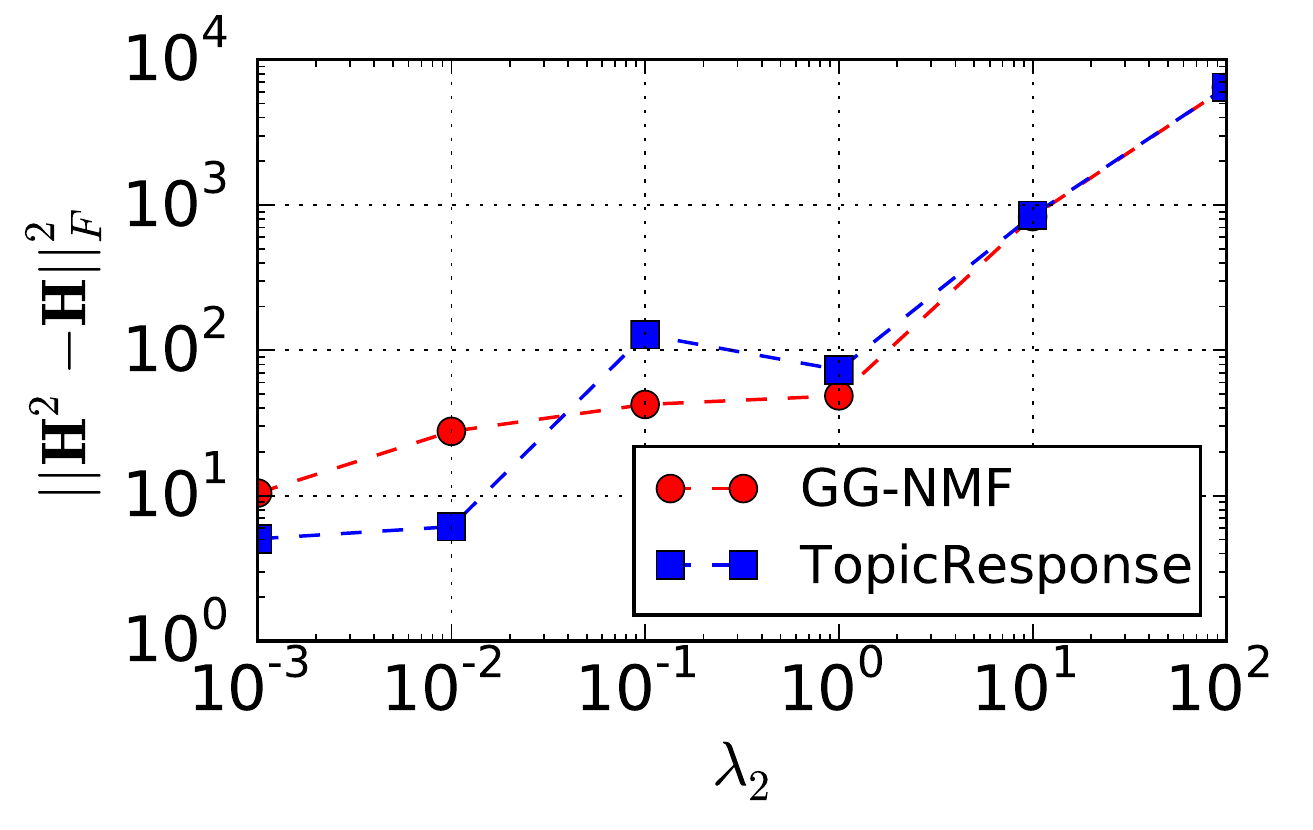}
            \caption{$\norm{\vbf{H}\circ\vbf{H}-\vbf{H}}$ on ECON}
            \label{fig:lambdahidealhbinary2}
        \end{subfigure}
        ~
        \begin{subfigure}[t]{0.32\textwidth}
            \centering
            \includegraphics[scale=0.30]{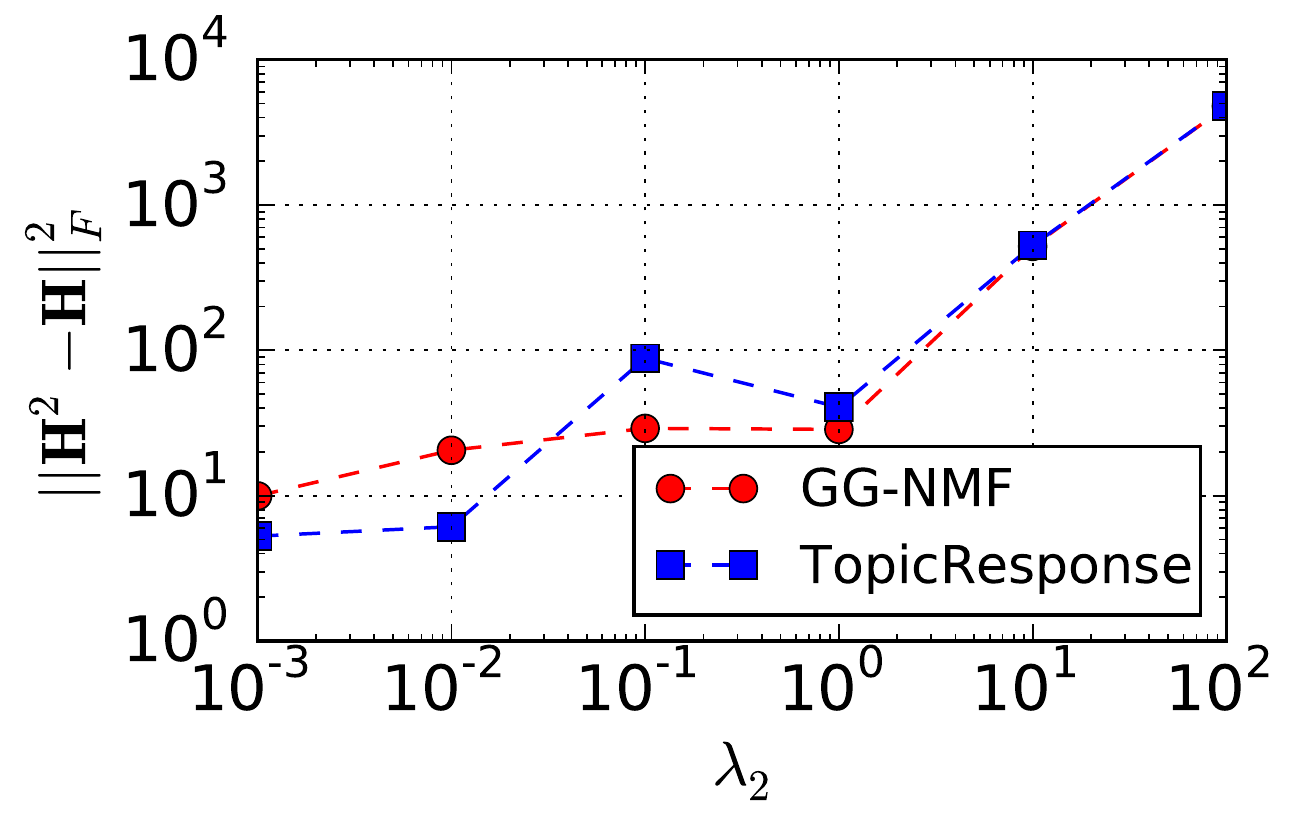}
            \caption{$\norm{\vbf{H}\circ\vbf{H}-\vbf{H}}$ on OPT}
            \label{fig:lambdahidealhbinary3}
        \end{subfigure}
        \caption{Performance of GG-NMF and TopicResponse with varying $\lambda_2$.}
        \label{fig:lambda1hideal}
    \end{figure}
    
    \begin{figure}[!htb]
        \centering
        \begin{subfigure}[t]{0.32\textwidth}
            \centering
            \includegraphics[scale=0.30]{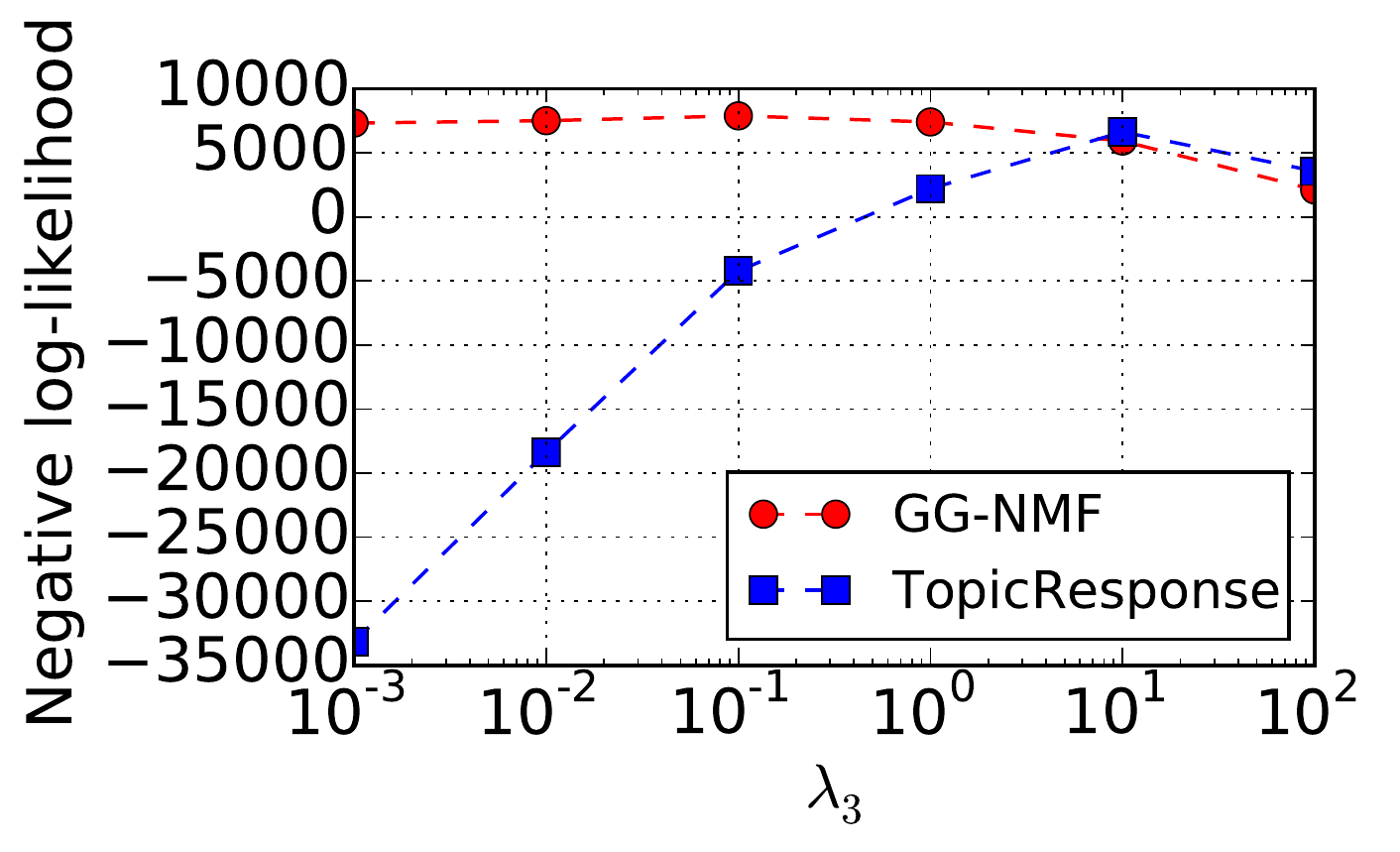}
            \caption{Likelihood on EDU}
        \end{subfigure}%
        ~
        \begin{subfigure}[t]{0.32\textwidth}
            \centering
            \includegraphics[scale=0.30]{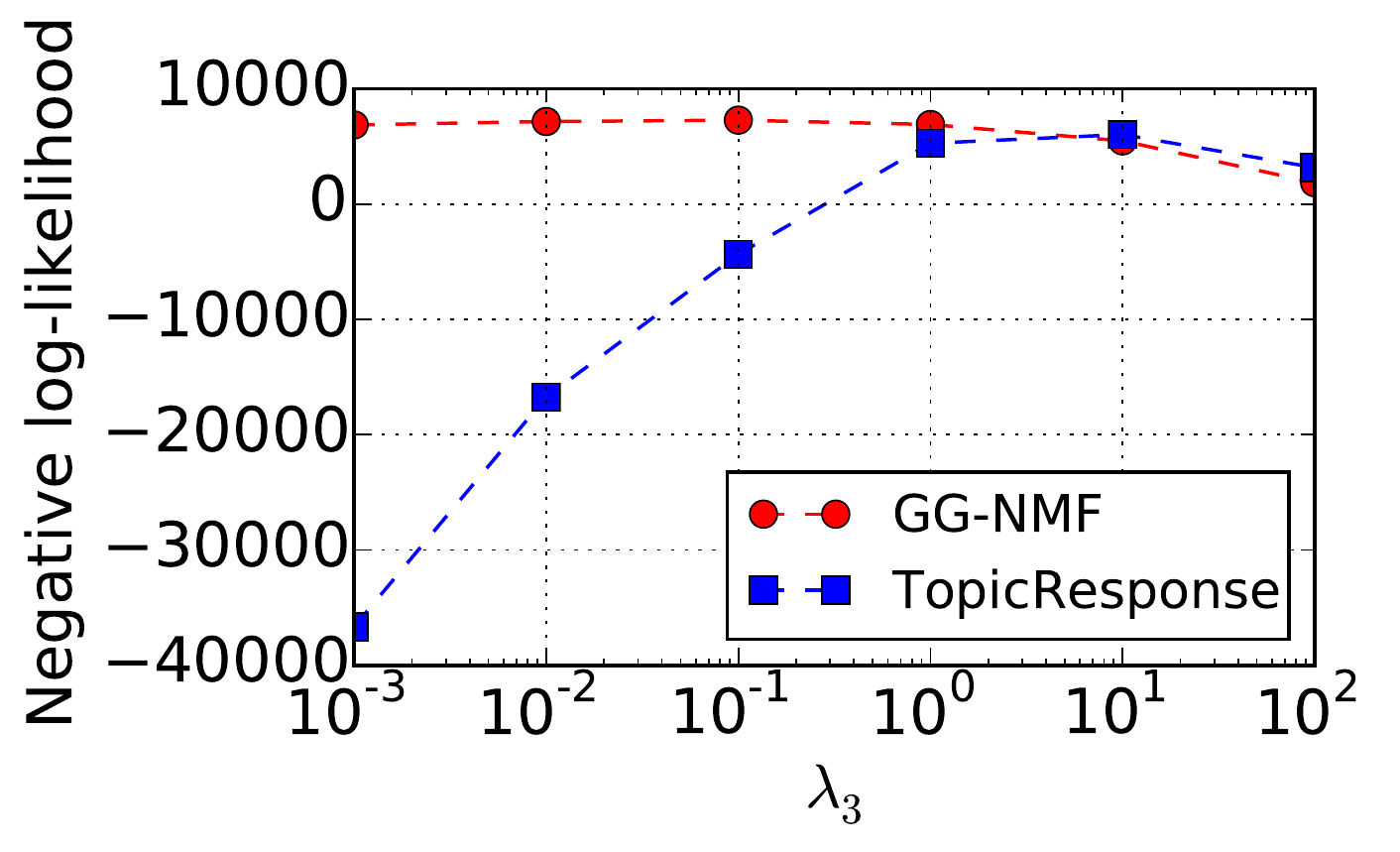}
            \caption{Likelihood on ECON}
        \end{subfigure}
        ~
        \begin{subfigure}[t]{0.32\textwidth}
            \centering
            \includegraphics[scale=0.30]{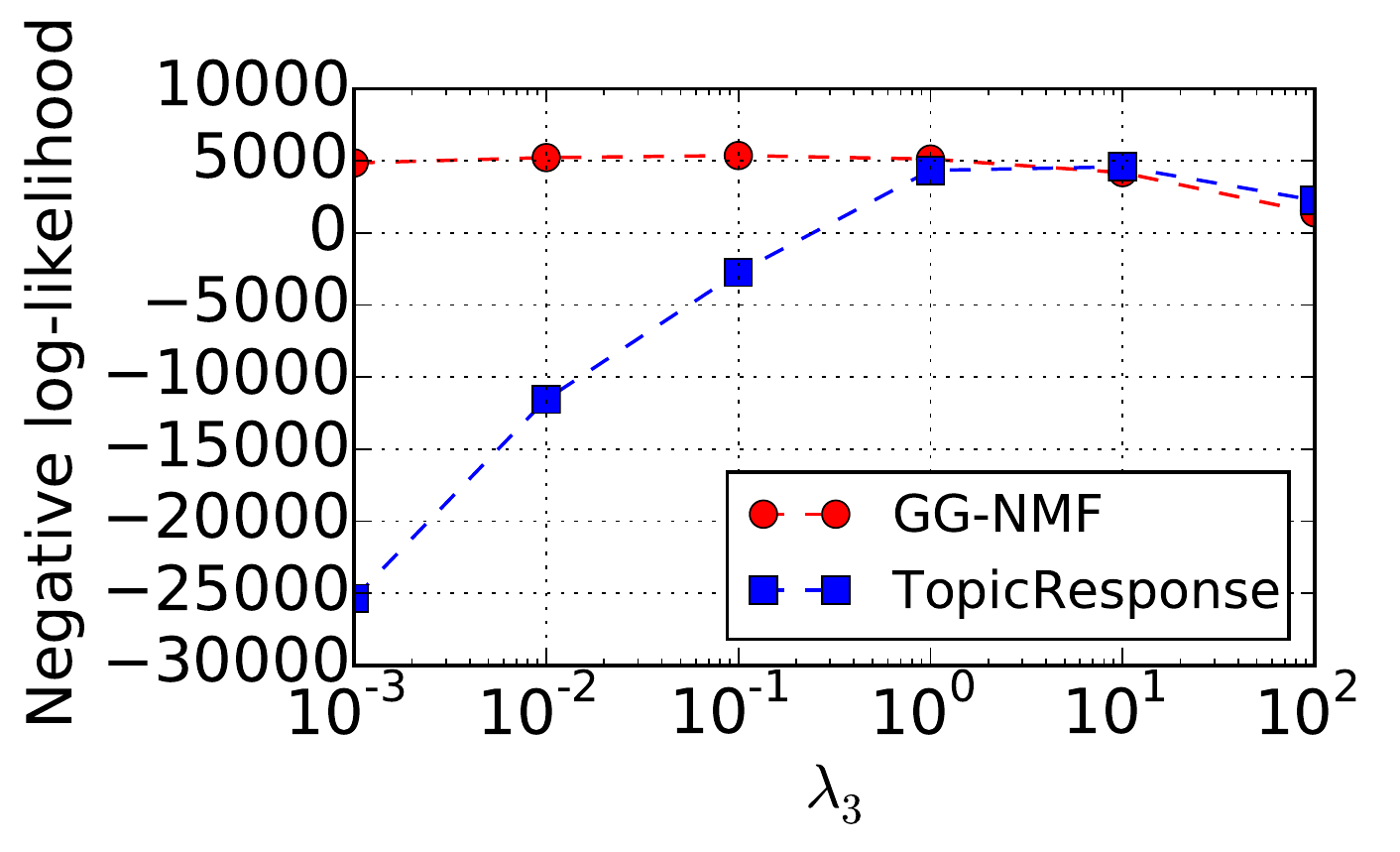}
            \caption{Likelihood on OPT}
        \end{subfigure}
        ~
        \begin{subfigure}[t]{0.32\textwidth}
            \centering
            \includegraphics[scale=0.30]{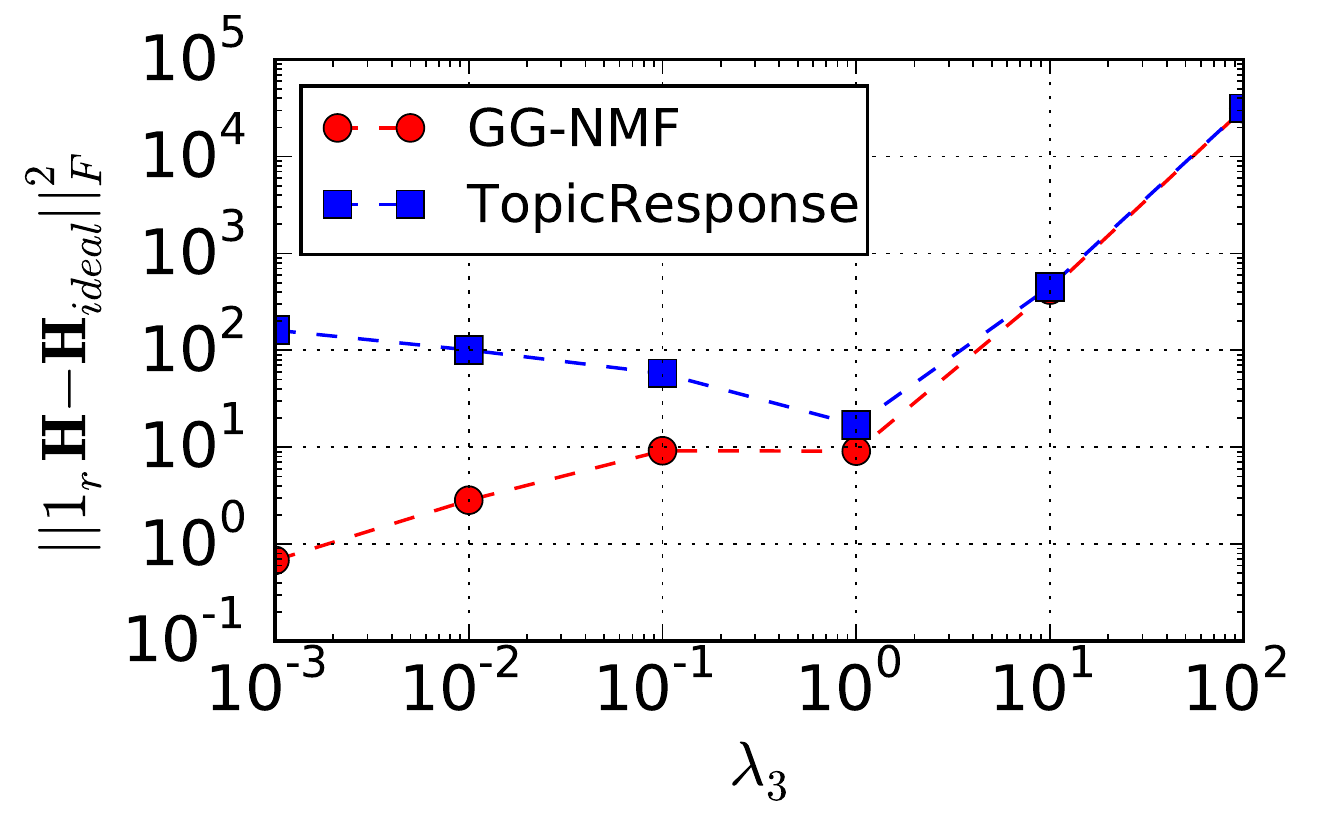}
            \caption{$\norm{\vbf{1}_r\vbf{H}-\vbf{H}_{ideal}}$ on EDU}
            \label{fig:lambdahbinaryhideal1}
        \end{subfigure}%
        ~
        \begin{subfigure}[t]{0.32\textwidth}
            \centering
            \includegraphics[scale=0.30]{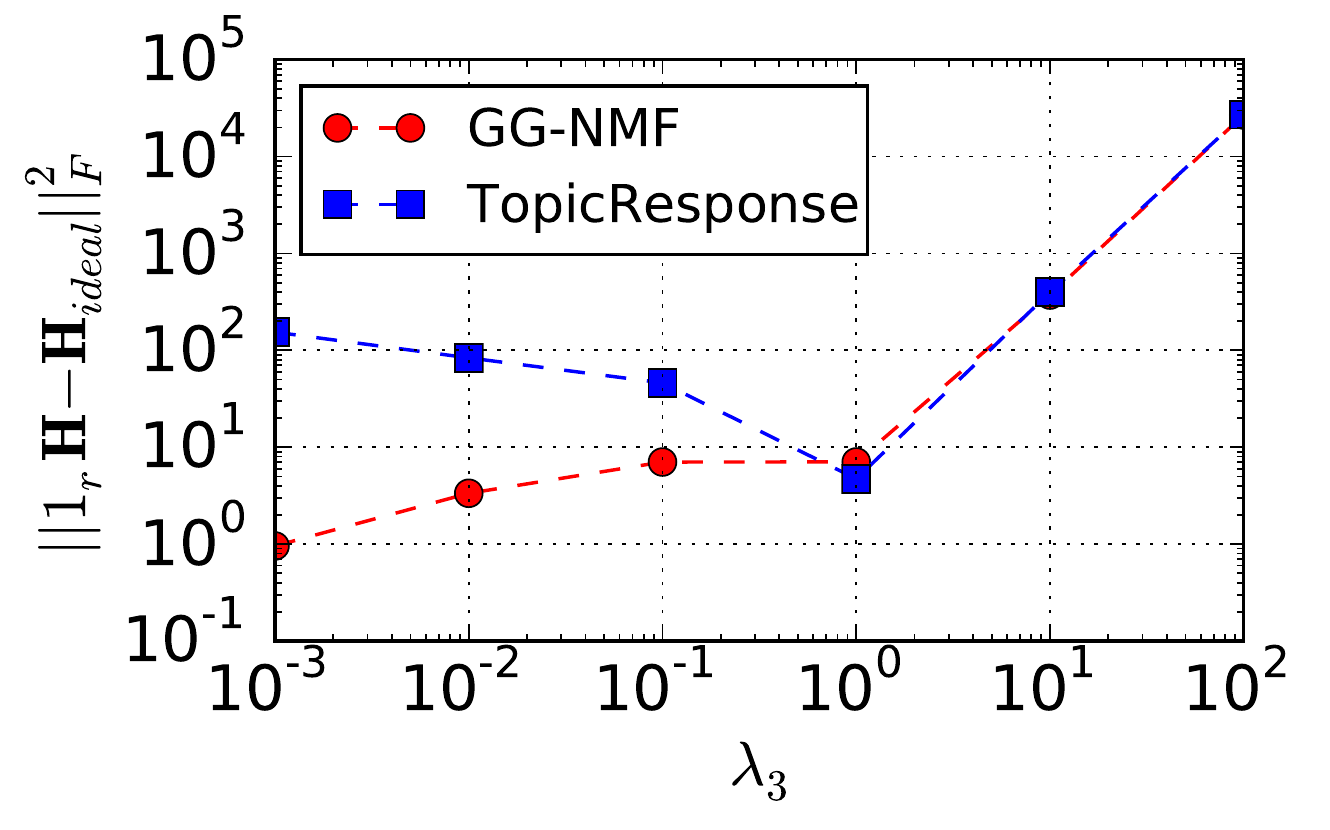}
            \caption{$\norm{\vbf{1}_r\vbf{H}-\vbf{H}_{ideal}}$ on ECON}
            \label{fig:lambdahbinaryhideal2}
        \end{subfigure}
        ~
        \begin{subfigure}[t]{0.32\textwidth}
            \centering
            \includegraphics[scale=0.30]{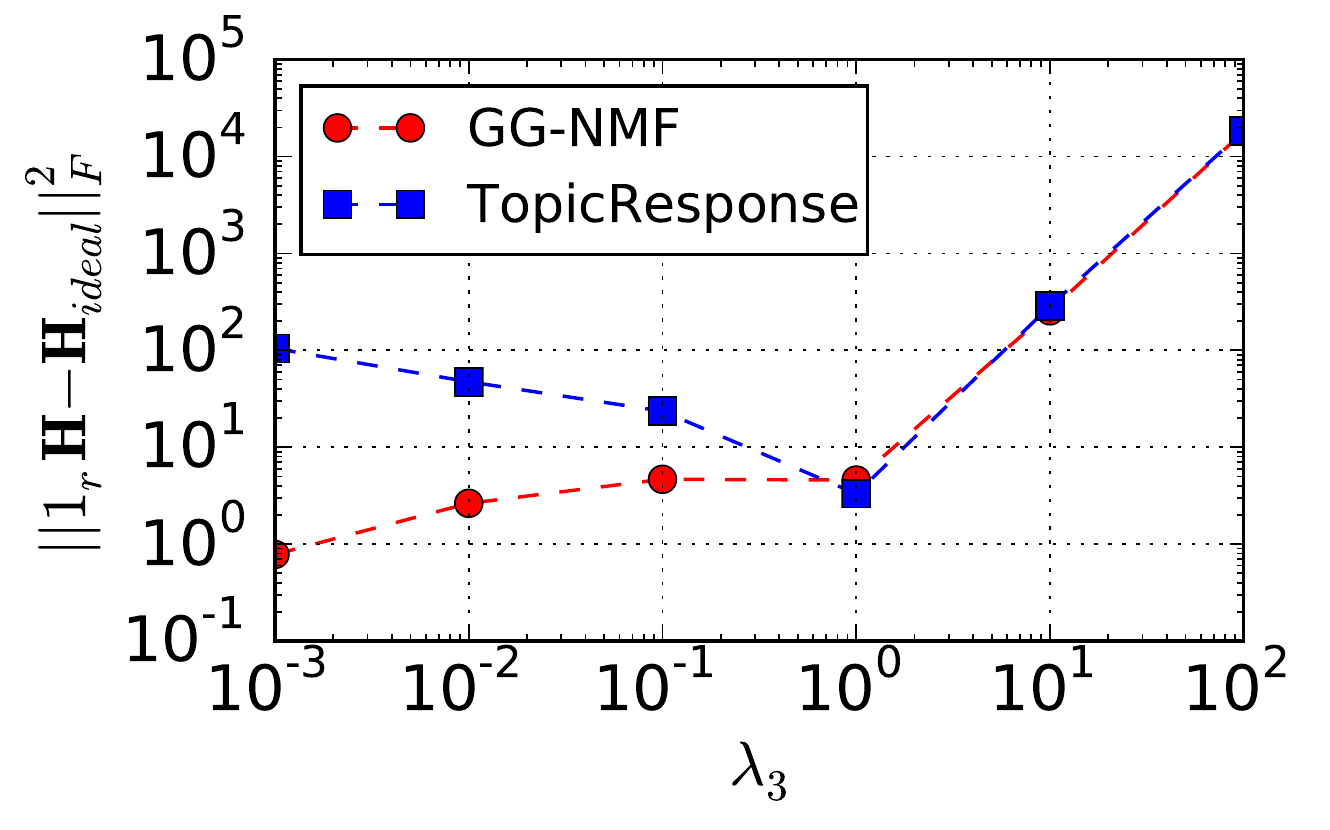}
            \caption{$\norm{\vbf{1}_r\vbf{H}-\vbf{H}_{ideal}}$ on OPT}
            \label{fig:lambdahbinaryhideal3}
        \end{subfigure}
        ~
        \begin{subfigure}[t]{0.32\textwidth}
            \centering
            \includegraphics[scale=0.30]{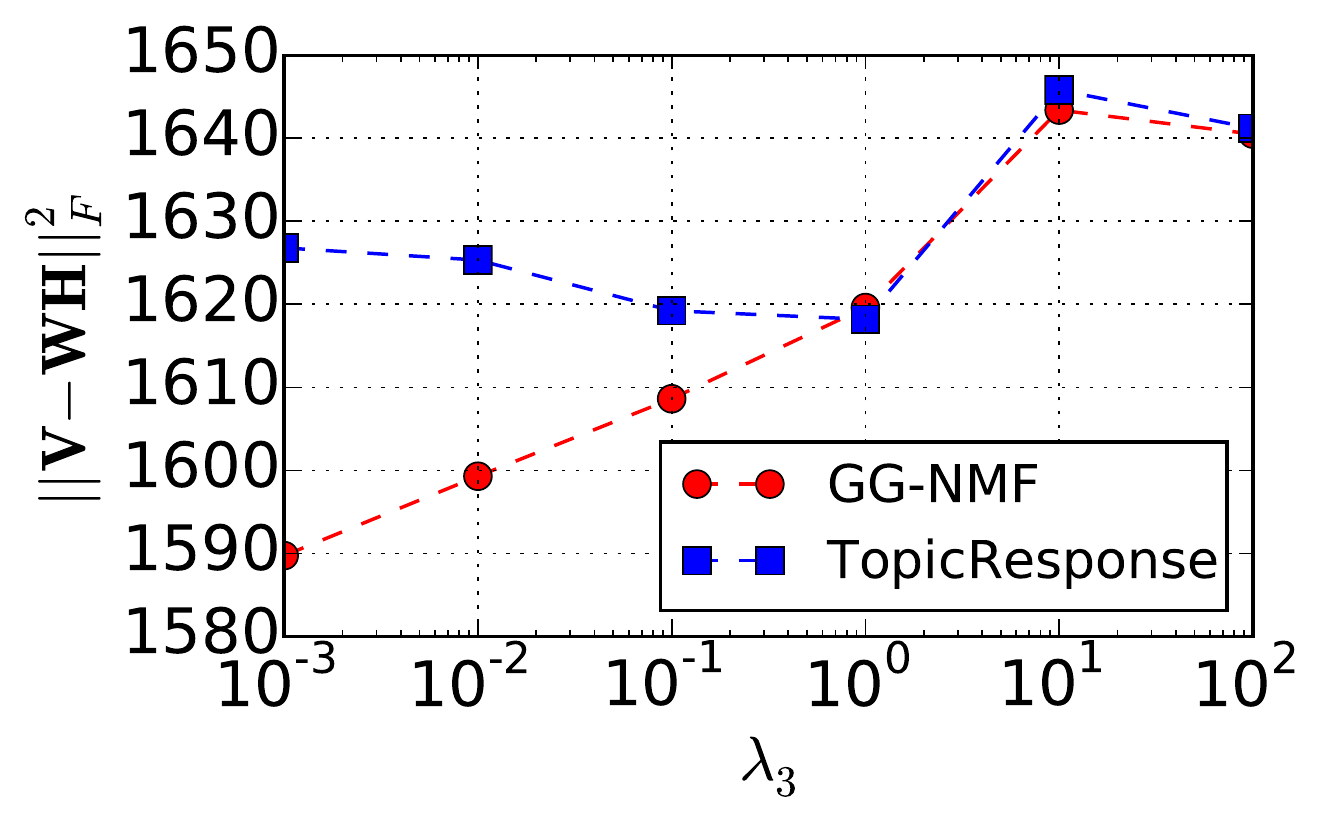}
            \caption{$\norm{\vbf{V}-\vbf{WH}}$ on EDU}
        \end{subfigure}%
        ~
        \begin{subfigure}[t]{0.32\textwidth}
            \centering
            \includegraphics[scale=0.30]{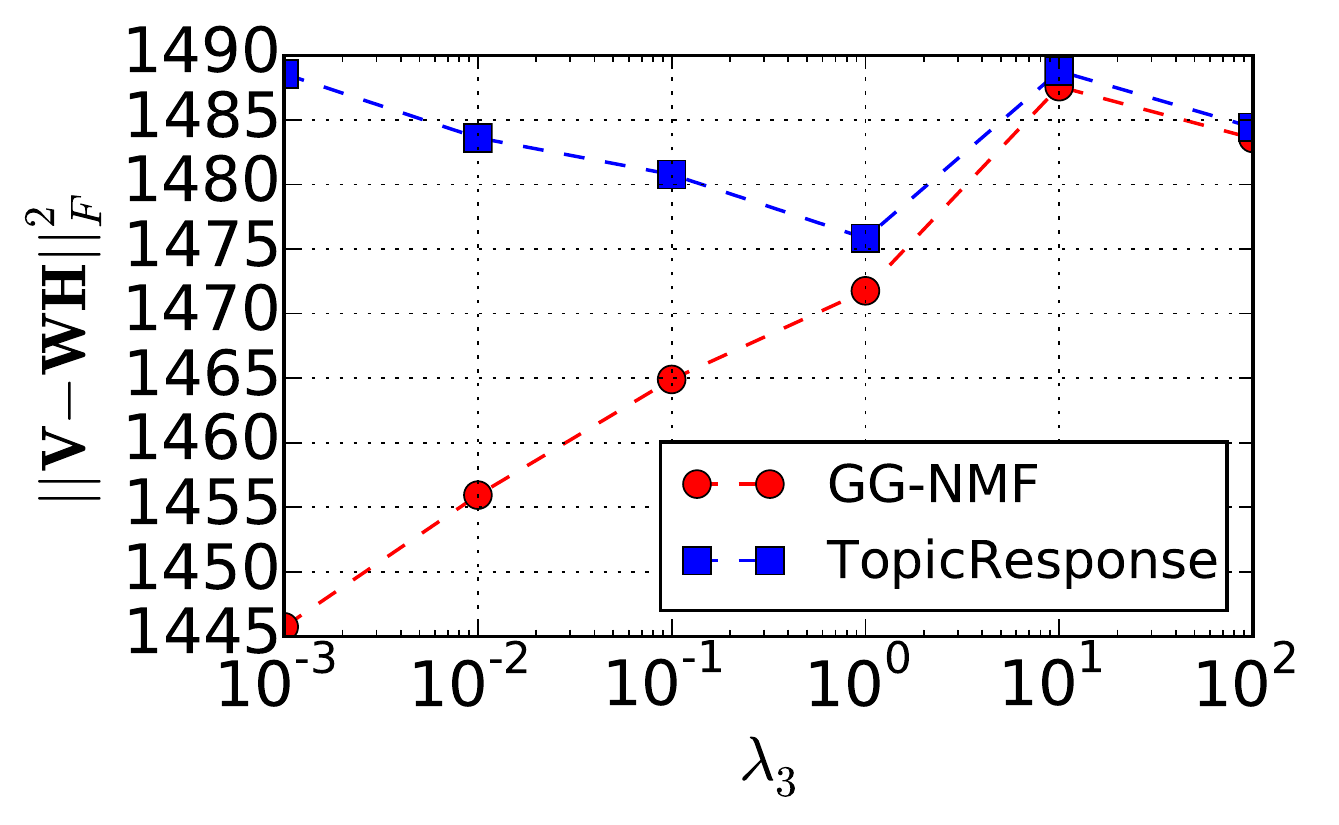}
            \caption{$\norm{\vbf{V}-\vbf{WH}}$ on ECON}
        \end{subfigure}
        ~
        \begin{subfigure}[t]{0.32\textwidth}
            \centering
            \includegraphics[scale=0.30]{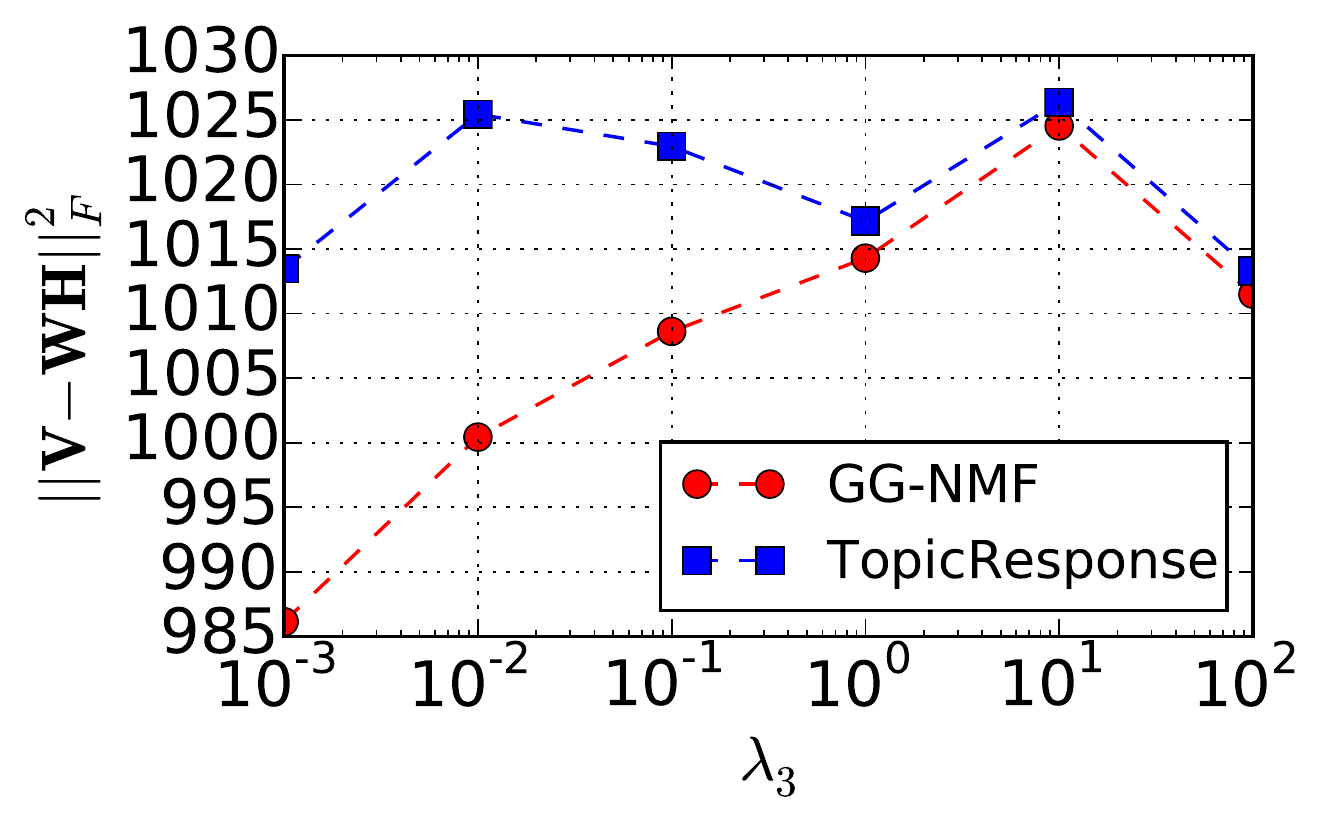}
            \caption{$\norm{\vbf{V}-\vbf{WH}}$ on OPT}
        \end{subfigure}
        ~
        \begin{subfigure}[t]{0.32\textwidth}
            \centering
            \includegraphics[scale=0.30]{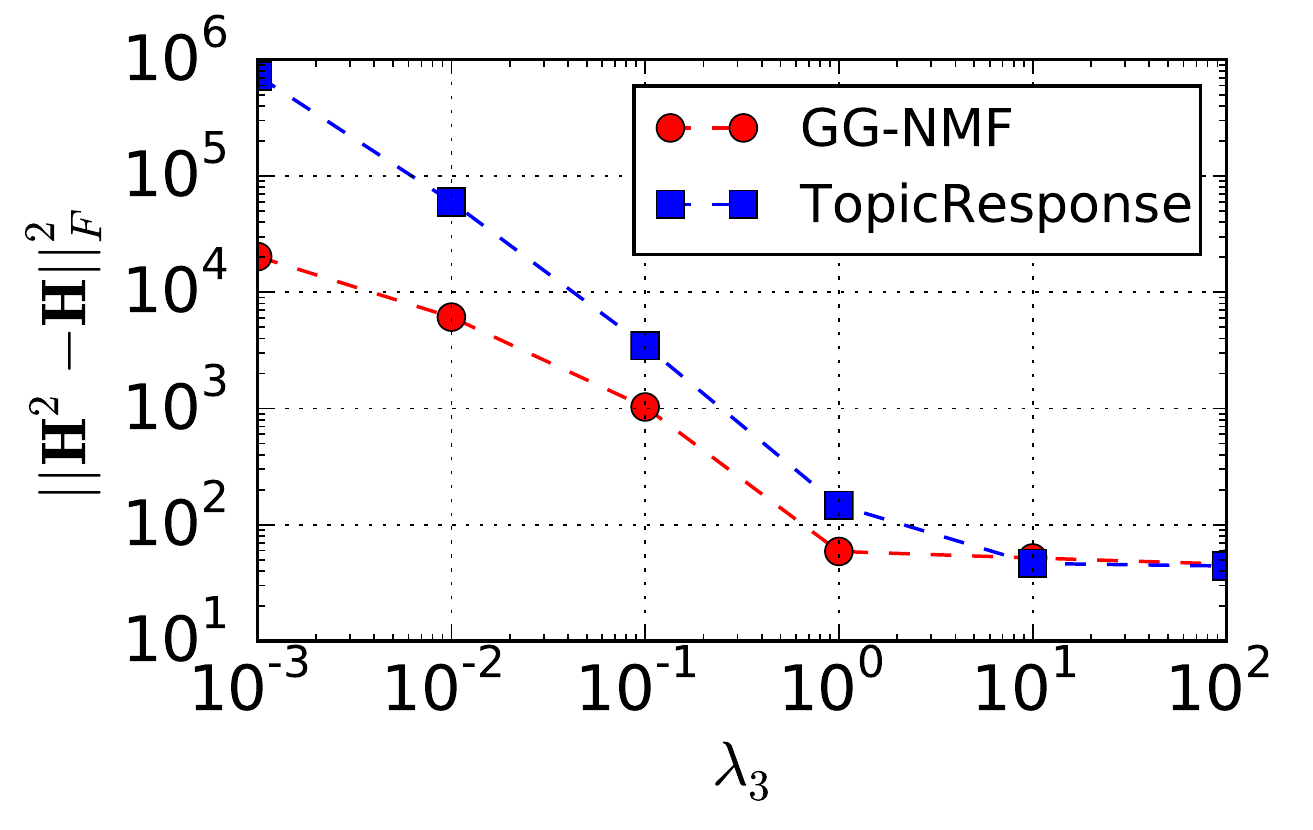}
            \caption{$\norm{\vbf{H}\circ\vbf{H}-\vbf{H}}$ on EDU}
            \label{fig:lambdahbinaryhbinary1}
        \end{subfigure}%
        ~
        \begin{subfigure}[t]{0.32\textwidth}
            \centering
            \includegraphics[scale=0.30]{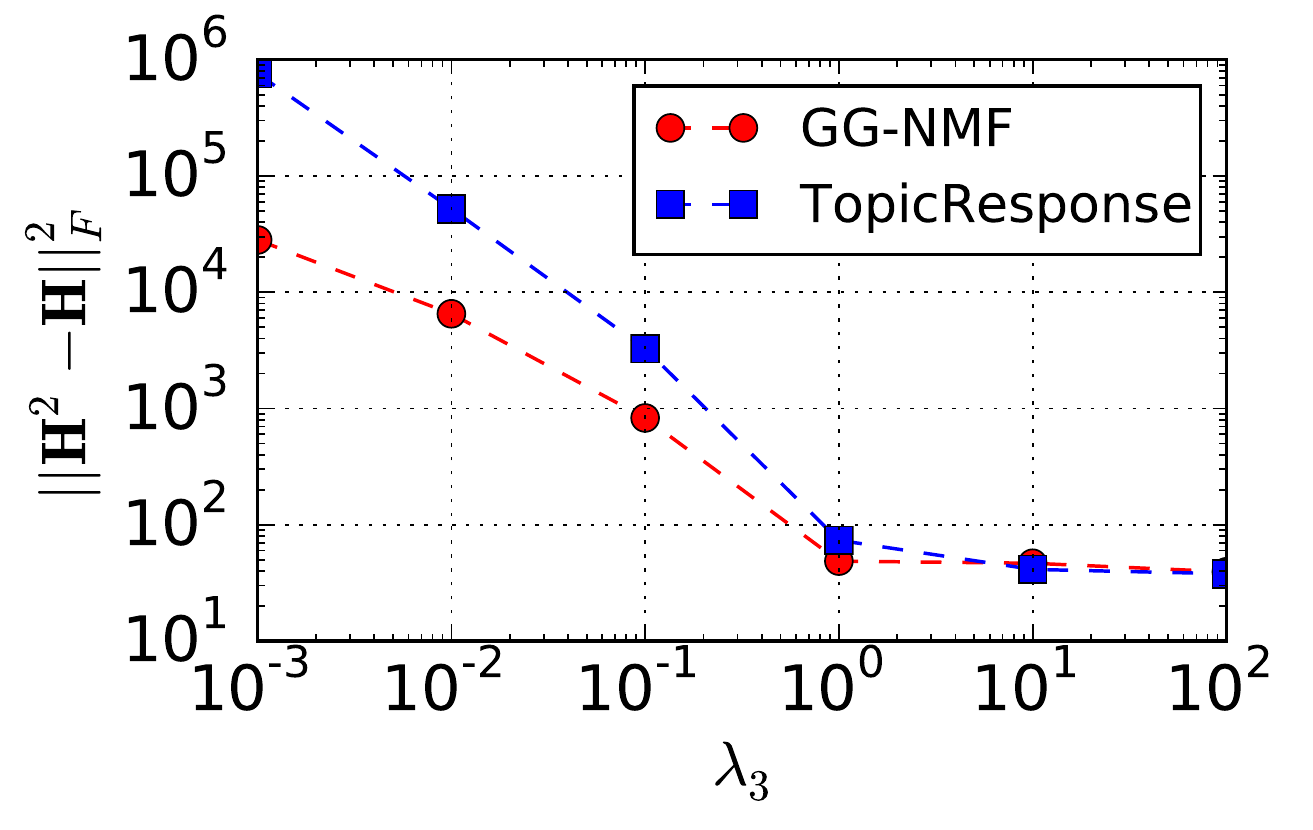}
            \caption{$\norm{\vbf{H}\circ\vbf{H}-\vbf{H}}$ on ECON}
            \label{fig:lambdahbinaryhbinary2}
        \end{subfigure}
        ~
        \begin{subfigure}[t]{0.32\textwidth}
            \centering
            \includegraphics[scale=0.30]{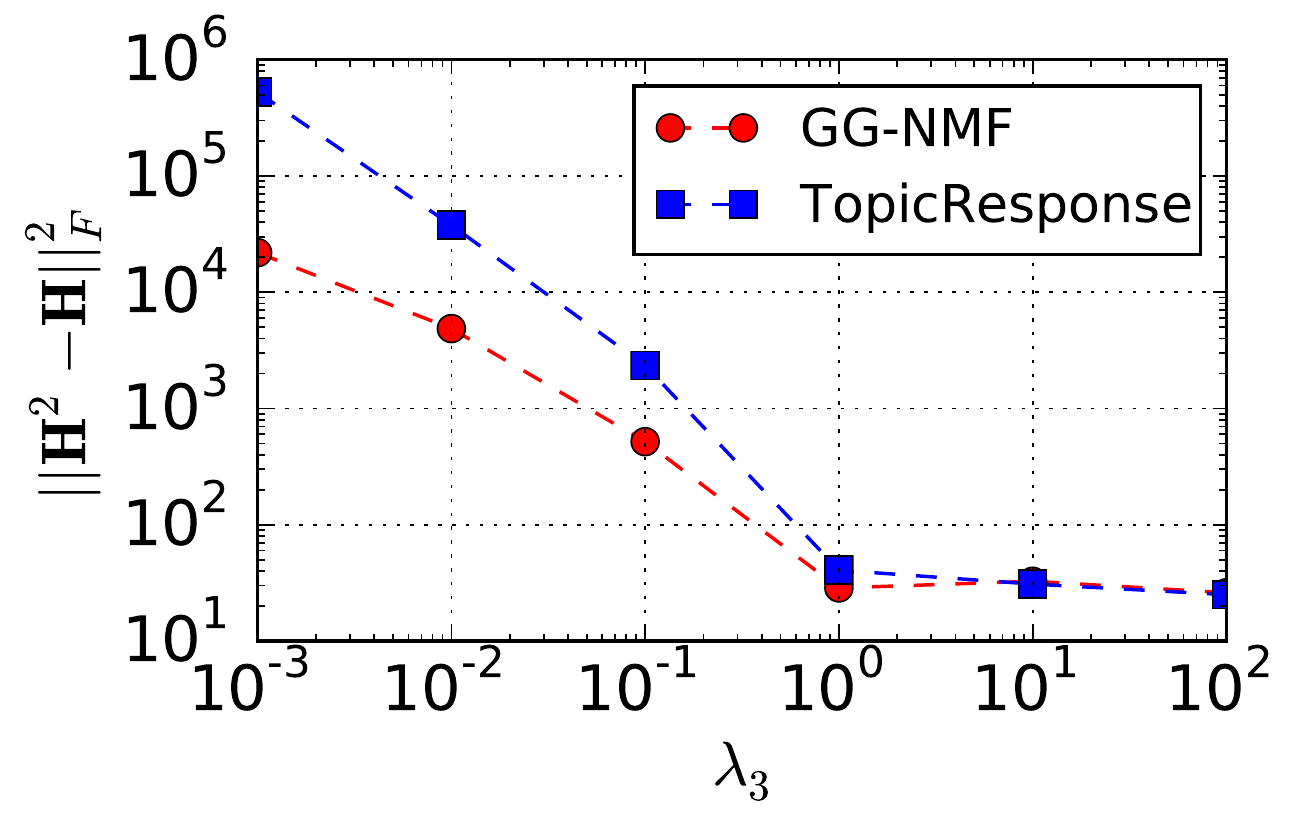}
            \caption{$\norm{\vbf{H}\circ\vbf{H}-\vbf{H}}$ on OPT}
            \label{fig:lambdahbinaryhbinary3}
        \end{subfigure}
        \caption{Performance of GG-NMF and TopicResponse with varying $\lambda_3$.}
        \label{fig:lambda2hbinary}
    \end{figure}
    
    \begin{figure}[!htb]
        \centering
        \begin{subfigure}[t]{0.32\textwidth}
            \centering
            \includegraphics[scale=0.30]{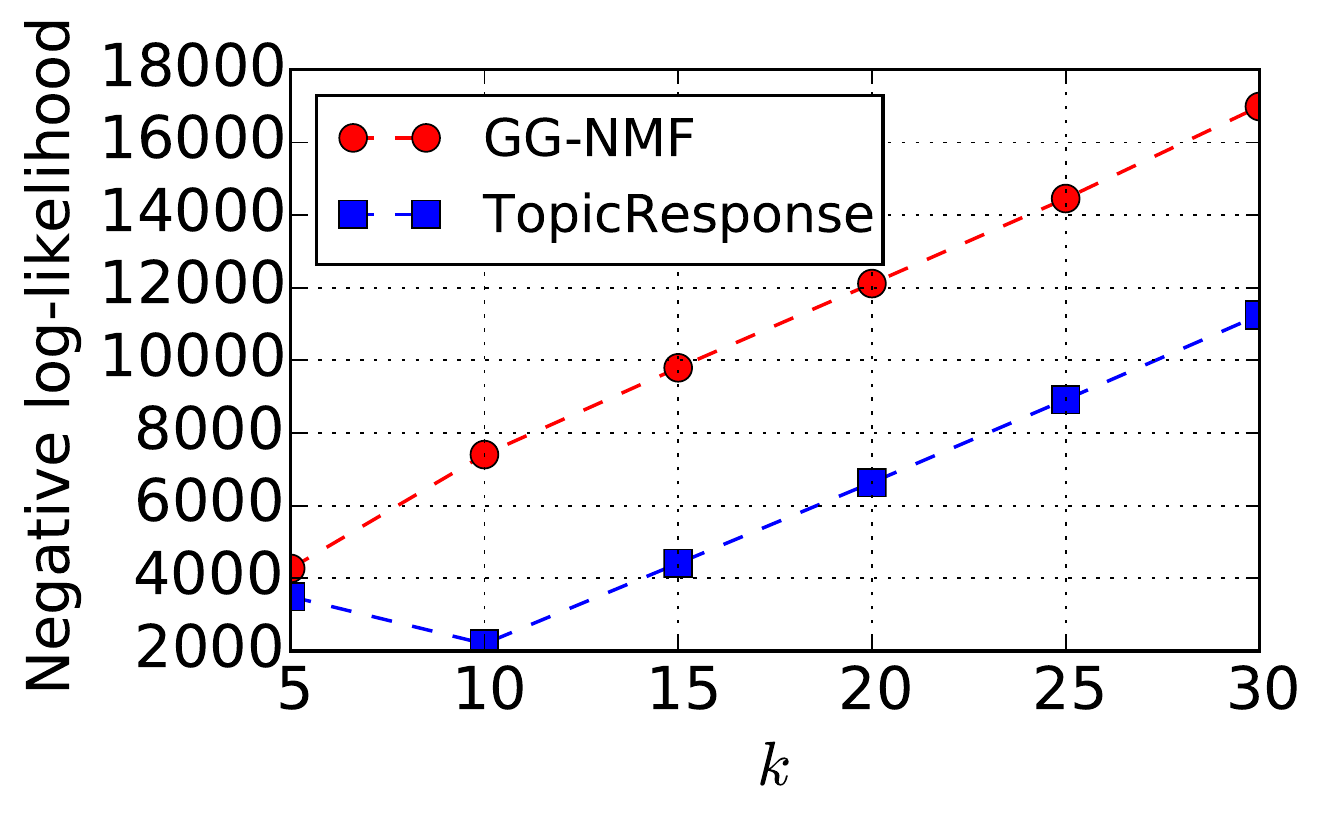}
            \caption{Negative log-ikelihood on EDU}
        \end{subfigure}%
        ~
        \begin{subfigure}[t]{0.32\textwidth}
            \centering
            \includegraphics[scale=0.30]{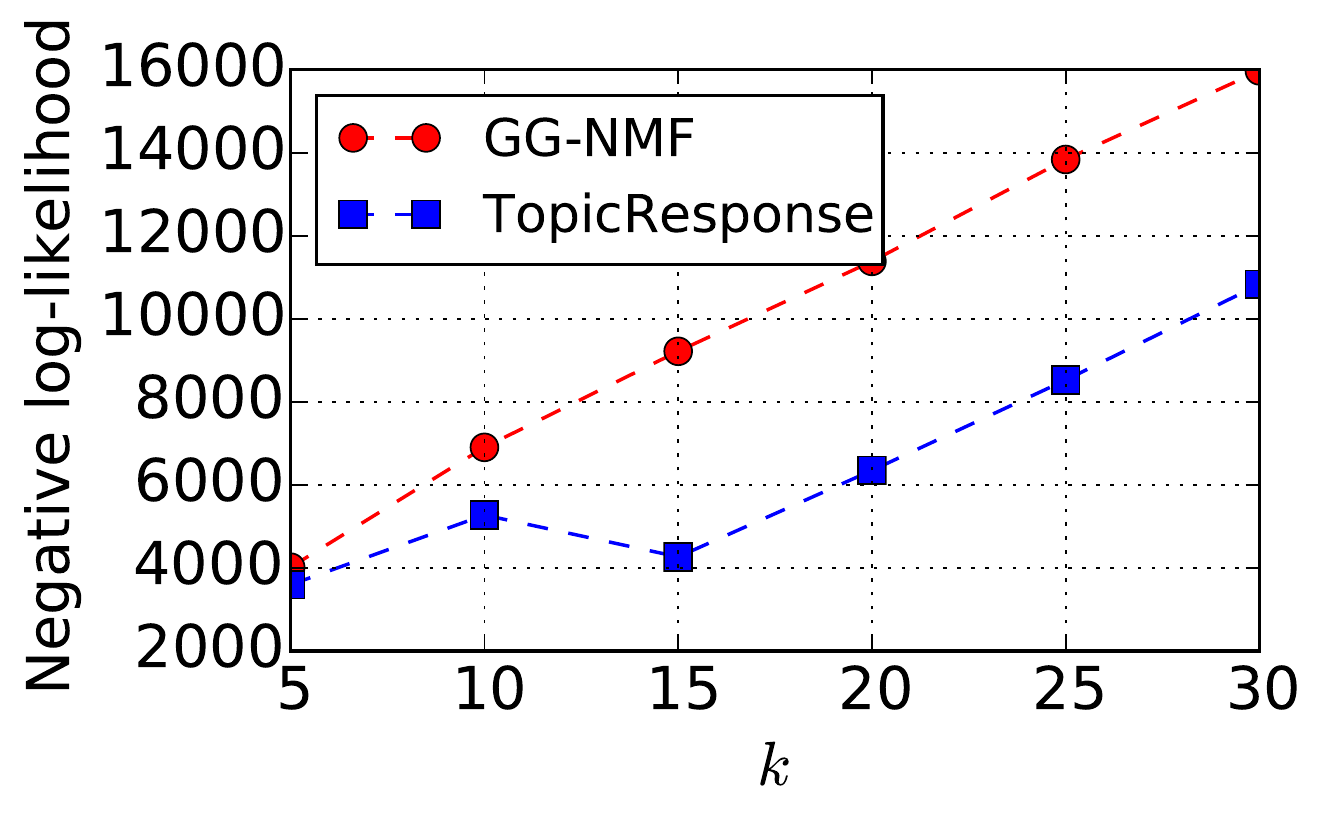}
            \caption{Negative log-ikelihood on ECON}
        \end{subfigure}
        ~
        \begin{subfigure}[t]{0.32\textwidth}
            \centering
            \includegraphics[scale=0.30]{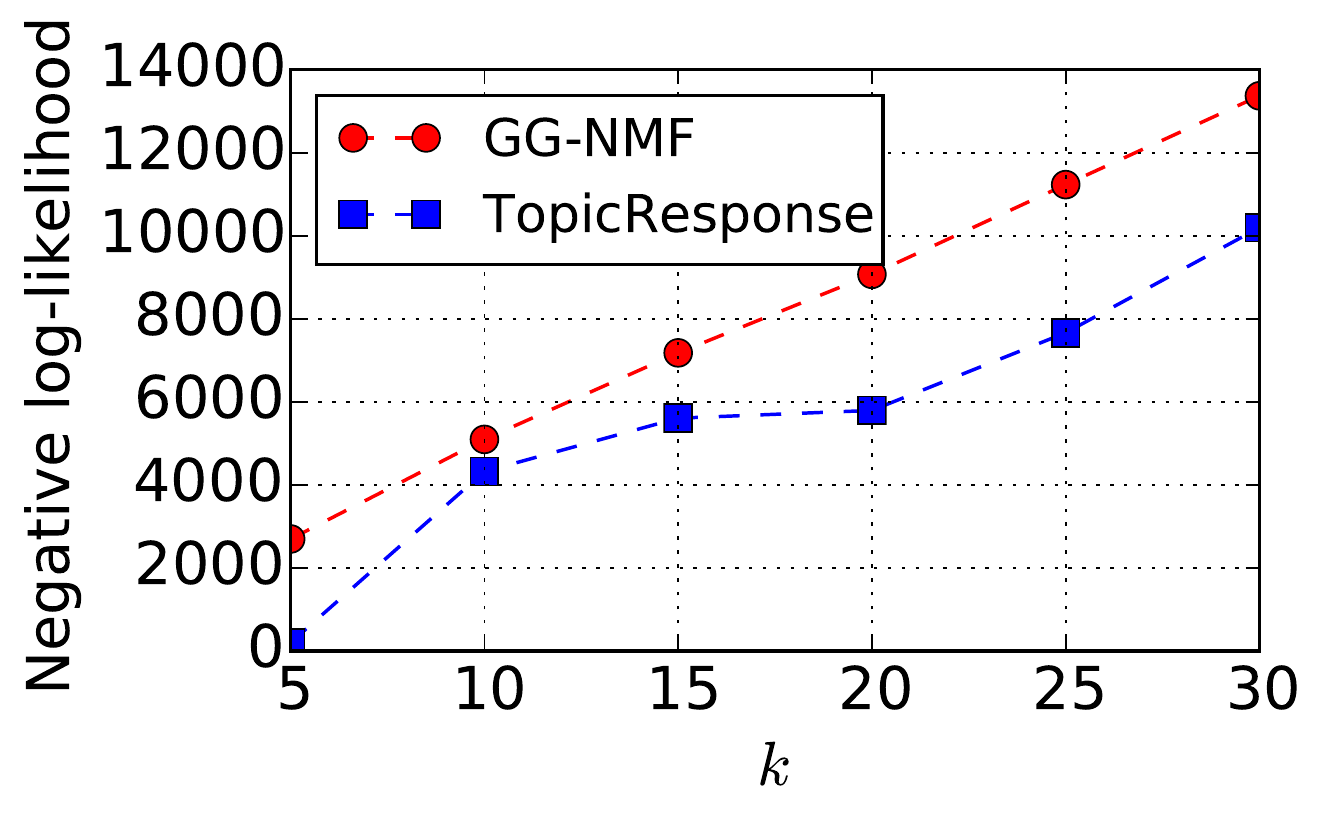}
            \caption{Negative log-ikelihood on OPT}
        \end{subfigure}
        ~
        \begin{subfigure}[t]{0.32\textwidth}
            \centering
            \includegraphics[scale=0.30]{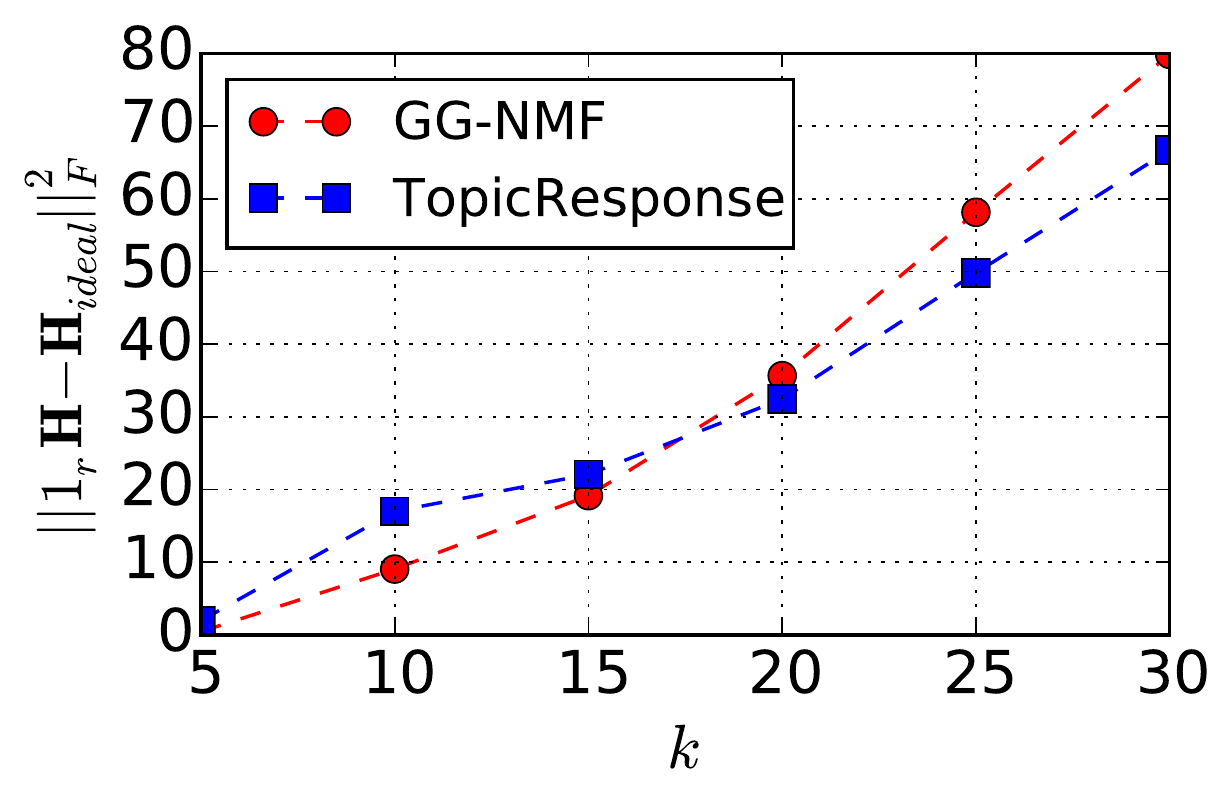}
            \caption{$\norm{\vbf{1}_r\vbf{H}-\vbf{H}_{ideal}}$ on EDU}
        \end{subfigure}%
        ~
        \begin{subfigure}[t]{0.32\textwidth}
            \centering
            \includegraphics[scale=0.30]{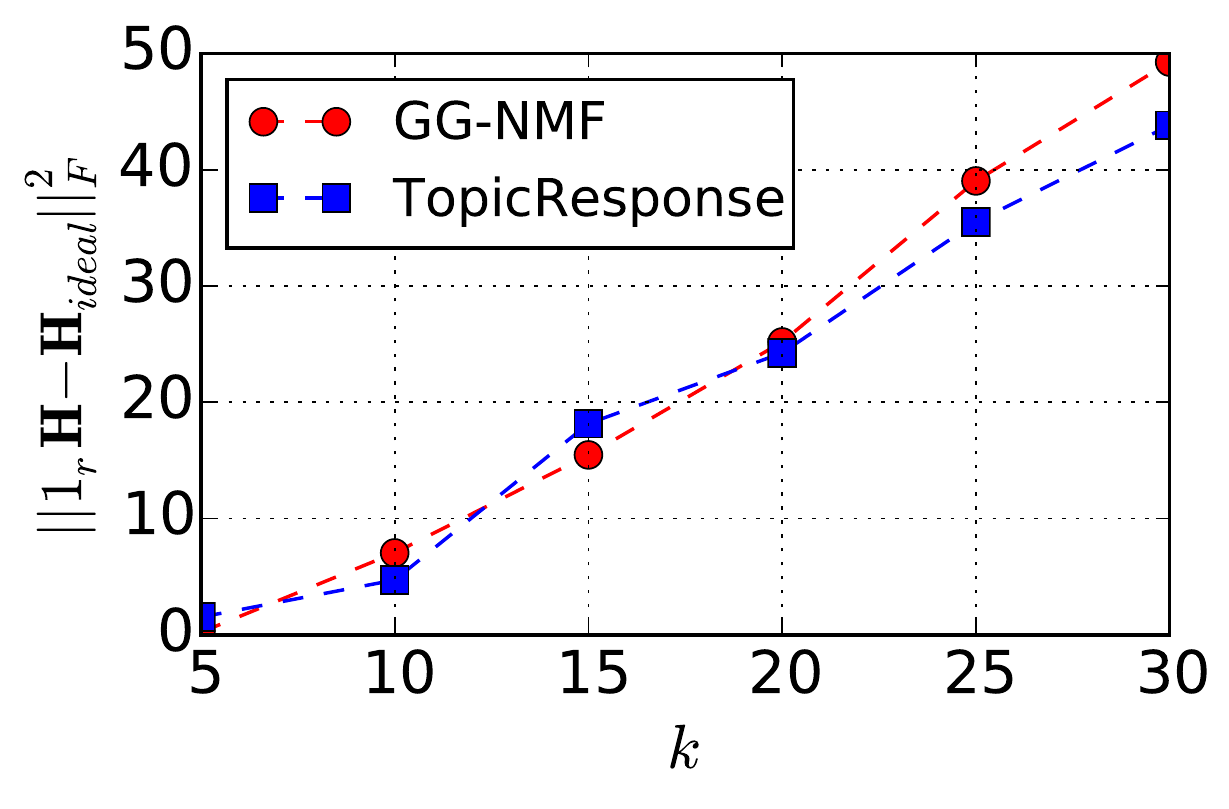}
            \caption{$\norm{\vbf{1}_r\vbf{H}-\vbf{H}_{ideal}}$ on ECON}
        \end{subfigure}
        ~
        \begin{subfigure}[t]{0.32\textwidth}
            \centering
            \includegraphics[scale=0.30]{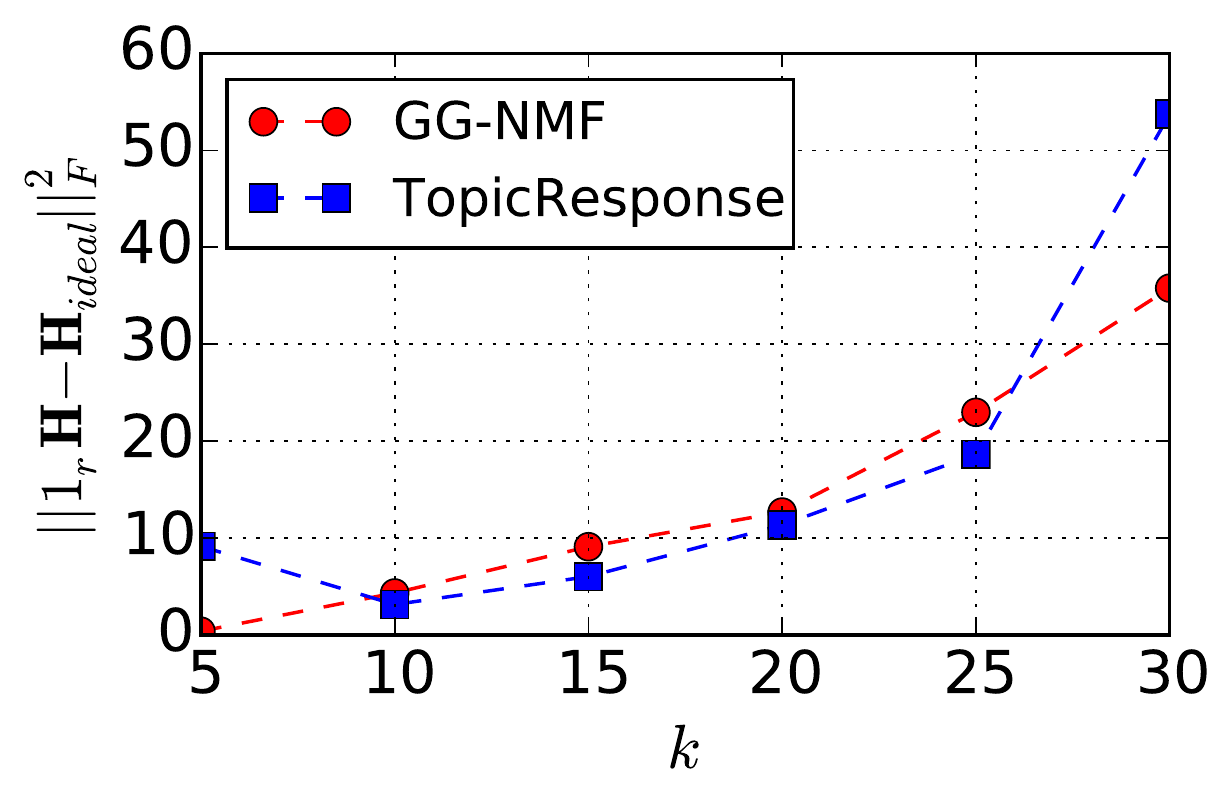}
            \caption{$\norm{\vbf{1}_r\vbf{H}-\vbf{H}_{ideal}}$ on OPT}
        \end{subfigure}
        ~
        \begin{subfigure}[t]{0.32\textwidth}
            \centering
            \includegraphics[scale=0.30]{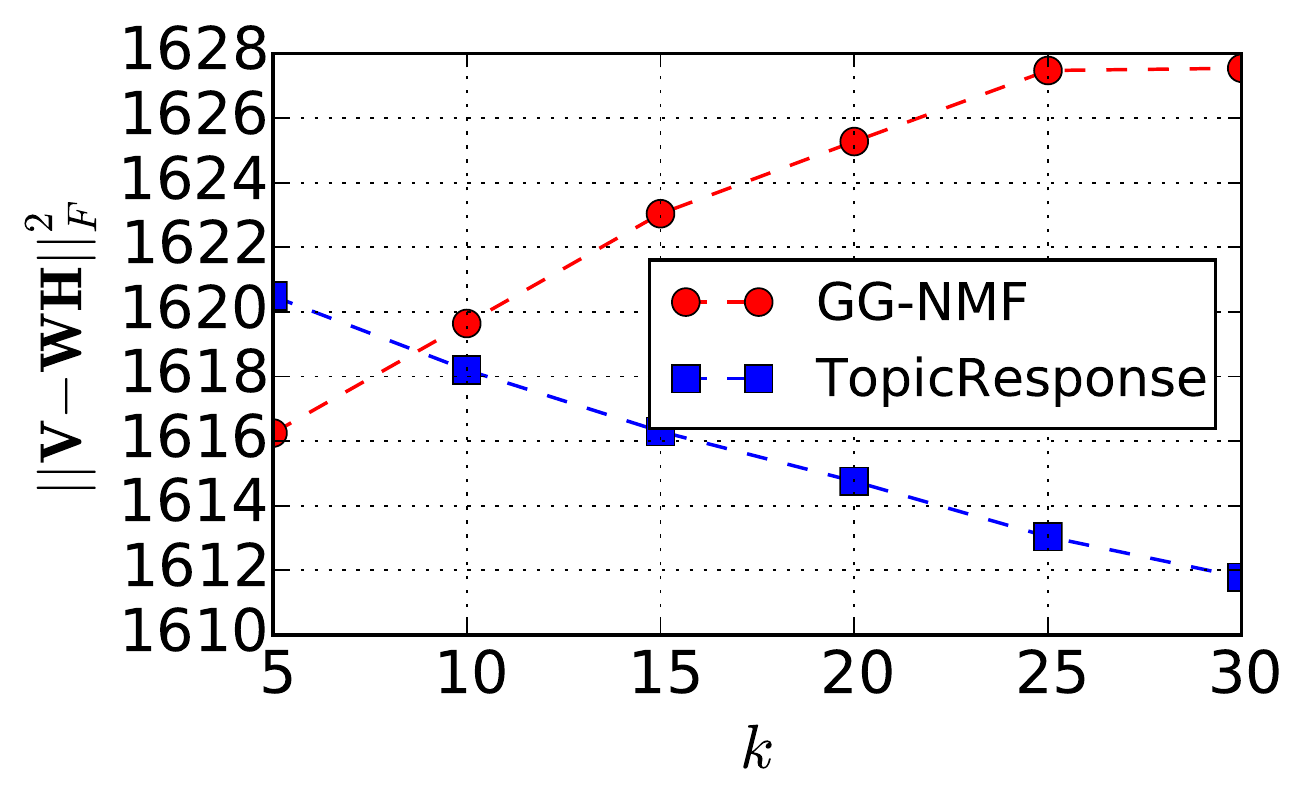}
            \caption{$\norm{\vbf{V}-\vbf{WH}}$ on EDU}
        \end{subfigure}%
        ~
        \begin{subfigure}[t]{0.32\textwidth}
            \centering
            \includegraphics[scale=0.30]{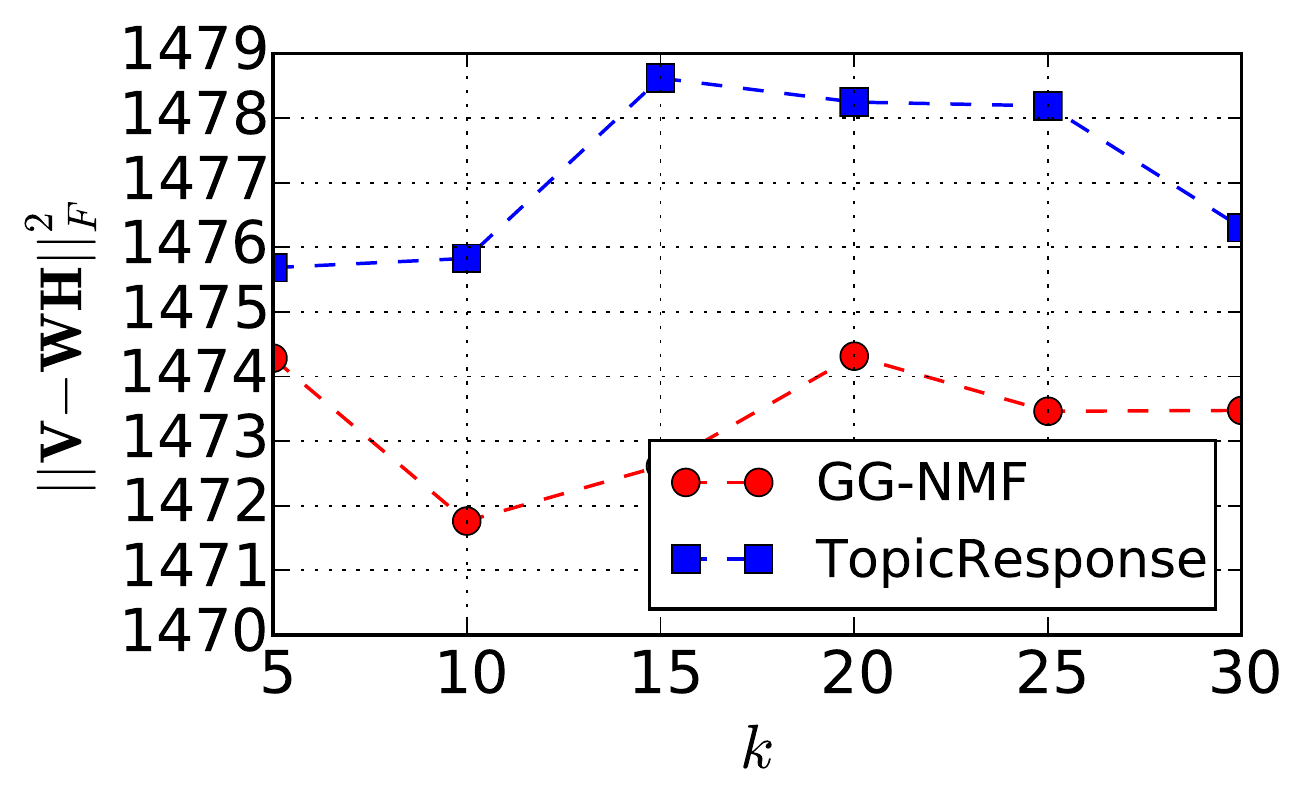}
            \caption{$\norm{\vbf{V}-\vbf{WH}}$ on ECON}
        \end{subfigure}
        ~
        \begin{subfigure}[t]{0.32\textwidth}
            \centering
            \includegraphics[scale=0.30]{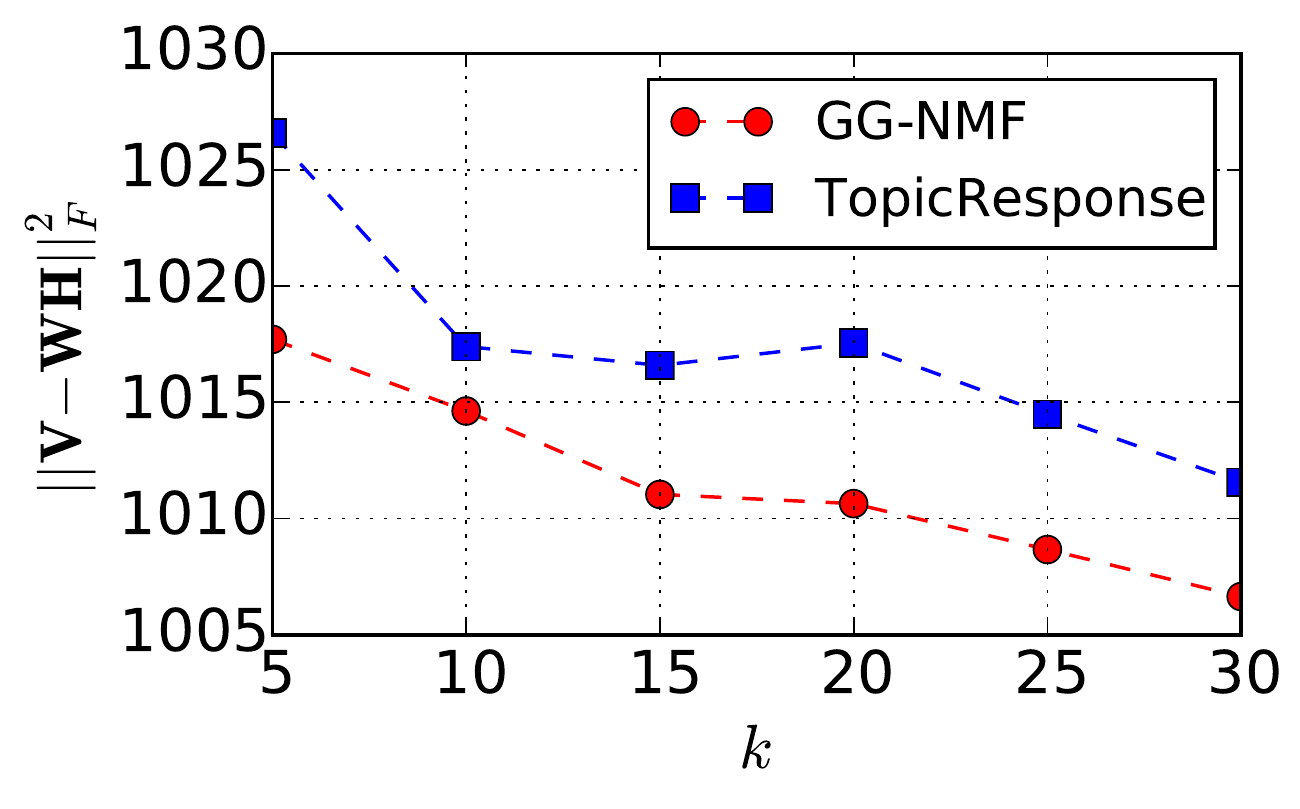}
            \caption{$\norm{\vbf{V}-\vbf{WH}}$ on OPT}
        \end{subfigure}
        ~
        \begin{subfigure}[t]{0.32\textwidth}
            \centering
            \includegraphics[scale=0.30]{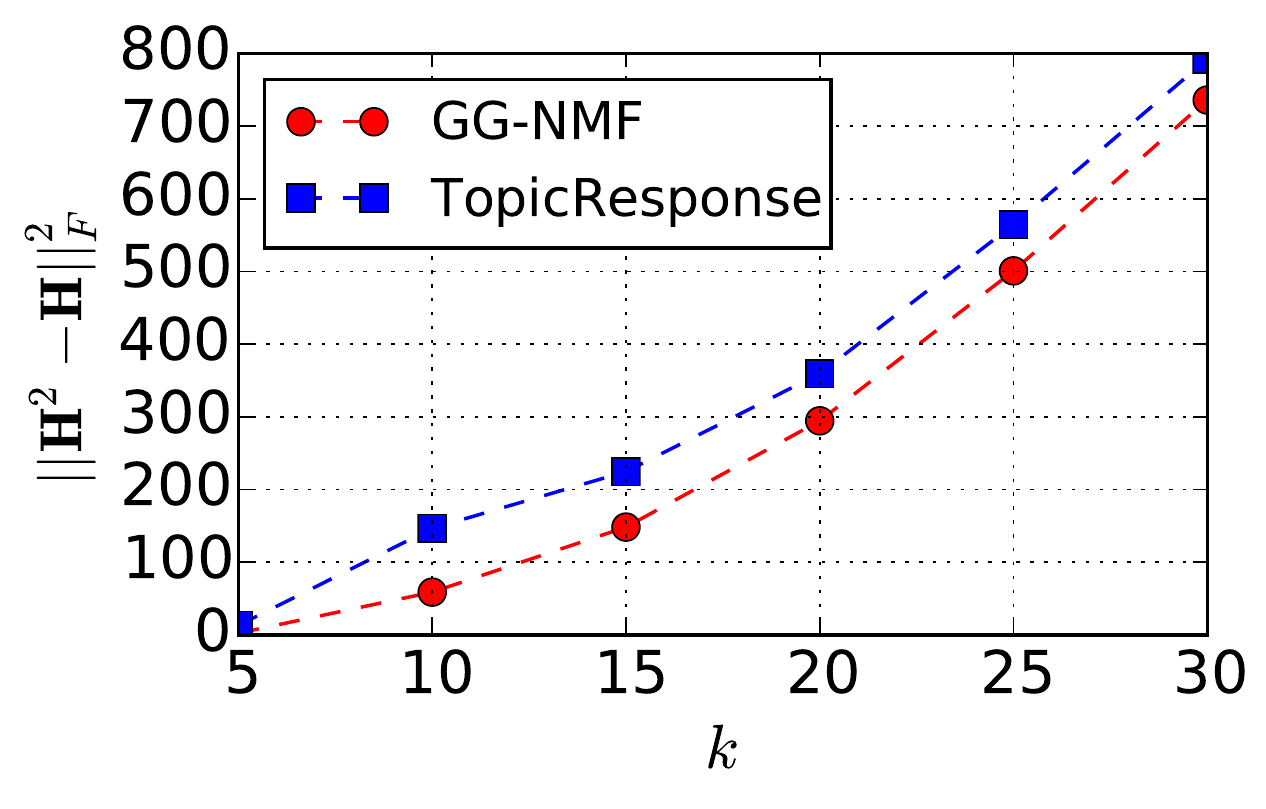}
            \caption{$\norm{\vbf{H}\circ\vbf{H}-\vbf{H}}$ on EDU}
        \end{subfigure}%
        ~
        \begin{subfigure}[t]{0.32\textwidth}
            \centering
            \includegraphics[scale=0.30]{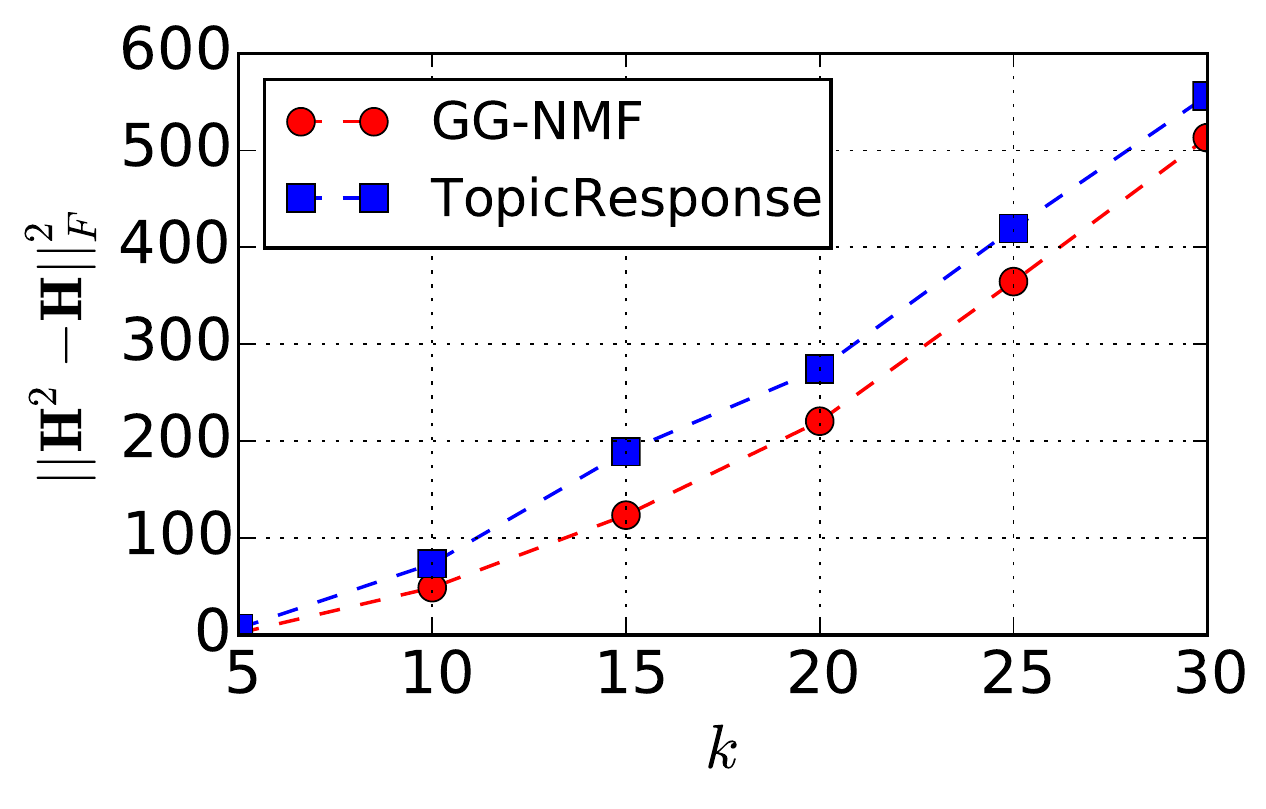}
            \caption{$\norm{\vbf{H}\circ\vbf{H}-\vbf{H}}$ on ECON}
        \end{subfigure}
        ~
        \begin{subfigure}[t]{0.32\textwidth}
            \centering
            \includegraphics[scale=0.30]{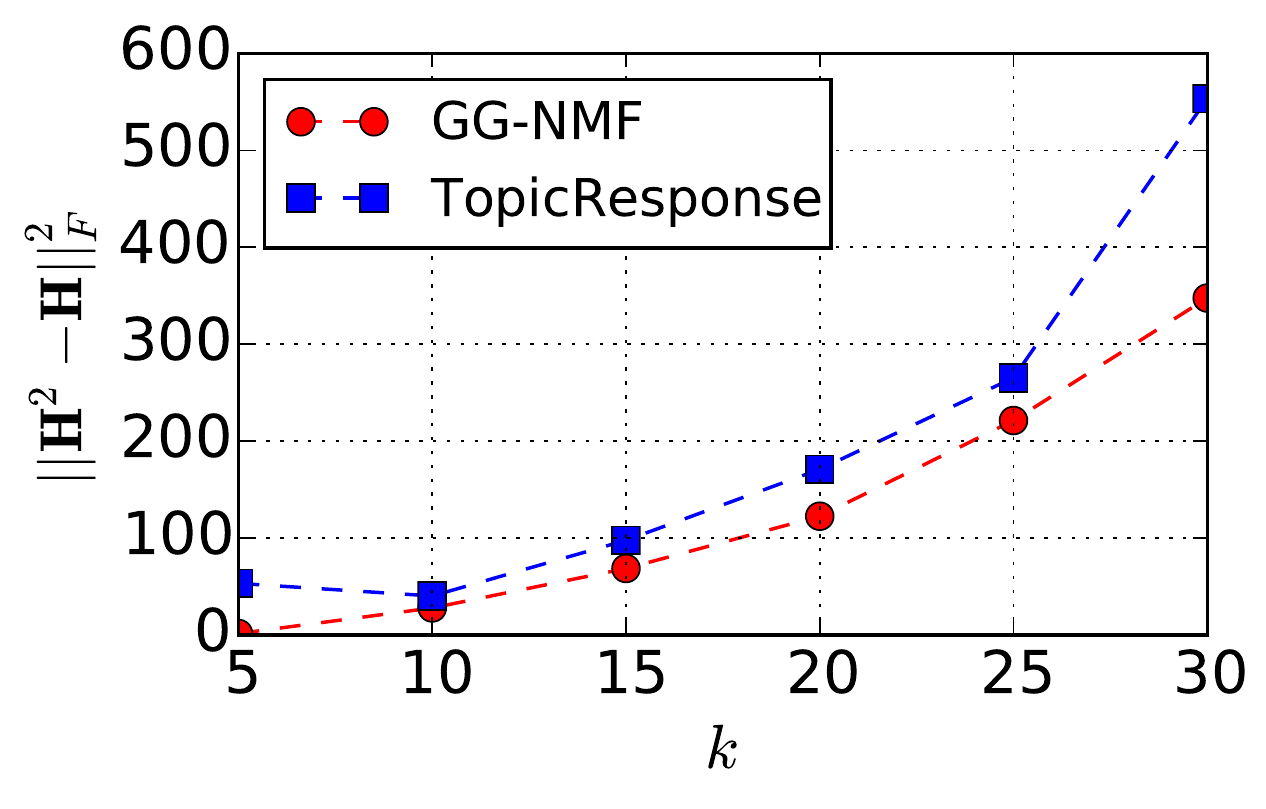}
            \caption{$\norm{\vbf{H}\circ\vbf{H}-\vbf{H}}$ on OPT}
        \end{subfigure}
        \caption{Performance of GG-NMF and TopicResponse with varying $k$.}
        \label{fig:lambdatopic}
    \end{figure}

\end{document}